\documentclass[11pt,reqno]{amsart}
\usepackage{amsaddr}

\usepackage[margin=1in]{geometry}
\usepackage[numbers,sort&compress]{natbib}
\usepackage[utf8]{inputenc} 
\usepackage[T1]{fontenc}    
\usepackage{booktabs}       
\usepackage{amsfonts}       
\usepackage{nicefrac}       
\usepackage{microtype}      
\usepackage{xcolor}         
\usepackage{amsmath, amssymb}
\usepackage{graphicx}
\usepackage{subcaption}
\usepackage{arydshln}
\usepackage{enumitem}
\usepackage{mdframed}

\usepackage{xurl}
\usepackage[colorlinks=true,linkcolor=blue,citecolor=blue,urlcolor=blue]{hyperref}

\newcommand{\Z}{\mathbb{Z}}
\newcommand{\C}{\mathbb{C}}
\newcommand{\R}{\mathbb{R}}
\newcommand{\E}{\mathbb{E}}
\newcommand{\norm}[1]{\left\lVert #1 \right\rVert}
\newcommand{\ip}[2]{\left\langle #1, #2 \right\rangle}
\theoremstyle{plain}
\newtheorem{theorem}{Theorem}[section]
\newtheorem{proposition}[theorem]{Proposition}
\newtheorem{corollary}[theorem]{Corollary}
\newtheorem{lemma}[theorem]{Lemma}
\theoremstyle{remark}

\usepackage[most]{tcolorbox}
\definecolor{HighlightGreen}{HTML}{B5E300}
\definecolor{HighlightGreenBg}{HTML}{F7FBEA}

\newcommand{\codeurl}{https://github.com/AndyLiuin/Spectral-lens-llm-training}

\definecolor{tabblue}{HTML}{1F77B4}
\definecolor{taborange}{HTML}{FF7F0E}
\definecolor{tabgreen}{HTML}{2CA02C}
\definecolor{tabred}{HTML}{D62728}
\definecolor{tabpurple}{HTML}{6D28D9}

\title{Spectral Lens: Activation and Gradient Spectra as Diagnostics of LLM Optimization}

\author{Andy Zeyi Liu}
\address{Department of Applied Physics, Yale University, New Haven, CT 06511, United States}
\email{andy.liu@yale.edu}

\author{Elliot Paquette}
\address{Department of Mathematics and Statistics, McGill University}
\email{elliot.paquette@mcgill.ca}

\author{John Sous}
\address{Department of Applied Physics, Yale University, New Haven, Connecticut 06511, USA}
\address{Energy Sciences Institute, Yale University, West Haven, Connecticut 06516, USA}
\email{john.sous@yale.edu}

\date{\today}

\begin{document}

\begin{abstract}
Training loss and throughput can hide distinct internal representation in language-model training. To examine these hidden mechanics, we use spectral measurements as practical and operational diagnostics. Using a controlled family of decoder-only models adapted from the modded NanoGPT codebase, we introduce an empirical protocol based on activation covariance and per-sample gradient SVD spectra. This dual-view reveals three empirical findings and one mechanistic explanation. First, batch size acts as a latent determinant of representation geometry: runs that reach equal loss settle into systematically distinct activation spectra. Second, the activation covariance tail measured early in training reliably forecasts downstream token efficiency. Third, movement of the activation spectrum head (leading modes), together with gradient spectra, characterizes underlying learning-dynamics changes, separating learning-side architectural improvements from primarily execution-side gains. These predictive and diagnostic signals persist across the 12-, 36-, and 48-layer model tiers. Finally, a mechanistic model proves the main observations and explains how activation covariance spectra correlate with task-aligned feature learning. Our code is available \href{\codeurl}{here}.

\end{abstract}

\maketitle
\setcounter{tocdepth}{1}
\tableofcontents

\begin{figure}[h]
  \centering
  \begin{subfigure}[t]{0.45\textwidth}
    \raggedright
    \includegraphics[width=\linewidth]{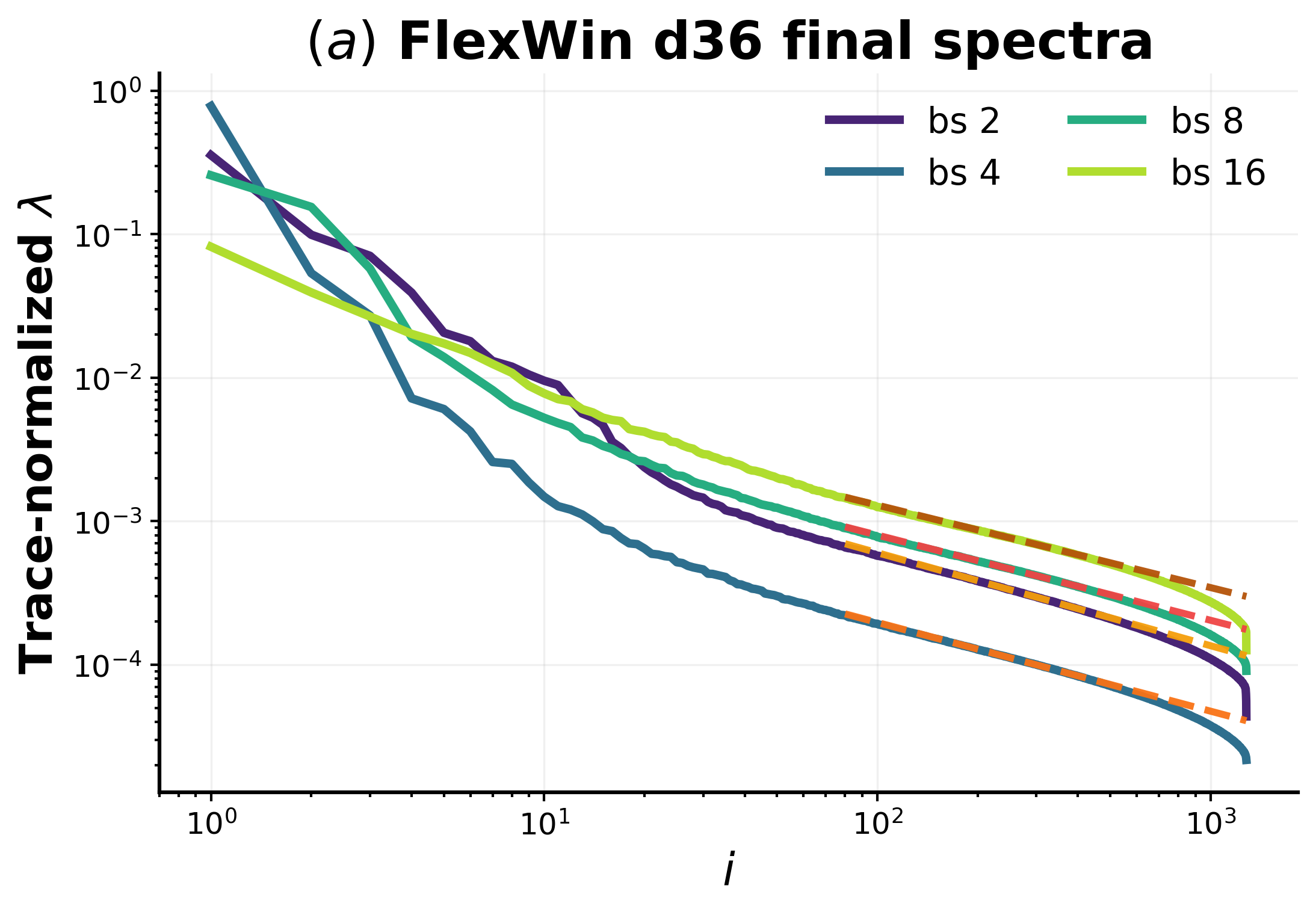}
    \label{fig:killerA}
  \end{subfigure}\hfill
  \begin{subfigure}[t]{0.45\textwidth}
    \raggedleft
    \includegraphics[width=\linewidth]{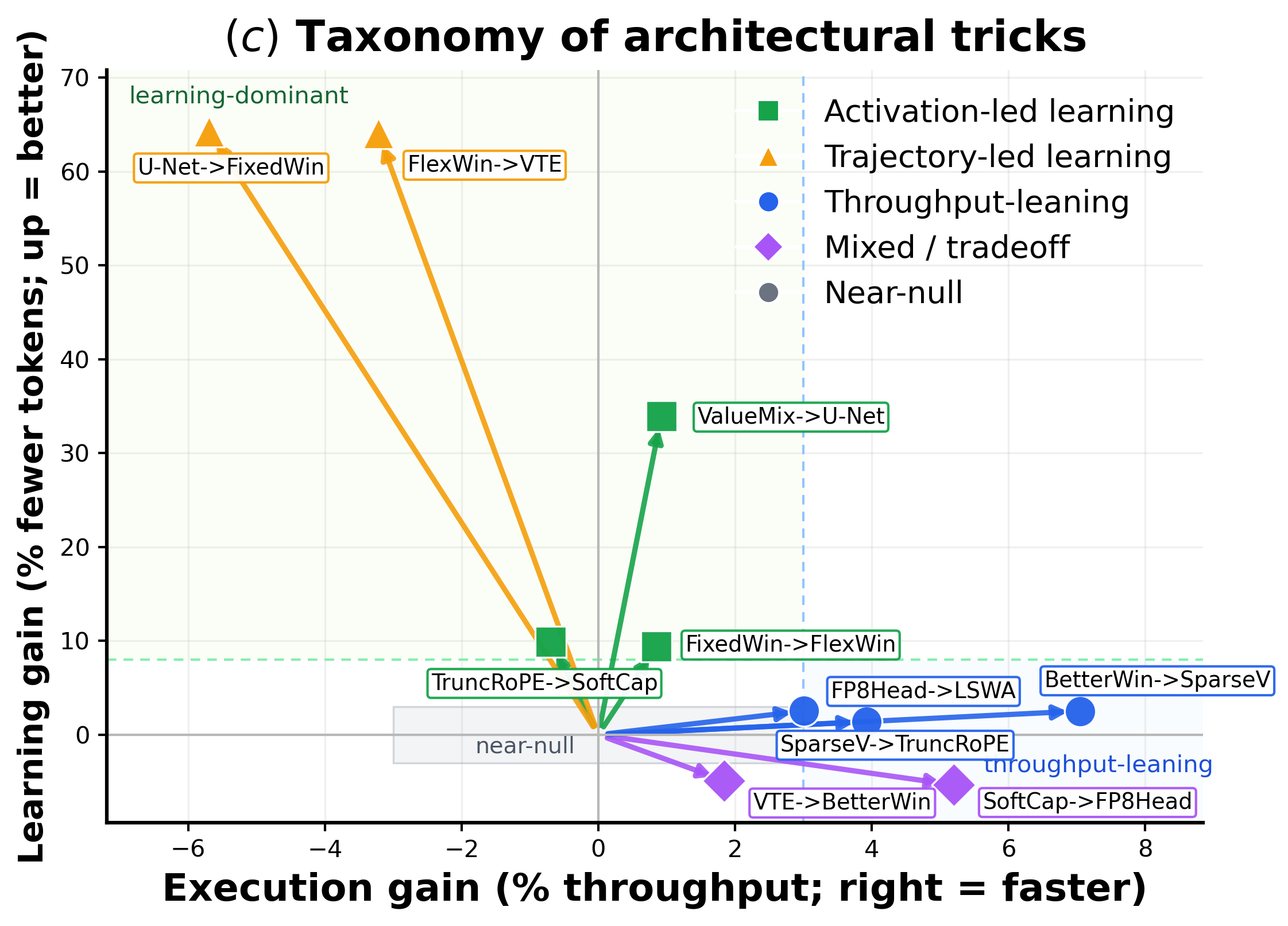}
    \label{fig:killerC}
  \end{subfigure}

  \vspace{0.5em}

  \begin{subfigure}[t]{0.9\textwidth}
    \centering
    \includegraphics[width=\linewidth]{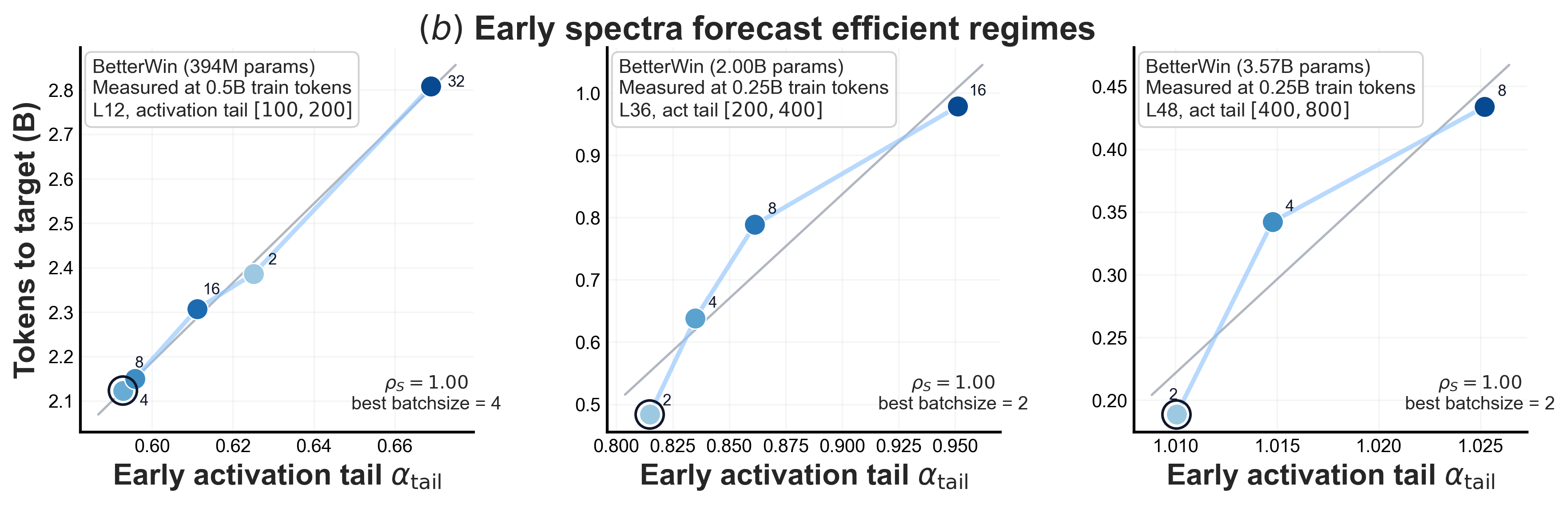}
    \label{fig:killerB}
  \end{subfigure}
\caption{\textbf{Spectral diagnostics as operational tools.} Each panel compares decoder-only language-model runs using trace-normalized activation covariance spectra and per-sample gradient SVD summaries, aligned either at matched loss or a fixed early token budget. \textbf{(a) Matched loss, distinct internal geometry.} Final-layer activation spectra of FlexWin d36 runs aligned at a common target loss. Dashed overlays show power-law fits over the tail window, where $\alpha_{\mathrm{tail}}$ is the log-log slope magnitude. \textbf{(b) Early prediction across scale.} Correlation between an early activation-tail statistic (measured at a fixed token budget) and the total tokens required to reach the target loss, evaluated across the d12, d36, and d48 layer tiers. \textbf{(c) Intervention taxonomy.} Consecutive architectural transitions mapped on a matched bs8 trunk by token gain and throughput gain, clustered by spectral diagnostic signals.}
  \label{fig:killer}
\end{figure}

\section{Introduction}

The optimization of large language models is largely guided by neural scaling laws, which predict that training loss evolves predictably with compute, data, and model size~\citep{kaplan2020scaling, hoffmann2022training}. Downstream capabilities then appear to be unlocked largely as a function of improving loss~\citep{brown2020language, rae2021scaling, ganguli2022predictability}. Yet, training loss alone often creates an optimization mirage: while training loss decreases smoothly and monotonically, it offers an incomplete account of internal learning, masking profound, non-monotonic transformations in the model's internal structure~\citep{li2025tracing, liu2025superposition}. This disconnect exposes a critical blind spot: if loss fails to capture how these internal geometries change, how can we reliably diagnose and tune the training process? Choosing training regimes proactively rather than merely summarizing them post hoc requires a deeper understanding of how internal representations are affected by practitioner choices.

Spectral probes offer a natural lens here. We focus on the heavy-tail exponent ($\alpha$), which summarizes how variance is distributed across internal dimensions and connects naturally to scaling-law efficiency~\citep{bahri2024explaining, martin2021implicit}. Crucially, these spectral signatures scale: diagnostic patterns observed in small-scale models transfer predictably to larger tiers, providing a stable protocol for forecasting dynamics as models grow in depth and parameter count.

Our empirical study proceeds in two stages: a systematic scan of decoder-only models to identify transferable spectral signatures, followed by a toy-model analysis for mechanistic grounding. Using the modded-NanoGPT codebase~\citep{modded_nanogpt_2024}, we generate a controlled intervention chain and evaluate diagnostics across layer tiers. Fig.~\ref{fig:killer} previews these results: we identify internal geometries that diverge despite matched loss (Fig.~\ref{fig:killer}a), demonstrate that early-training signatures identified in 12-layer runs reliably predict the token-efficiency of 36- and 48-layer models up to 3.57B parameters (Fig.~\ref{fig:killer}b), and categorize architectural optimizations along a learning-versus-throughput axis (Fig.~\ref{fig:killer}c). Finally, we ground these empirical observations in a modular-arithmetic toy model where Fourier structure makes task-relevant feature learning directly observable.

\begin{tcolorbox}[
  colframe=HighlightGreen,
  colback=HighlightGreenBg,
  boxrule=1.8pt,
  arc=4pt,
  left=7pt,
  right=7pt,
  top=6pt,
  bottom=4pt,
  width=\linewidth
]
\textbf{Main Results \& Contributions}
\vspace{-0.1em}
\begin{itemize}[
    label={\Large\textbullet},
    leftmargin=*,
    topsep=0pt,
    itemsep=0pt,
    parsep=0pt,
    partopsep=0pt
]

    \item \textbf{\textit{Internal Geometry Disparity:}} Batch size acts as a latent determinant of representation; reaching the same loss with different batch sizes does not imply a shared internal state.
    \item \textbf{\textit{Cross-Scale Predictivity:}} Early-training activation tails serve as a proxy for downstream token efficiency, identifying optimal batch regimes without full-scale training.
    \item \textbf{\textit{Optimization Taxonomy:}} Joint activation and gradient spectra provide a diagnostic signal to distinguish fundamental learning-centric gains from purely execution-side speedups.
    \item \textbf{\textit{Mechanistic Basis:}} A modular-arithmetic toy model provides a theory linking observed spectral shifts to feature-learning and preconditioning dynamics.

\end{itemize}
\end{tcolorbox}

\paragraph{Related work.}
Spectral measurements have been used as descriptive windows into representation
and optimization geometry, including eigendecay and effective-rank summaries of
activation covariance~\citep{li2025tracing,chang2022geometry,garrido2023rankme},
gradient-subspace concentration and Hessian outliers~\citep{gurari2018subspace,
ghorbani2018hessian,papyan2019hessian}, and heavy-tailed exponents connecting
spectral decay to batch-size effects and scaling-law
efficiency~\citep{mahoney2019heavy,bahri2024explaining,xie2023stochastic}.
The role of batch size in shaping internal geometry and efficiency has been studied
via sharp versus flat minima~\citep{keskar2017large} and gradient noise
scale~\citep{mccandlish2018empirical}.
In contrast, we use activation and gradient spectra as operational diagnostics
for hidden batch-induced geometry, early token-efficiency prediction, and
intervention-level mechanism differences. Additional related work is in Appendix~\ref{sec:appendix_related_work}.

\section{Preliminaries}
\label{sec:methods}
A standard decoder-only transformer~\citep{vaswani2017attention} maps token embeddings through $L$ residual blocks. Rather than evaluating only the final output, our spectral diagnostics probe intermediate representations within these blocks. For a chosen layer $\ell$, let $h_i^{(\ell)} \in \mathbb{R}^d$ denote a hidden state from a fixed held-out pool of validation sequences, reused across checkpoints and runs. Flattening all token positions from this pool yields an activation matrix $H_\ell \in \mathbb{R}^{N \times d}$, where $N$ counts token positions, and we estimate the centered covariance:
\begin{equation}
    H_\ell =
    \begin{bmatrix}
        (h_1^{(\ell)})\dots (h_N^{(\ell)})
    \end{bmatrix}^\top
    \in \mathbb{R}^{N \times d}, \quad \hat\Sigma_h = \frac{1}{N-1}\sum_{i=1}^N (h_i^{(\ell)}-\bar h)(h_i^{(\ell)}-\bar h)^\top.
\end{equation}
We compute the sorted eigenvalues $\lambda_1 \ge \lambda_2 \ge \dots \ge \lambda_d \ge 0$ of $\hat\Sigma_h$ and trace-normalize them as $\tilde\lambda_j = \frac{\lambda_j}{\sum_k \lambda_k}$. This keeps comparisons focused on spectral \emph{shape} rather than raw scale. 
Following theoretical work linking spectral decay to neural scaling laws and kernel learning curves \citep{bahri2024explaining, maloney2022solvable, bordelon2020spectrum, canatar2021spectral, spigler2019asymptotic}, we characterize the representation's power-law structure $\tilde{\lambda}_j \propto j^{-\alpha}$ by defining the band-restricted exponent $\alpha(I)$:
\begin{equation}
\alpha(I) := -\mathrm{slope}\Big({(\log j, \log \tilde{\lambda}_j)},{j \in I}\Big)
\end{equation}
where the slope is the least-squares fit in log-log coordinates. We evaluate $\alpha(I)$ over a pre-registered bank of expanding rank windows:
\begin{equation}
\mathcal{I}_{\mathrm{bank}} = \{[100,200], [200,400], [400,800]\}
\end{equation}
We define the scale-dependent diagnostic $\alpha_{\mathrm{tail}}$ as $\alpha(I^{(s)})$, where $s$ indexes the model scale---specifically $d12$, $d36$, and $d48$ for the 12-, 36-, and 48-layer models studied in this paper. Rather than fixing a single universal window across all scales, we select $I^{(s)}$ from $\mathcal{I}_{\mathrm{bank}}$ according to model capacity. The motivation is that larger models possess more hidden dimensions and learn a larger number of resolved feature directions, so the boundary between the learned head and the unresolved tail of the activation spectrum sits at a higher rank. Measuring $\alpha$ over a window that is too low-rank for a large model would therefore probe already-learned features rather than the tail of interest. Concretely, we calibrate the informative window on the smallest scale ($I^{(d12)}=[100,200]$) and shift it outward at larger scales ($I^{(d36)}=[200,400]$, $I^{(d48)}=[400,800]$). For even larger models, the same rule suggests continuing to the next window in the bank. This protocol allows $\alpha_{\mathrm{tail}}$ to serve as a prospective diagnostic that remains comparable across scales. 

Complementing the activation representations, we also probe the model's optimization dynamics. For a selected weight matrix $W \in \mathbb{R}^{d_{\mathrm{out}} \times d_{\mathrm{in}}}$ with $P = d_{\mathrm{out}} \cdot d_{\mathrm{in}}$ parameters, we compute per-sample gradients on a fixed held-out pool of validation sequences, where one \emph{sample} means one validation sequence rather than one token. Let $g_m = \mathrm{vec}\!\left(\nabla_W L_m(\theta)\right) \in \mathbb{R}^{P}$ denote the gradient of the mean autoregressive loss on sample $m$. Stacking $M$ such rows gives
\begin{equation}
    G = \begin{bmatrix} g_1 \dots g_M \end{bmatrix}^\top \in \mathbb{R}^{M \times P},
\end{equation}
and we analyze the singular value spectrum of $G$. Since $G$ is matrix-specific, gradient spectra are probe-dependent; we default to the final-block attention output projection $W_O$ and audit alternatives in Appendix~\ref{sec:appendix_grad_matrix_choice}, which shows that matrix choice changes the detailed concentration and tail shape but not the qualitative diagnostic conclusions. The two spectra play complementary roles: activation spectra answer ``where representation variance lives,'' while gradient spectra answer ``how concentrated the update is.'' To rigorously compare runs across these two views, we compare spectra under matched loss, matched tokens, and matched schedule fraction, focusing mainly on matched loss as it most sharply isolates differences in internal states despite identical performance.

Finally, for early prediction, each architecture family defines its own relative efficiency. We use \emph{effective batch size} $B$ because, from the FixedWin modded-NanoGPT variant onward, the 65{,}536-token context forces local batch size 1 on a single GPU; $B$ therefore denotes the global batch induced by gradient accumulation.\footnote{Tokens per optimizer step scale as effective batch size times sequence length. The short-context prefix uses 1{,}024-token sequences, while the FixedWin-and-later long-context trunk uses 65{,}536-token sequences; see Table~\ref{tab:interventions} and Appendix~\ref{sec:appendix_training_setup}.} Let $T(B)$ be the total tokens required for a training run at effective batch size $B$ to reach a target validation loss. We evaluate efficiency via the within-family token ratio:
\begin{equation}
    \epsilon_{\mathrm{tok}}(B) = \frac{T(B)}{\min_{B'}T(B')},
\end{equation}
where the minimum is taken over the set of effective batch sizes $B'$ swept within the same architectural family.

\subsection{Model architectures and training setup}
\label{sec:setup}

Our experiments evaluate architectural variants and optimization techniques adopted from the \texttt{modded-nanogpt} codebase~\cite{modded_nanogpt_2024}, as summarized in Table~\ref{tab:interventions}. These modifications form an incremental refinement path: each subsequent variant is cumulative, incorporating the optimizations of its predecessors. Full architectural dimensions, the sequence of concatenated code modifications, per-variant attributions, and validation of the spectral equivalence between the 100B-token and 10B-token FineWeb splits are detailed in Appendix~\ref{sec:appendix_training_setup}.

\begin{table}[htbp]
  \centering
  \small
  \setlength{\tabcolsep}{4pt}
  \caption{\textbf{Experimental blocks and named model families.}\protect\footnotemark\ We summarize the core configurations here; exact architectural dimensions and cumulative code changes are detailed in Appendix~\ref{sec:appendix_training_setup}.}
  \label{tab:interventions}
  \begin{tabular}{p{0.18\linewidth}p{0.30\linewidth}p{0.42\linewidth}}
    \toprule
    Experimental block & Model families & Core setting \\
    \midrule
    Variants 1--4 &
    Baseline, RoPE, Muon, Untied &
    d12 GPT-2-small scale models (1k context), trained on the 100B-token split of FineWeb. \\
    \addlinespace
    Variants 5--16 &
ValueMix, U-Net, FixedWin, FlexWin, VTE, BetterWin, SparseV, TruncRoPE, SoftCap, FP8Head, LSWA, AttnScale &
d12 models trained on the 10B-token split of FineWeb. ValueMix and U-Net remain in the short-context regime; the 1k$\rightarrow$65k context shift occurs at U-Net$\rightarrow$FixedWin. All regimes match a total batch budget of $\approx$0.5M tokens/step (\emph{bs8-equivalent}).
 \\
    \addlinespace
    Larger-scale robustness follow-up &
    FlexWin \texttt{d36}, BetterWin \texttt{d36}, SparseV \texttt{d36}; BetterWin \texttt{d48} &
    Scale-up runs testing generalization, scaling up to 48 layers and 32k context across multiple batch tiers. \\
    \bottomrule
  \end{tabular}
\end{table}
\footnotetext{Citations and contributor attributions for each listed variant are collected in Appendix~\ref{sec:appendix_variant_attribution}.}

Throughout our language-model experiments, the main d12 comparisons target a validation loss of 3.2 on a fixed held-out validation set taken from the FineWeb-10B validation split. We use the 100B-token split of FineWeb~\citep{lozhkov2024fineweb} for earlier, less data-efficient variants, transitioning to the 10B-token split from the ValueMix variant onward. Unless noted otherwise, efficiency is measured strictly in training tokens consumed to reach this first target checkpoint, which we treat as a proxy for compute given the fixed model architecture and hardware setup within each comparison.

To ensure our spectral comparisons isolate batch-dependent geometry rather than under-tuned optimization, we perform ASHA-filtered learning-rate sweeps~\citep{li2020system} for each batch tier prior to analysis (see Appendix~\ref{sec:appendix_training_setup} for the full sweep procedure). Thus, our matched-loss figures compare each tier under its own optimal schedule. We evaluate batch tiers $B \in \{1, 2, 4, 8, 16, 32\}$ for the 12-layer models, and $B \in \{2, 4, 8, 16\}$ for the d36 robustness follow-up described in Table~\ref{tab:interventions}. 

\subsection{Transition outcome coordinates}\label{subsec:Tran-out}
Our subsequent analysis studies consecutive transitions along the matched-total-batch intervention chain.
For each architectural transition $a\rightarrow b$ on effective batch size 8 (see Table~\ref{tab:interventions}), we define the relative token and throughput gains:
\begin{equation}
\mathrm{TokGain}(a\!\rightarrow\! b)=\frac{T(a)}{T(b)}-1,\qquad
\mathrm{ThrGain}(a\!\rightarrow\! b)=\frac{Q(b)}{Q(a)}-1,
\end{equation}
where $T(\cdot)$ is the number of training tokens required to reach the target loss, and $Q(\cdot)$ is the throughput in tokens per second.
For visualization, we also use the log-gain coordinates:
\begin{equation}
g_{\mathrm{tok}}=\log\!\big(T(a)/T(b)\big),\qquad
g_{\mathrm{thr}}=\log\!\big(Q(b)/Q(a)\big).
\end{equation}

\section{Batch Size Leaves Distinct, Predictive, and Robust Spectral Signatures}
\label{sec:batch}
With our spectral metrics established, we first apply them to reexamine a fundamental training variable: the effective batch size, and show that it is a fundamental driver of representational state. This section establishes three batch-size-specific findings. We find that batch size induces disparate activation covariance spectra. We then demonstrate that these spectral signatures are predictive, allowing us to identify the compute-optimal batch size early in the training process. Finally, we validate that these findings are scale-stable, showing that the signatures identified in 12-layer models generalize to 36- and 48-layer architectures.

\begin{figure}[t]
  \centering

  \begin{subfigure}[t]{0.49\textwidth}
    \centering
    \includegraphics[width=\linewidth]{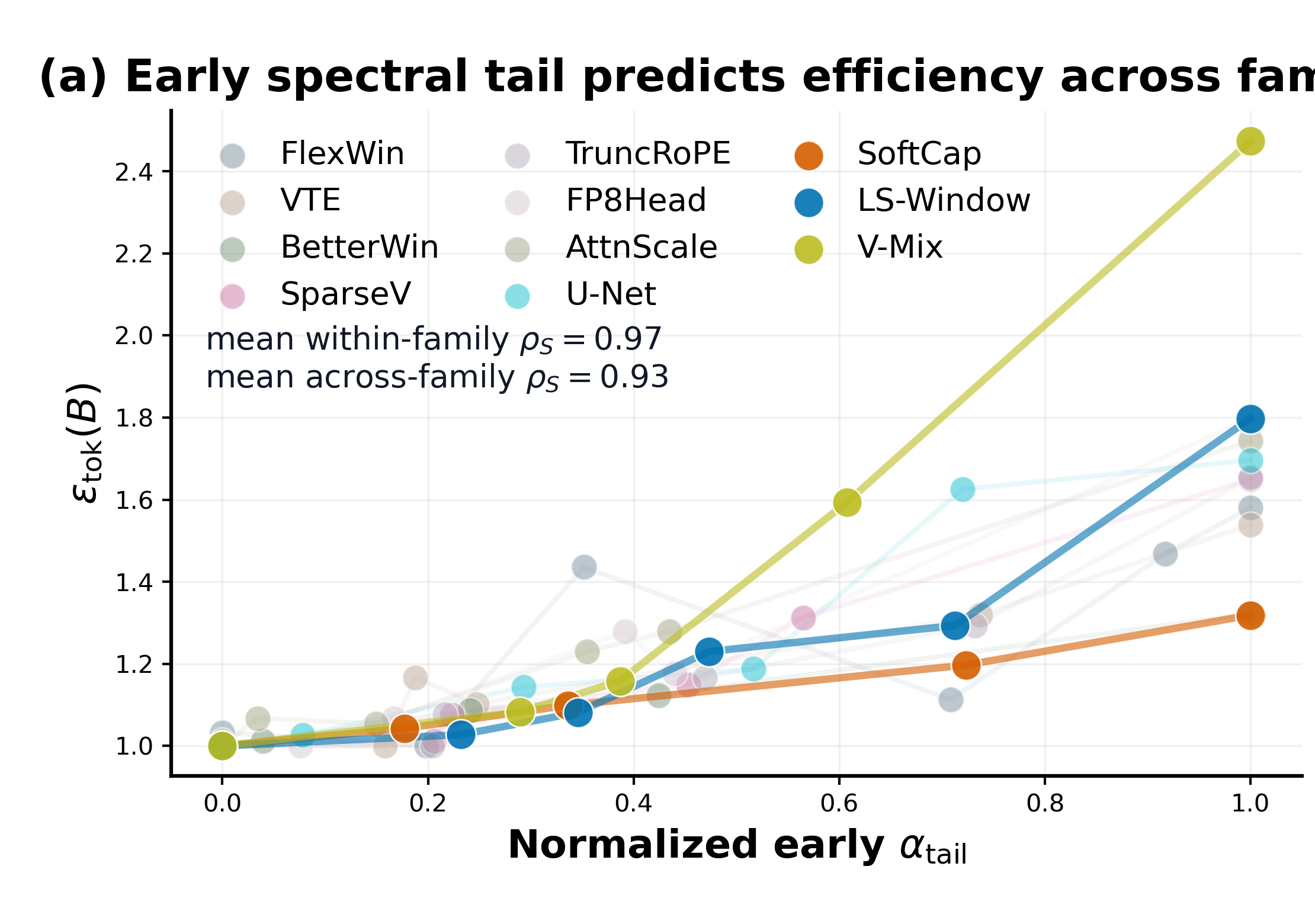}
  \end{subfigure}\hfill
  \begin{subfigure}[t]{0.49\textwidth}
    \centering
    \includegraphics[width=\linewidth]{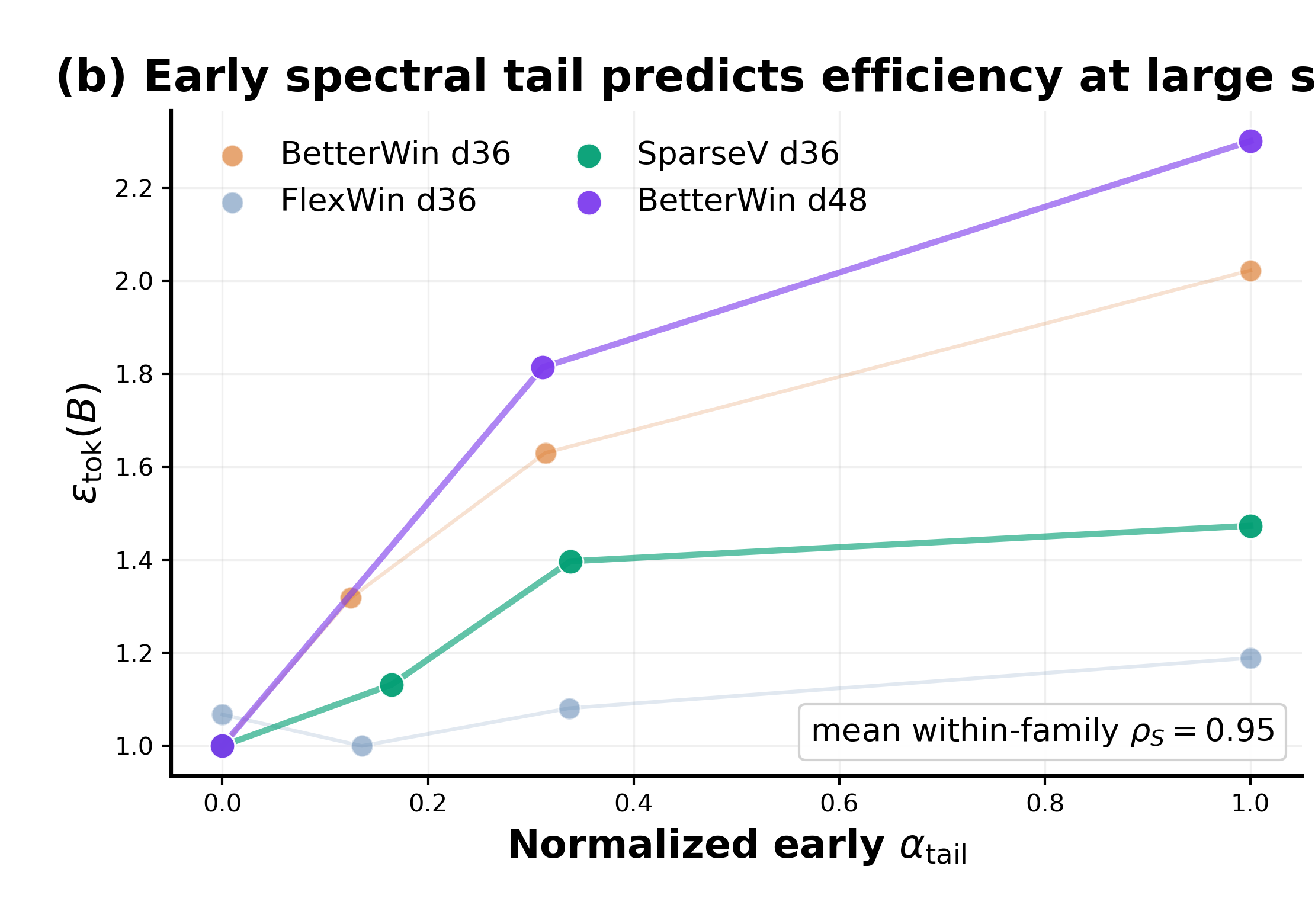}
  \end{subfigure}

\caption{\textbf{Early spectral-tail signal predicts efficient training regimes across scale.}
In both panels, the x-axis shows the normalized early tail exponent $\alpha_{\mathrm{tail}}$, while the y-axis shows token inefficiency $\epsilon_{\mathrm{tok}}(B)$, so lower values indicate better token efficiency.
(a) At d12 scale, this early spectral-tail statistic already organizes later efficiency across model families: families that move to the right also tend to move upward, yielding mean within-family Spearman correlation $\rho_S = 0.97$ and mean across-family $\rho_S = 0.93$.
(b) The same normalized early-tail view remains predictive across the available 36-/48-layer settings, with mean within-family Spearman correlation $\rho_S = 0.95$.
Together, these panels show that a spectral-tail signal is already informative about which training regimes will be most token-efficient, and that this signal persists from d12 to the 36- and 48-layer support runs.}

  \label{fig:scaled_hidden_regimes}
\end{figure}

\begin{enumerate}[leftmargin=*, topsep=0pt, itemsep=2pt]
    \item \textbf{Matched loss does not imply matched geometry.} Empirically, matched-loss runs within a given architecture settle into systematically different activation-covariance spectra, as demonstrated in Fig.~\ref{fig:killer}(a). We justify the universality of this observation in Appendix~\ref{sec:appendix_additional_plots}, providing spectra across all tested model families and scales. Besides, since our architectural optimizations are cumulative, with most built upon the Muon optimizer, we perform a separate ablation study in Appendix~\ref{sec:appendix_muon_adam} comparing Adam and Muon. We find that batch size remains a meaningful latent determinant under Adam, though Muon is more batch-sensitive, producing sharper head concentration and greater spectral separation. We provide a mechanistic analogue of this phenomenon in Section~\ref{sec:toy}. This same toy model is used to mathematically demonstrate that equal loss does not guarantee an equivalent spectrum. 
    \item \textbf{Early spectra predict efficient tiers.} As initially demonstrated in Fig.~\ref{fig:killer}b, the early activation covariance tail exponent ($\alpha_{\mathrm{tail}}$) exhibits a strong positive correlation with the total training tokens required to reach a target loss. We expand this analysis in Fig.~\ref{fig:scaled_hidden_regimes}a by aggregating all depth-12 model
variants across their respective batch-size tiers. To enable cross-variant
comparisons, we apply a per-variant min-max normalization to the early
$\alpha_{\mathrm{tail}}$ statistic, scaling each variant's raw values linearly
so that its smallest exponent maps to 0 and its largest to 1.

This normalized view yields strong evidence that early spectra accurately diagnose proximity to the token-optimal batch size. We quantify this predictive power using the Spearman rank correlation ($\rho_S$) to measure monotonic alignment, where $\rho_S = 1$ indicates a perfectly monotonic relationship between the early diagnostic and final token efficiency. The high mean \emph{within-family} correlation ($\rho_S = 0.97$) demonstrates that early $\alpha_{\mathrm{tail}}$ correctly orders the eventual token efficiency of batch tiers inside any given architecture. Furthermore, the \emph{across-family} correlation ($\rho_S = 0.93$) on the pooled normalized data confirms that this transformation successfully aligns distinct architectures onto a shared efficiency axis. To track how this predictive signal develops, Appendix~\ref{sec:appendix_predictive_ablations} provides an ablation study mapping the evolution of the correlation coefficient across training progress percentages. Finally, Section~\ref{sec:toy} provides a mechanistic interpretation of this empirical signal, where we demonstrate that flatter informative tails mathematically correspond to shorter remaining time-to-target when learning is bottlenecked by the recruitment of task-relevant tail features.

\item \textbf{The predictive signal survives scale.} 
Our larger-tier experiments confirm that both facets of this phenomenon generalize to the 36- and 48-layer architectures, including the 3.57B-parameter BetterWin d48 run. The d36 hidden-state panel in Fig.~\ref{fig:scaled_hidden_regimes} demonstrates that matched-loss spectral separation persists across batch tiers despite increased depth and context length. Crucially, the early predictive signal remains highly reliable; correlation analysis on the scaled d36/d48 models (Fig.~\ref{fig:scaled_hidden_regimes}b) yields a strong mean within-family correlation of $\rho_S=0.95$. In deeper models, the informative tail windows shifts to higher ranks, consistent with deeper models having a quantitatively higher rank of learned features.
\end{enumerate}

Appendix~\ref{sec:appendix_predictive_ablations} provides two controls for the early-prediction result. First, a random-seed control shows that FlexWin tier-16 runs initialized with different seeds cluster much closer to one another than to different batch tiers, supporting that the observed separation is driven by effective batch size. Second, a d36 layerwise ablation shows that the predictive signal is depth-sensitive: early and middle layers are weak or sign-inconsistent, while the deepest stored probe layer gives the strongest positive correlation across the available d36 families. This supports our convention of using the deepest available activation layer, while leaving full layer-selection rules to future work.

\section{Spectral Taxonomy of Architectural Acceleration}
\label{sec:taxonomy}

\begin{figure}[!htbp]
    \centering
    \includegraphics[width=\textwidth]{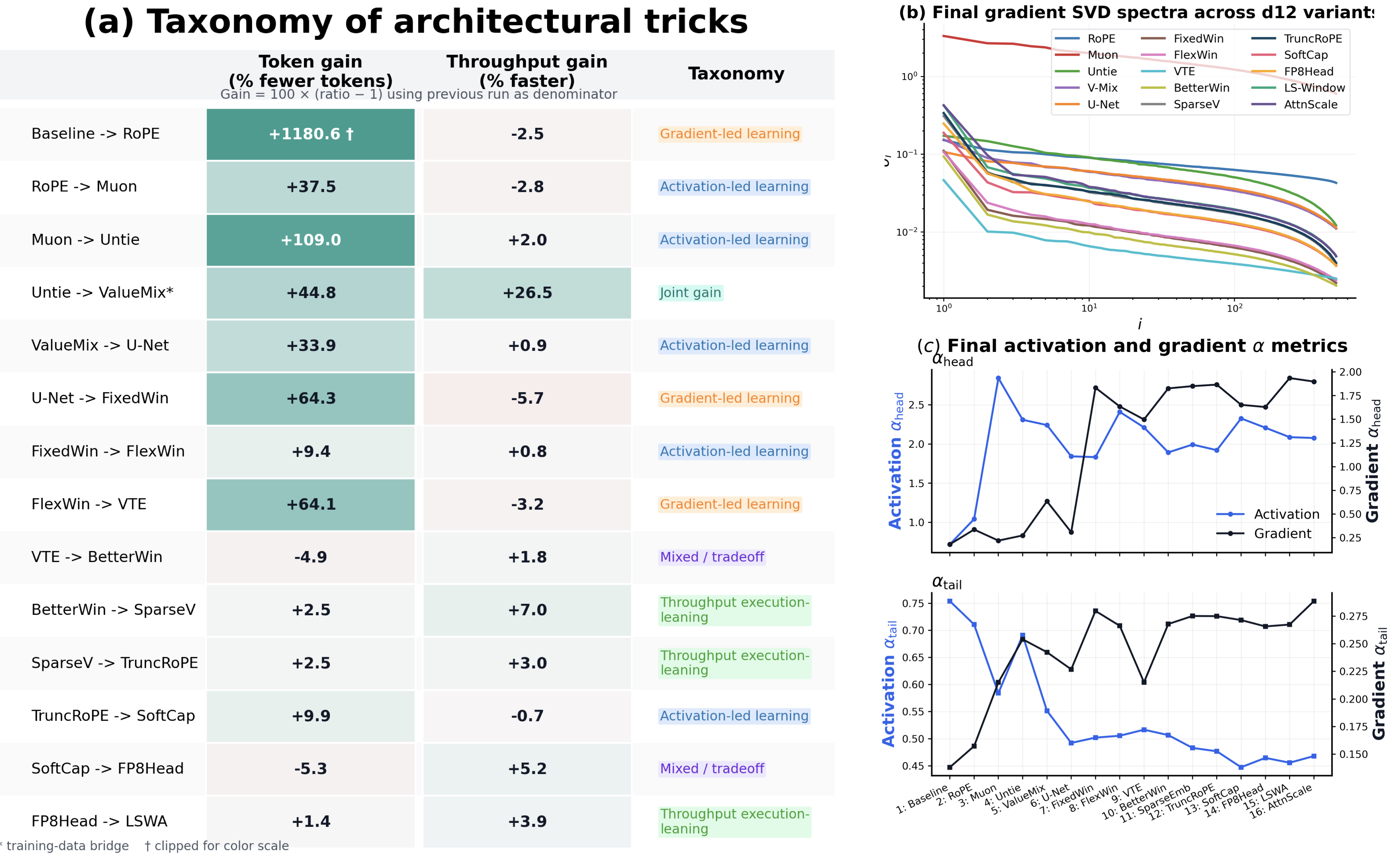}
  \caption{\textbf{Architectural tricks fall into clear empirical taxonomy.}
  (a) Each consecutive d12 transition is summarized by token gain, throughput gain, and its taxonomy label. The outcome columns separate learning-side, throughput-side, joint, and tradeoff effects, but the activation-led versus gradient-led split requires the spectral evidence in panels (b)--(c); a representative path-level example is deferred to Appendix~\ref{sec:appendix_additional_plots}.
  (b) Final gradient SVD spectra across the d12 variants show distinct update-subspace regimes rather than a single smooth continuum.
  (c) Final activation and gradient $\alpha$ summaries vary partly independently across the chain, helping distinguish activation-side gains from gradient-side gains.}
  \label{fig:taxonomy}
\end{figure}

Our second main result turns the spectral framework into an optimization-level diagnostic. Moving beyond the batch-size analysis, we ask a different operational question: once an architectural intervention improves performance, \emph{how} did it help? 

We address this by placing each consecutive transition along our incremental refinement path onto two observable outcome axes: token gain and throughput gain (defined in Section~\ref{subsec:Tran-out}). Since these architectural tricks are cumulative, the reported percentage changes represent the incremental improvement relative to the immediately preceding variant. As validated in Appendix~\ref{sec:appendix_training_setup}, the 100B-token and 10B-token splits of FineWeb produce nearly identical spectral signatures. Because of this baseline equivalence, we link the optimization tricks together regardless of the dataset split, evaluating the entire progression (variants~1--18) as a single, continuous chain in Fig.~\ref{fig:taxonomy}a.

While the outcome coordinates categorize performance, they lack a mechanistic explanation. To bridge this gap, we use activation covariance and gradient SVD spectra as in Fig.~\ref{fig:taxonomy}(b--c) to diagnose the underlying geometric driver, classifying transitions into four explicit categories:
\begin{itemize}[leftmargin=*, topsep=2pt, itemsep=2pt, parsep=0pt]
    \item \textcolor{tabblue}{\textbf{Activation-led learning:}} The token gain corresponds to a clear displacement in final activation covariance spectrum, in particular $\alpha_{\text{head}}$. These are interventions that directly alter representational degrees of freedom or readout geometry (e.g., RoPE, untied heads, FlexWin, SoftCap).
    \item \textcolor{taborange}{\textbf{Gradient-led learning:}} The dominant geometric shift appears in the gradient spectra or checkpoint-wise path separation, even if the final activation endpoint remains relatively static. These typically involve modifications to routing or attention horizons (e.g., FixedWin, VTE).
    \item \textcolor{tabgreen}{\textbf{Throughput execution-leaning}:} These interventions primarily recondition execution cost or trade off learning dynamics for throughput. They encompass masking, kernel refinements, and optimizer tunings (e.g., SparseV, TruncRoPE, FP8Head).
\item \textcolor{tabpurple}{\textbf{Mixed / tradeoff:}} The intervention improves one training objective while weakening the other comparably (e.g., positive token gain with negative throughput gain), thus failing to land cleanly in a single mechanistic bucket.
\end{itemize}
This diagnostic split aligns systematically with the underlying architectural mechanisms detailed in Appendix~\ref{sec:appendix_training_setup}. By examining the operational definition of each intervention, we can trace why specific tricks map to specific spectral outcomes:

\textbf{Representational shifts manifest as \textcolor{tabblue}{activation-led learning}.}
The clearest cases are transitions that directly expand or reshape the model's representational degrees of freedom. In Muon$\rightarrow$Untied, removing the weight-tying constraint between token embedding and LM head expands the output-layer parameterization. In TruncRoPE$\rightarrow$SoftCap, the bounded $\tanh$ reshapes the output geometry seen by the loss. In both cases, the dominant signature is a clear displacement in the final activation covariance spectrum.

\textbf{Routing and horizon changes manifest as \textcolor{taborange}{gradient-led learning}.}
These transitions modify how information moves through the network without comparably large changes to the final readout space, leaving the activation endpoint relatively stable while changing the update path. U\mbox{-}Net$\rightarrow$FixedWin replaces full causal attention with a sliding FlexAttention mask, altering which tokens interact rather than the feature space itself. FlexWin$\rightarrow$VTE injects a layer-by-layer value-token embedding pathway, changing value routing. Both are evidenced empirically by sharper changes in gradient concentration and checkpoint-wise path separation than in the final activation endpoint.

\textbf{Systems-motivated constraints manifest as \textcolor{tabgreen}{throughput execution-leaning}.}
These are late-trunk transitions whose main effect is to improve execution efficiency while producing comparatively smaller activation-side movement. BetterWin$\rightarrow$SparseV trades dense value embeddings for sparse reusable tables, SparseV$\rightarrow$TruncRoPE simplifies positional application, and FP8Head$\rightarrow$LSWA primarily improves the execution side relative to the preceding row. Their common signature is clearer throughput gain than learning-side geometric displacement.

Crucially, this taxonomy relies on different spectral features than our earlier batch-size analysis. While early prediction (Section~\ref{sec:batch}) depends on the activation \emph{tail}—reflecting the distribution of unresolved mass—this architectural taxonomy is governed by activation \emph{head} movement and gradient concentration. As mathematically grounded by our toy model (Section~\ref{sec:toy}), head and gradient dynamics capture the shifting balance between learned and residual energy, rather than unresolved structural capacity.\footnote{Phase-like geometry dynamics (collapse--expansion--compression) are also observable in our low-batch runs. However, because their visibility depends heavily on measurement alignment, we treat them as secondary qualitative evidence and defer the full analysis to Appendix~\ref{sec:appendix_additional_plots}.}

\section{Mechanistic Model: From Spectral Signatures to Feature Learning} 
\label{sec:toy}

\begin{figure}[t]
  \centering
  \begin{subfigure}[t]{0.32\textwidth}
    \centering
    \includegraphics[width=\linewidth]{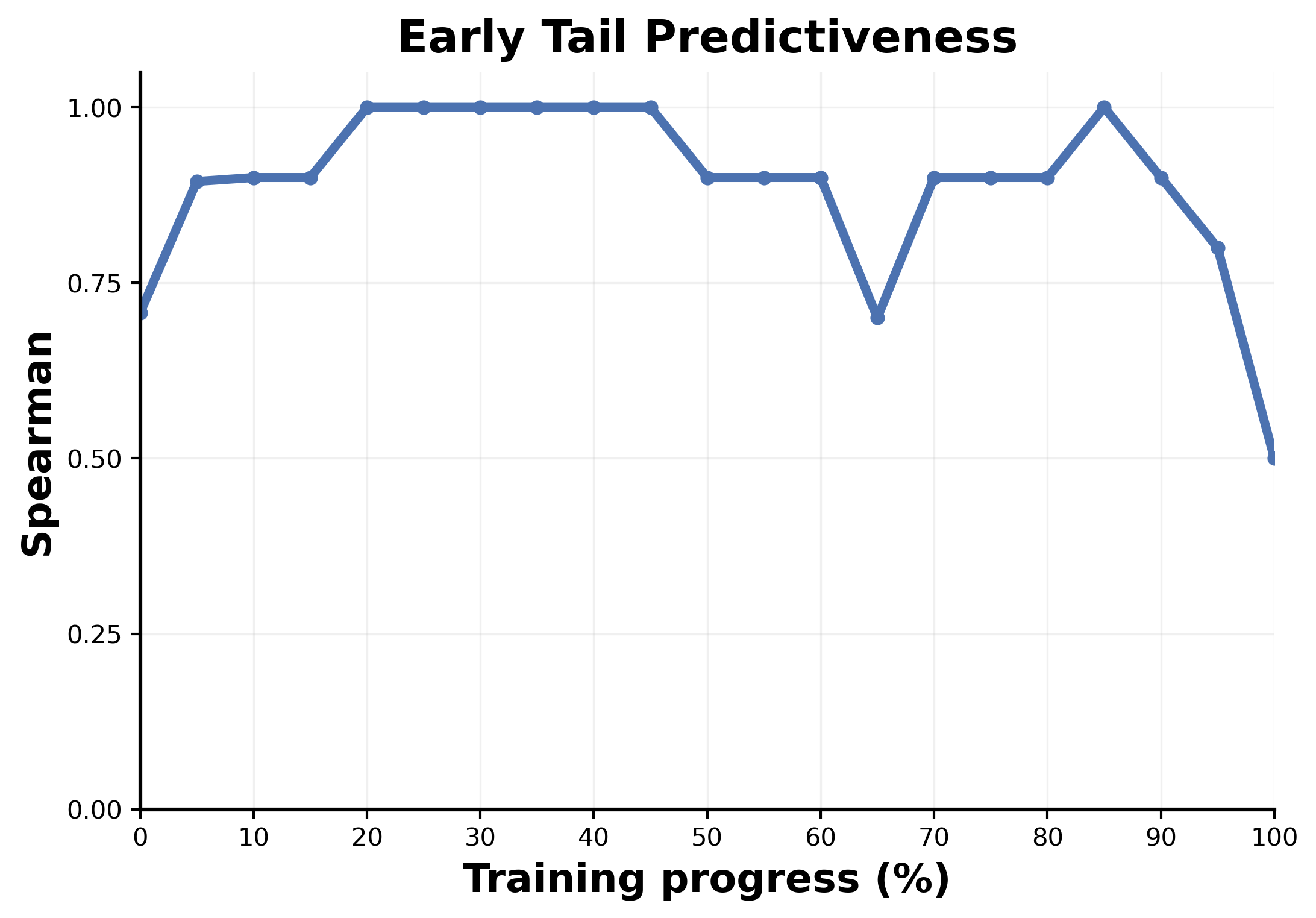}
  \end{subfigure}\hfill
  \begin{subfigure}[t]{0.32\textwidth}
    \centering
    \includegraphics[width=\linewidth]{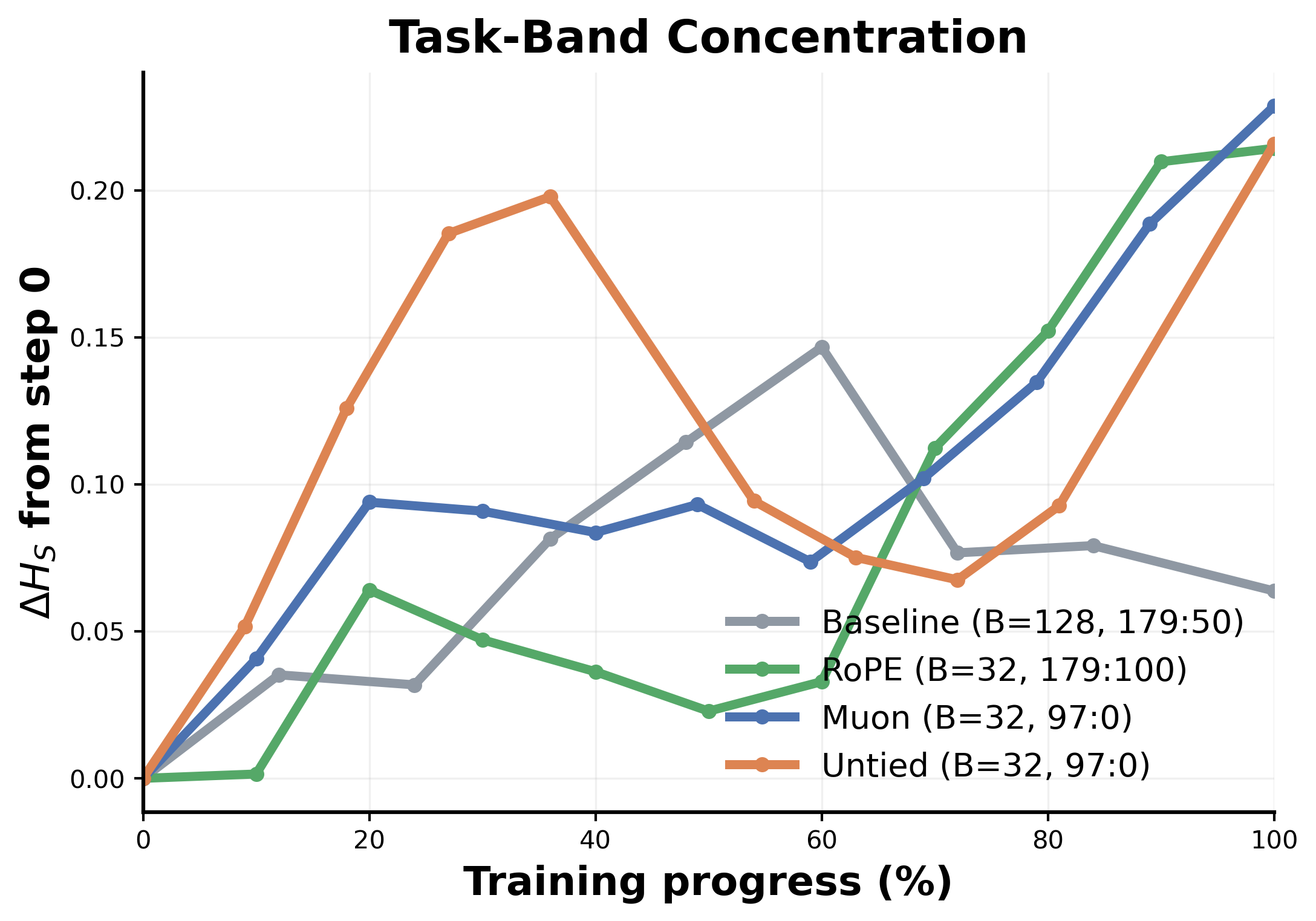}
  \end{subfigure}\hfill
  \begin{subfigure}[t]{0.32\textwidth}
    \centering
    \includegraphics[width=\linewidth]{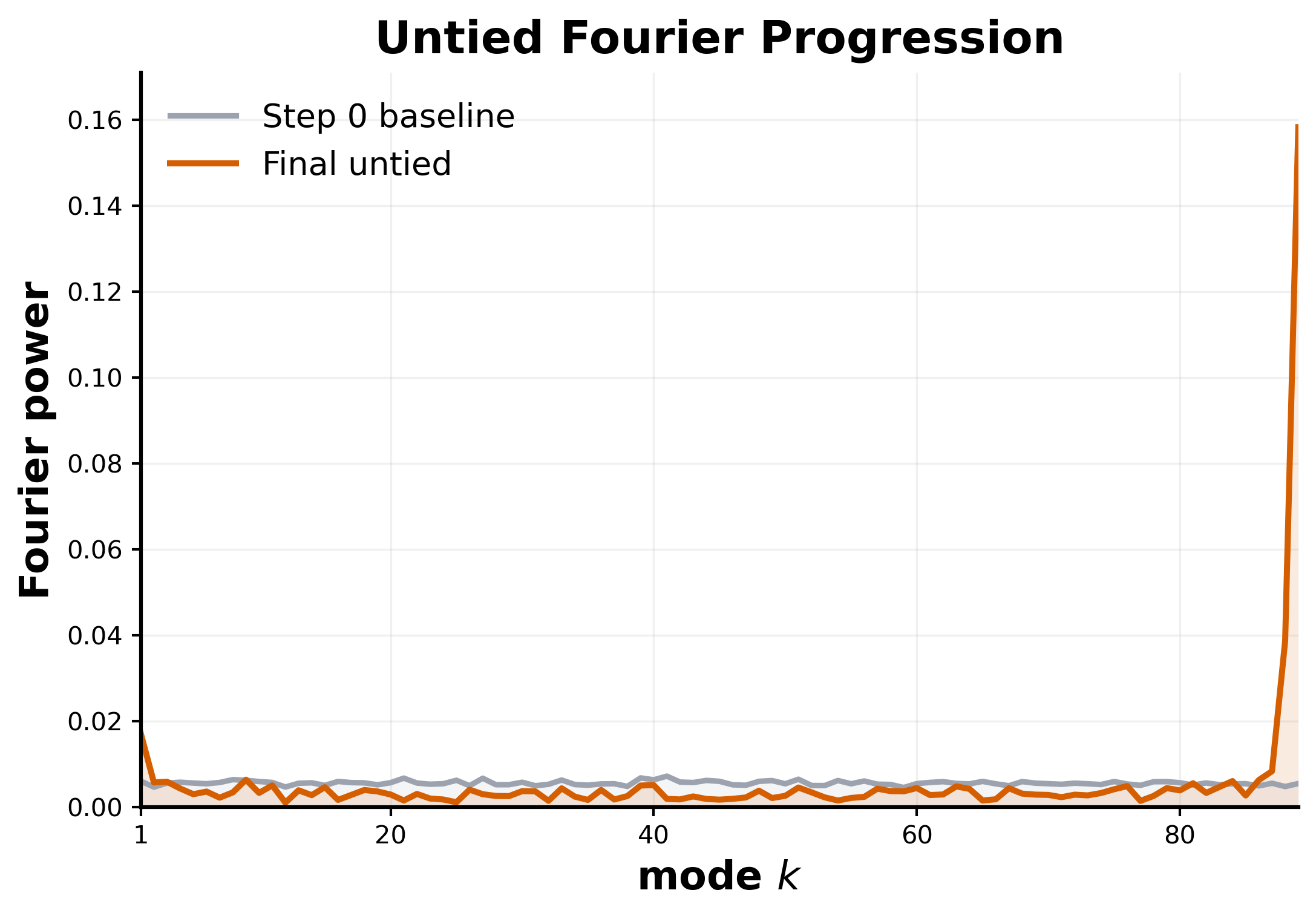}
  \end{subfigure}
  \caption{\textbf{Toy simulation links spectral diagnostics to Fourier feature learning.}
  (a) In the Muon stage, a local activation-tail statistic on ranks $10{:}40$ predicts eventual token efficiency early in training, reaching perfect Spearman correlation around $20\%$ training progress.
  (b) Targeted best-regime replays track $H_S$, the task-band concentration score used in the theory, across the baseline, RoPE, Muon, and untied stages.
  (c) The untied replay makes feature emergence explicit: the initial Fourier profile is diffuse, while the final profile concentrates on a few modes.
  Together, the panels support the interpretation that spectral changes track task-aligned feature recruitment in this controlled setting.}
  \label{fig:toy}
\end{figure}

Sections~\ref{sec:batch}--\ref{sec:taxonomy} show that spectra predict and diagnose training regimes empirically. This section explains why by studying toy tasks whose Fourier features are known, so activation spectra can be compared directly to task-aligned feature learning. Inspired by mechanistic interpretability and grokking studies on algorithmic and modular-arithmetic tasks \citep[e.g.,][]{nanda2023progress,liu2022towards,gromov2023grokking}, we introduce a controlled Fourier random feature model for next-token prediction where the latent task structure is analytically visible. 
Fix an integer cycle length $c \ge 2$, a step size $d \in \{0,1,\dots,c-1\}$, an offset $o$, and a context length $L \ge 1$. 
For a latent phase $a \in \Z_c$, we define the clean single-component sequence and its target as:
\begin{equation}
    x(a)=\bigl(o+(a+jd)\bmod c\bigr)_{j=0}^{L-1}, \qquad y(a)=o+(a+Ld)\bmod c.
\end{equation}
As the essential latent state is the phase $a$ on the finite cyclic group $\Z_c$, the natural representation coordinates are the Fourier characters $\chi_r(a)=e^{2\pi i r a/c}$ for $r \in \{0,1,\dots,c-1\}$. 
These characters form an orthonormal basis on $\Z_c$, serving as a dictionary of fundamental Fourier random features. 

To bridge these observables with our empirical taxonomy, we incrementally apply selected optimization tricks RoPE, Muon, and untied embeddings, to a baseline, allowing us to empirically track how each accelerates task-aligned feature learning (full protocol in Appendix~\ref{sec:appendix_toy_setup}). To formally ground these observations, our theoretical analysis proceeds through models of increasing complexity. We begin with a linearized gradient-flow model, progress to a two-layer diagonal Fourier factor model, and finally analyze a full two-layer Transformer. The main results for each setting are outlined below with proofs detailed in Appendix~\ref{sec:appendix_toy_theory}.

\begin{enumerate}[leftmargin=*, topsep=0pt, itemsep=2pt]
    \item \textbf{Informal result 1: Tail shape predicts family-local token efficiency.}
Using a linearized gradient-flow model, we first prove that distinct batch regimes can achieve identical loss while possessing fundamentally different activation spectra. 
In the cyclic-shift core where Fourier modes diagonalize the dynamics, the activation spectrum under specialization develops a learned head and an unresolved tail with exponent $\alpha_{\mathrm{tail}}(B) = p + 2q_B$. 
The appendix proves a sharper "efficiency theorem" (Corollary~\ref{cor:appendix_headmatched}): given two runs with the same matched early progress at an anchor rank, the run with the smaller (flatter) tail exponent $\alpha_{\mathrm{tail}}$ is guaranteed to reach any deeper task-relevant cutoff in fewer tokens. 
This establishes early spectral tails as a reliable mechanistic forecast for later efficiency, provided the comparison remains within an architectural family and a task-relevant fit window.

\item \textbf{Informal result 2: Task-band concentration tracks feature learning.}
The two-layer diagonal Fourier factor model makes feature learning visible as concentration of learned mass on the teacher-supported Fourier band.
In this model, each learned Fourier coefficient $m_r(t)=u_r(t)v_r(t)$ follows a nonlinear gradient-flow equation: modes with teacher coefficient $\beta_r>0$ grow toward their target value, while modes with $\beta_r=0$ decay.
Consequently, if the teacher signal lies in a task-relevant band $S$, the band-concentration statistic
\begin{equation}
H_S(t)=\frac{\sum_{r\in S} m_r(t)}{\sum_r m_r(t)}
\end{equation}
increases monotonically as long as both the in-band mass $\sum_{r\in S} m_r(t)$ and the off-band mass $\sum_{r\notin S} m_r(t)$ are nonzero (Appendix~\ref{sec:appendix_toy_theory}, especially Corollary~\ref{cor:appendix_hpeak}). This gives the theoretical reason to interpret increasing task-band concentration as feature learning in the toy model.

\item \textbf{Informal result 3: RoPE, Muon, and untied readouts act through distinct geometric mechanisms.}
The appendix proves an intervention-aligned mechanism results for the three toy-model tricks. RoPE restores the task's cyclic symmetry at the attention-score level: scores are shift-equivariant and depend on position only through relative offsets, while nontrivial absolute positional tables generically break this equivariance (Theorem~\ref{thm:appendix_rope_equiv} and Proposition~\ref{prop:appendix_absolute_breaks}). Muon changes the optimization geometry, not the model class: replacing a matrix gradient by its polar factor is steepest descent for an operator-norm trust region, equalizing singular directions and acting as spectral preconditioning (Theorem~\ref{thm:appendix_muon}). Untying the readout changes the model class: tied token-to-logit maps are confined to the embedding output subspace, while untied heads strictly enlarge the realizable class and remove this bottleneck (Theorem~\ref{thm:appendix_untied} and Corollary~\ref{cor:appendix_tied_lower_bound}). Thus faster concentration of task-relevant Fourier mass can arise from three sources: symmetry bias, matrix-update geometry, or output-factorization freedom.

\end{enumerate}

To connect the informal theory to an observable training process, we instantiate the cyclic task above with a small modular-arithmetic Transformer. The simulation uses a two-layer, four-head model and replays the cumulative prefix Baseline$\rightarrow$RoPE$\rightarrow$Muon$\rightarrow$Untied while sweeping batch size. Because the latent phase and Fourier coordinates are known by construction, this setting lets us measure activation spectra, Fourier-mode energy, and feature concentration directly rather than treating spectra only as black-box diagnostics. Full training, replay, and probe details are given in Appendix~\ref{sec:appendix_toy_setup}.

The simulation mirrors the larger-model phenomenon and then explains it in feature space.
As in the main experiments, different batch sizes can reach comparable objectives while retaining distinct spectra; the matched-loss batch-spectrum control in Appendix~\ref{fig:toy_appendix_matched_loss} shows this directly. Fig.~\ref{fig:toy} then asks whether those spectral differences correspond to feature recruitment. In Fig.~\ref{fig:toy}(a), the Muon local tail statistic reaches $\rho_S=1$ at roughly $20\%$ of training progress, showing that early spectral shape can already rank later token efficiency. Fig.~\ref{fig:toy}(b) tracks the same task-band concentration statistic used in Informal Result 2, $H_S(t)=\sum_{r\in S}m_r(t)/\sum_r m_r(t)$: larger values mean that a greater share of hidden-state Fourier energy lies in the task-relevant band, so its growth indicates increasingly selective Fourier feature learning. Fig.~\ref{fig:toy}(c) makes the same transition explicit, with the step-0 grey curve spread broadly across modes and the final untied curve developing sharp spikes on a small number of Fourier modes.
Together, these panels turn the spectral diagnostics from correlates of training into measurable signatures of task-aligned feature recruitment in the toy setting.

\section{Conclusion}
In this work we treat activation and gradient spectra as practical diagnostics for LLM
training. Across batch sweeps and an architectural intervention chain, the same two
measurements expose hidden batch-size regimes that scalar loss cannot distinguish,
predict token efficiency from early activation-tail shape, and separate learning-side
architectural gains from execution-side speedups. A modular-arithmetic toy model grounds
these signals in task-aligned feature learning. 

\paragraph{Limitations.} The current efficiency target is pretraining validation loss rather than downstream capability, so extending the same early-tail diagnostic to zero-shot or few-shot behavior remains an open question. Likewise, the present taxonomy is established on decoder-only cumulative intervention families and may require adaptation for MoE settings. Within this scope, the current analysis still compresses some detail: compact spectral summaries do not fully disentangle representation and systems changes in every variant, some early-prediction correlations are informative but still noisy,  the informative tail window is selected by a stable empirical protocol rather than a fully automatic estimator, and the toy models are best viewed as mechanistic abstractions rather than faithful models of language. Natural next steps are online batch-size or learning-rate control from spectral feedback, layer-resolved diagnostics of how spectral signatures propagate through depth, and automatic rules for locating the resolved-head / unresolved-tail boundary.

\section*{Acknowledgements}
We acknowledge helpful discussions with Atish Agarwala. This research was supported in part by the Yale Office of the Provost and experiments were run on the Yale Bouchet Cluster. E.~P was supported by an NSERC Discovery Grant RGPIN-2025-04643, an FRQNT--NSERC NOVA Grant, a CIFAR Catalyst Grant, an AFOSR grant and a gift from Google Canada.

\clearpage

\begingroup
\sloppy
\ifdefined\bibhang
\setlength{\bibhang}{0pt}
\fi
\bibliographystyle{unsrtnat}
\bibliography{neurips_2026}
\endgroup

\newpage
\appendix

\section{Additional Related Work}
\label{sec:appendix_related_work}
The main paper uses spectra as operational diagnostics, so the closest prior work spans representation geometry, optimization spectra, spectral learning theory, and simplified models of algorithmic feature learning.

\paragraph{Representation spectra and effective-rank summaries.}
RankMe~\citep{garrido2023rankme} motivates the entropy effective rank as a compact unsupervised summary of representation spread, while $\alpha$-ReQ~\citep{agrawal2022alphareq} uses eigenspectrum decay as a representation-quality signal in self-supervised learning. Covariance and dimensionality summaries also appear implicitly in redundancy-reduction objectives such as Barlow Twins and VICReg~\citep{zbontar2021barlow,bardes2022vicreg}, where representation collapse or excessive concentration is treated as a training pathology. In language models, anisotropy and low-dimensional dominant directions have long been observed in contextual embeddings~\citep{ethayarajh2019contextual,timkey2021rogue}; these results motivate reporting the whole spectrum, or at least several spectral summaries, rather than relying only on raw cosine geometry. More recent work uses covariance eigendecay and geometry phases to study representation changes during pretraining and post-training~\citep{li2025tracing,chang2022geometry}. The most directly related activation-spectrum precursor is \emph{Evolution of the Spectral Dimension of Transformer Activations}~\citep{liu2025evolution}, which studies heavy-tailed activation spectra and spectral-exponent evolution across layers and training. Our use is complementary: we test whether early and matched-loss spectra predict batch efficiency and distinguish architectural interventions.

\paragraph{Optimization-side spectra and heavy-tail views.}
Optimization spectra provide a complementary lens on the update path. Prior work shows that gradient descent can concentrate in a low-dimensional subspace~\citep{gurari2018subspace}, stochastic-gradient covariance can exhibit power-law structure~\citep{xie2023stochastic}, and Hessian spectra contain informative outliers tied to data and class structure~\citep{ghorbani2018hessian,papyan2019hessian}. Heavy-tailed random-matrix analyses likewise connect spectral decay exponents to optimization regimes and batch-size effects~\citep{mahoney2019heavy}, while scaling-law theory connects spectral structure to learning curves and generalization~\citep{bahri2024explaining,sharma2022manifold}. We therefore treat activation and gradient spectra as related but non-interchangeable measurements: one describes represented variance, the other describes update concentration.

\paragraph{Spectral learning theory and random-feature models.}
The toy analysis starts from linearized gradient flow because, in kernel and random-feature regimes, eigenvalues of the feature covariance or kernel operator directly control learning rates across target components. Random Fourier features provide a controlled finite-dimensional approximation to shift-invariant kernels~\citep{rahimi2007}, and later learning-curve analyses make the dependence on kernel spectra and task alignment explicit~\citep{bordelon2020spectrum,canatar2021spectral,spigler2019asymptotic}. Recent nonlinear spiked-covariance theory studies related signal-propagation and feature-selection questions~\citep{wang2024spiked}. This line of work motivates the appendix's first toy step: before studying transformers, isolate how a spectrum over Fourier modes shapes which task components are learned early, late, or not at all.

\paragraph{Fourier mechanisms in grokking and modular arithmetic.}
Modular arithmetic is a natural second toy setting because the cyclic group has an explicit Fourier basis. The original grokking experiments showed delayed generalization on small algorithmic datasets, including modular tasks~\citep{power2022grokking}. Subsequent work connected these dynamics to structured representation learning and phase diagrams~\citep{liu2022towards}, reverse-engineered modular-addition transformers into Fourier-space circuits with continuous progress measures~\citep{nanda2023progress}, and studied interpretable two-layer solutions for modular arithmetic~\citep{gromov2023grokking}. Recent analyses of two-layer modular-addition networks further emphasize single-frequency Fourier features, phase alignment, and gradient-flow competition among frequencies~\citep{he2026mechanism}. Other grokking studies frame delayed generalization through competing subnetworks or early spectral signatures of the learning curve~\citep{merrill2023tale,notsawo2023predicting}. Our toy-model sequence follows this literature but uses it for a narrower purpose: to check whether the spectral diagnostics used in language-model experiments track task-aligned feature recruitment, rather than only measuring black-box covariance concentration.

\section{Model Training Setup}
\label{sec:appendix_training_setup}
This section elaborates on the architectural tricks, model cards, dataset choices, and learning-rate selection protocol used by the language-model experiments. The suite follows a cumulative intervention path: each later label inherits the implementation choices of the previous label unless the row states a new change.

\subsection{Architectural tricks and variants}
\label{sec:appendix_variant_attribution}
Table~\ref{tab:appendix_variant_integrated} integrates the intervention descriptions, external provenance, and mechanism-level grouping used throughout the paper. The table is intentionally descriptive rather than code-level: it names what changes in the model or optimizer and how that change should be interpreted in the spectral analysis.

\begin{table*}[t]
\centering
\scriptsize
\setlength{\tabcolsep}{10pt}
\begin{tabular}{@{}p{0.10\textwidth}p{0.17\textwidth}p{0.43\textwidth}p{0.22\textwidth}@{}}
\toprule
\textbf{Stage} & \textbf{Group} & \textbf{Incremental change} &\textbf{ Provenance / reference} \\
\midrule
Baseline & Reference trunk & GPT-2-small-style decoder with learned absolute positions, tied token embedding and LM head, dense causal attention, RMS normalization, and squared-ReLU MLPs. & GPT-2~\citep{radford2019language}; modded-NanoGPT record~\citep{modnanogpt_record1_baseline}. \\
RoPE & Positional control & Replaces learned positions with rotary embeddings and explicit query/key/value projections; normalizes query/key states before rotation. & RoFormer~\citep{su2021roformer}; record~\citep{modnanogpt_record2_rope}. \\
Muon & Optimizer & Keeps the RoPE trunk but changes hidden-layer matrix updates to Muon; embeddings, LM head, and scalar/control parameters remain being trained with Adam. & Muon~\citep{jordan2024muon}; record~\citep{modnanogpt_record3_muon}. \\
Untied & Parameterization & Unties the token embedding and LM head and assigns them separate optimizer groups. & Record~\citep{modnanogpt_record8_untied}. \\
ValueMix & Value pathway & Adds cross-layer value mixing through learned interpolation between the current value tensor and previous-layer value state. & Record~\citep{modnanogpt_record9_valuemix}. \\
U-Net & Depth routing & Adds encoder--decoder-style skip connections across the 12-layer stack. & U-Net~\citep{ronneberger2015unet}; record~\citep{modnanogpt_record11_unet}. \\
FixedWin & Attention geometry & Replaces dense causal attention with document-aware local FlexAttention over a fixed horizon. & Record~\citep{modnanogpt_record12_fixedwin}. \\
FlexWin & Attention geometry & Warms the local attention horizon during training instead of using one fixed window throughout. & Record~\citep{modnanogpt_record13_flexwin}. \\
VTE & Value pathway & Adds a dedicated value-token embedding pathway injected into the value stream. & Record~\citep{modnanogpt_record14_vte}. \\
BetterWin & Attention/value geometry & Refines the VTE/windowed trunk with split value embeddings, block sliding windows, and separate block masks. & Record~\citep{modnanogpt_record16_betterwin}. \\
SparseV & Value sparsity & Replaces dense per-layer value-token embeddings with a sparse reusable embedding pattern. & Record~\citep{modnanogpt_record17_sparsev}. \\
TruncRoPE & Positional shaping & Applies rotary phase information to only a subset of each head dimension. & Record~\citep{modnanogpt_record17_sparsev}. \\
SoftCap & Output shaping & Applies bounded logit soft-capping before the loss. & Gemma 2~\citep{gemma2024}; record~\citep{modnanogpt_record18_softcap}. \\
FP8Head & Output/system path & Moves LM-head computation to an FP8 matrix-multiply path with explicit scaling. & FP8 formats~\citep{micikevicius2022fp8}; record~\citep{modnanogpt_record19_fp8head}. \\
LSWA & Attention geometry & Uses paired long and short attention masks assigned across layers. & Gemma 2~\citep{gemma2024}; record~\citep{modnanogpt_record20_lswindow}. \\
AttnScale/SubLR & Schedule controls & Rescales attention scores or selected parameter learning rates while keeping the late trunk fixed. & Record family~\citep{modnanogpt_record20_lswindow}. \\
\bottomrule
\end{tabular}
\caption{\textbf{Cumulative intervention catalog with mechanism groupings.}
The 16 named variants form a single incremental chain: each row inherits every
prior row's changes and describes only its own delta. The \textbf{Group} column
encodes the mechanism-level axis used in the spectral taxonomy of
Section~\ref{sec:taxonomy}: positional control, optimizer, parameterization,
value pathway, depth routing, attention geometry, and output/system path; allowing
each transition's spectral signature to be read back against its architectural role.
Provenance and contributor attributions follow the modded-NanoGPT record chain.}
\label{tab:appendix_variant_integrated}
\end{table*}

\subsection{Model card and parameter counts}
Table~\ref{tab:appendix_model_tiers} gives the canonical tier notation used throughout the paper: d12, d36, and d48 refer to 12-, 36-, and 48-layer Transformer profiles, respectively. Table~\ref{tab:appendix_model_card_params} then summarizes the model families and parameter counts used in the paper. Counts include registered trainable parameters and exclude buffers, masks, cached rotary tables, and other runtime state. All counts use the padded training vocabulary size of 50{,}304 tokens. Optimizer-only controls, such as the LSWA Adam comparison, therefore share the same parameter count as the corresponding architectural family.

\begin{table*}[t]
\centering
\scriptsize
\setlength{\tabcolsep}{3pt}
\begin{tabular}{@{}p{0.10\textwidth}p{0.08\textwidth}p{0.12\textwidth}p{0.15\textwidth}p{0.12\textwidth}p{0.13\textwidth}p{0.22\textwidth}@{}}
\toprule
Tier & Layers & Heads / width & Parameter range & Context length & Dataset split & Representative runs \\
\midrule
d12 prefix & 12 & 6 / 768 & 123.57M--162.20M & 1{,}024 & FineWeb-100B & Baseline, RoPE, Muon, Untied. \\
d12 trunk & 12 & 6 / 768 & 162.20M--625.80M & 65{,}536 & FineWeb-10B & ValueMix through AttnScale. \\
d36 support & 36 & 20 / 1280 & 836.57M--2.00B & 32{,}768 & FineWeb-10B & FlexWin, BetterWin, SparseV, AttnScale. \\
d48 support & 48 & 25 / 1600 & 1.87B--3.57B & 32{,}768 & FineWeb-10B & BetterWin and SparseV scale follow-ups. \\
\bottomrule
\end{tabular}
\caption{\textbf{Canonical model-tier notation.} The paper names scale primarily by layer count: d12, d36, and d48 denote the 12-, 36-, and 48-layer profiles. Parameter ranges vary because architectural families differ in value-token embeddings, sparse value pathways, and output parameterization.}
\label{tab:appendix_model_tiers}
\end{table*}

\begin{table*}[t]
\centering
\small
\setlength{\tabcolsep}{3pt}
\begin{tabular}{@{}p{0.18\textwidth}p{0.22\textwidth}p{0.20\textwidth}p{0.14\textwidth}p{0.16\textwidth}@{}}
\toprule
Family & Representative variants & Shape & Parameters & Training role \\
\midrule
Baseline learned-position tied & Baseline & 12 layers / 6 heads / width 768 & 124.35M & Short-context reference. \\
RoPE tied & RoPE, Muon & 12 / 6 / 768 & 123.57M & Positional and optimizer prefix. \\
Untied RoPE & Untied & 12 / 6 / 768 & 162.20M & Output-parameterization control. \\
ValueMix & ValueMix & 12 / 6 / 768 & 162.20M & Cross-layer value mixing. \\
U-Net/fixed/flex & U-Net, FixedWin, FlexWin & 12 / 6 / 768 & 162.20M & Main local-attention trunk before VTE. \\
Full VTE & VTE & 12 / 6 / 768 & 625.80M & Dense value-token embedding pathway. \\
Half VTE / BetterWin & BetterWin & 12 / 6 / 768 & 394.00M & Split/reused value-token embedding pathway. \\
Sparse family & SparseV, TruncRoPE, SoftCap, FP8Head, LSWA, LSWA-Adam, SubLR & 12 / 6 / 768 & 275.74M & Late sparse value and output-shaping trunk. \\
Scaled U-Net/flex & FlexWin d36 & 36 / 20 / 1280 & 836.57M & Scale robustness for the local-attention trunk. \\
Scaled BetterWin & BetterWin d36 / d48 & 36 / 20 / 1280 or 48 / 25 / 1600 & 2.00B / 3.57B & Largest half-VTE robustness runs. \\
Scaled sparse family & SparseV d36 / d48, AttnScale d36 & 36 / 20 / 1280 or 48 / 25 / 1600 & 1.02B / 1.87B & Larger sparse-trunk support. \\
\bottomrule
\end{tabular}
\caption{\textbf{Model-card summary with parameter counts.} The 12-layer experiments use GPT-2-small-scale width unless otherwise noted. The d36 and d48 follow-ups use the larger long-context profiles. Scaled sparse-family counts use the runtime-resolved sparse attention pattern used in training. The largest completed robustness run reported in the paper is BetterWin d48 at 3.57B parameters.}
\label{tab:appendix_model_card_params}
\end{table*}

The scaled full-VTE profile is larger than the half-VTE BetterWin profile, but the completed scale-robustness analysis reported here uses BetterWin and sparse-family follow-ups. We therefore reserve the ``largest run'' statement for the completed BetterWin d48 experiment rather than for every model profile that can be instantiated by the architecture.

\subsection{Training specifics}
The short-context prefix variants use sequence length $1{,}024$ and the 100B-token FineWeb sample for training; the long-context variants uses sequence length $65{,}536$ for the 12-layer runs and the 10B-token split for training. The d36 and d48 robustness runs use the long-context family with sequence length $32{,}768$. All 12-layer runs are trained on a single H200; the d36/d48 follow-up uses a single B200. Unless stated otherwise, the main d12 comparisons target validation loss 3.2 on a fixed held-out validation set drawn from the FineWeb-10B validation split, which is reused across the 100B/10B training-data bridge. Token efficiency is measured by the number of training tokens consumed before the first checkpoint meeting that target.

FineWeb is a cleaned English CommonCrawl corpus drawn from 96 dumps spanning 2013 through April 2024 and containing approximately 15 trillion GPT-2 tokens~\citep{lozhkov2024fineweb}. The two sampled splits are used for practical reasons: less token-efficient early variants are run on the 100B-token sample, while later variants use the 10B-token split. To verify that this switch does not introduce a spectral confound, we compare $1{,}024$ windows of $1{,}024$ tokens from the local FineWeb-10B shard against the Hugging Face FineWeb sample-100BT split. As shown in Fig.~\ref{fig:appendix_fineweb_compare}, the covariance spectra nearly overlap: the Jensen--Shannon divergence between the two splits is $4.31 \times 10^{-4}$, well below the run-to-run variability observed in the same spectral diagnostic. This supports treating the two data samples on a common spectral scale while explicitly noting the protocol difference.

\paragraph{Spectral measurement protocol.}
Spectra are extracted from deterministic prefixes of a held-out validation shard and reused across checkpoints and runs within each family. From FixedWin onward in the main d12 trunk, we use $T=65{,}536$, 512 validation sequences for activation covariance, and 512 validation sequences for per-sample gradient SVD, with covariance batch size=2 and gradient batch size=2; activation covariance therefore uses $N = 512 \times 65{,}536$ token-position rows. Before FixedWin, the short-context prefix uses the equivalent activation budget $32{,}768 \times 1{,}024$ and 512 gradient samples. Short-context targets are formed by shifting within a sequence and masking the last position, whereas the long-context trunk reads length-$(T+1)$ chunks so that each input token has a true next-token target.

\begin{figure}[t]
  \centering
  \includegraphics[width=0.6\textwidth]{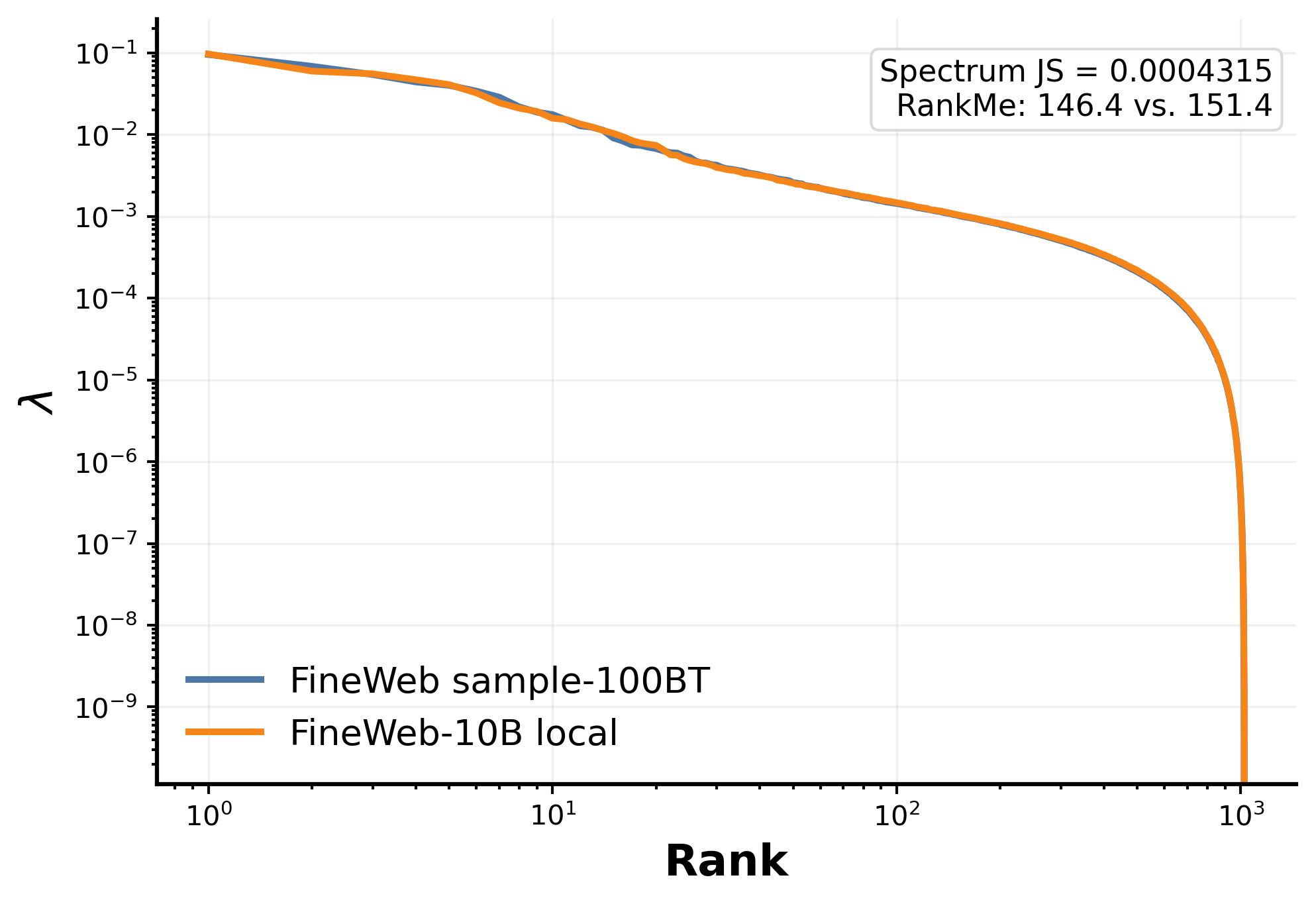}
  \caption{\textbf{FineWeb-10B and FineWeb-100B samples have nearly identical token-window spectra.} The plot compares trace-normalized covariance spectra from matched 1{,}024-token windows. The very small spectral Jensen--Shannon divergence supports treating the data switch as a minor spectral confound relative to the batch and architecture effects analyzed in the paper.}
  \label{fig:appendix_fineweb_compare}
\end{figure}

\paragraph{Learning-rate selection.}
Learning rates are selected separately for each model family and effective batch tier.
The reference tier matches the tier-8 modded-NanoGPT token batch: 8 sequences for
65{,}536-token runs and 512 sequences for the 1{,}024-token prefix runs, with reference
rates taken from the corresponding modded-NanoGPT configuration. Concretely, the tied
AdamW prefix uses $6\times10^{-4}$; the Muon prefix uses embedding and hidden-matrix
rates $(3.6\times10^{-3},\,3.6\times10^{-4})$; the untied prefix uses
$(0.3,\,0.002,\,0.02)$ for embedding, LM-head, and Muon; the long-context trunk uses
$(0.6,\,0.008,\,0.04,\,0.04)$ for embedding, LM-head, hidden-matrix, and
scalar/control parameters. The LSWA Adam control keeps the same embedding, head, and
scalar rates but sets the Adam matrix rate to $0.004$.

To scale to a target tier $B$ from reference $B_0$, we form scaling centers
$\sqrt{B/B_0}$ and $B/B_0$ and evaluate each at local multipliers
$\{0.5,\,1/\!\sqrt{2},\,1,\,\sqrt{2},\,2\}$, deduplicating coincident candidates.
Each multiplier is applied uniformly across all optimizer groups, preserving relative
rates between parameter classes while varying the global step size.

All d12 variants sweep over $\{1,2,4,8,16,32\}$; d36 uses $\{2,4,8,16\}$ and d48
uses $\{1,2,4,8\}$, subject to hardware constraints. Candidates are filtered by a
synchronized successive-halving procedure~\citep{li2020system}: all candidates train
to 50M tokens, survivors extend to 150M, and remaining candidates extend to 500M,
with validation loss evaluated every 20M tokens. Promotion retains the top half at
each rung, always keeping at least one survivor per family/tier pair. Power-law
extrapolations to 1B tokens are recorded as diagnostics but do not drive promotion,
which is based on observed validation loss at each rung. Final candidates were
selected under the top-half rule throughout; a top-third rule used in early pilots
was abandoned before the main sweep. The surviving configuration for each family and
batch tier is then used for the final constant-loss spectral runs---a necessary step,
since matched-loss spectral comparisons are only meaningful when each tier trains
under its own well-tuned schedule.

\section{Per-Sample Gradient Computation}
\label{sec:appendix_gradient_computation}
\label{sec:appendix_grad_matrix_choice}
For a hidden activation matrix $H\in\mathbb{R}^{N\times d}$ collected from a fixed measurement pool, we compute the centered covariance $C=N^{-1}(H-\bar H)^\top(H-\bar H)$ and analyze its eigenvalue spectrum after trace normalization. RankMe is the entropy effective rank,
\begin{equation}
\operatorname{RankMe}(C)=\exp\left(-\sum_j p_j\log p_j\right),\qquad p_j=\lambda_j/\sum_k\lambda_k,
\end{equation}
where $\lambda_j$ are the covariance eigenvalues. We also fit band-restricted power-law slopes $\lambda_j\propto j^{-\alpha}$ over head or tail rank windows; these exponents summarize whether variance is concentrated in leading modes or spread through a heavy tail.

For gradient spectra, choose a trainable weight matrix $W\in\mathbb{R}^{d_{\rm out}\times d_{\rm in}}$ and compute one loss gradient per sample from a fixed held-out pool of validation sequences, where one \emph{sample} means one validation sequence. Writing $g_m=\operatorname{vec}(\nabla_W \ell_m)$ for the gradient of the sample-mean autoregressive loss on sample $m$, we stack the rows into
\begin{equation}
G=\begin{bmatrix}
g_1^\top\\
\cdots\\
g_M^\top
\end{bmatrix}\in\mathbb{R}^{M\times d_{\rm out}d_{\rm in}},
\end{equation}
and analyze the singular values of $G$. This matrix is a tensor-specific view of update concentration, not a global invariant of the architecture. In the main text we standardize on the deepest saved attention-output projection because it is broadly available and directly measures attention writeback into the residual stream.

\begin{figure}[t]
  \centering
  \includegraphics[width=0.95\textwidth]{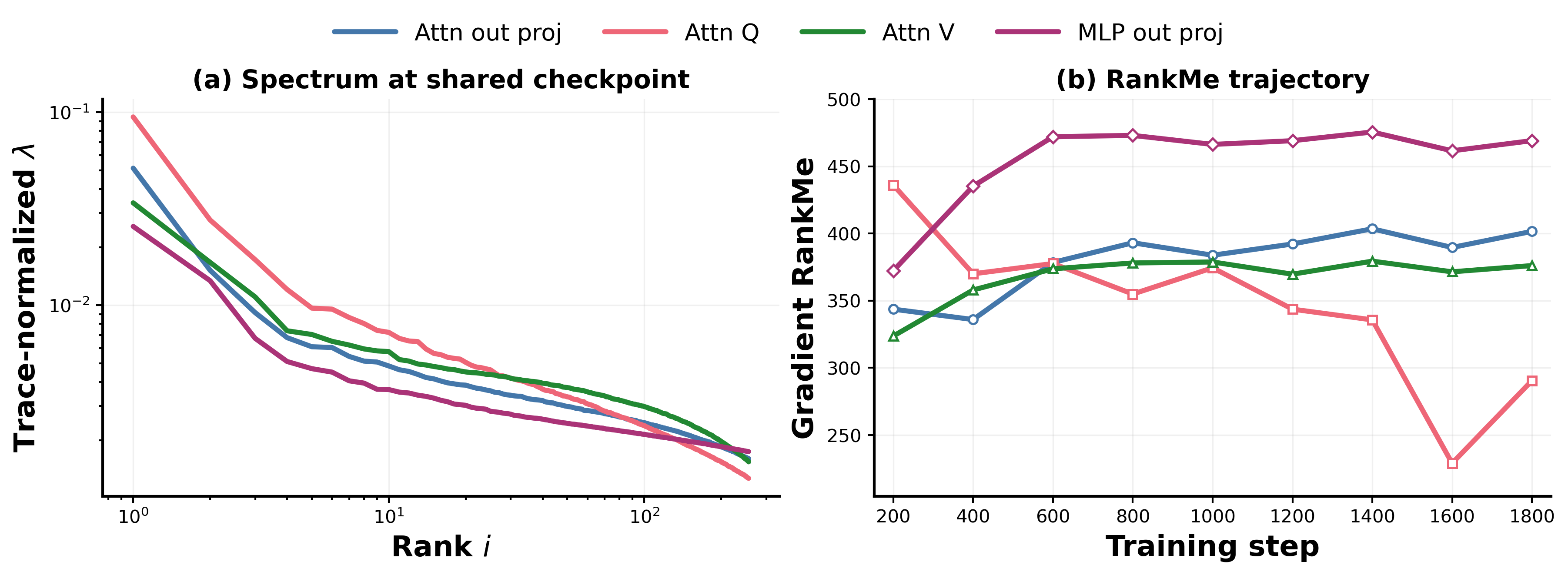}
  \caption{\textbf{Gradient spectra depend on the probed weight matrix.} Representative matrix-level comparison from the d12 BetterWin bs8 run at layer~11. The query projection, value projection, attention-output projection, and MLP-output projection produce different concentration levels and RankMe trajectories. This is why gradient spectra are interpreted as tensor-specific complements to activation spectra rather than architecture-level summaries.}
  \label{fig:appendix_grad_matrix_choice}
\end{figure}

Fig.~\ref{fig:appendix_grad_matrix_choice} shows the matrix-choice dependence. Each weight matrix exposes a different stage of the attention computation. Query and key
gradients reflect how the model is revising its \emph{attention routing}, which positions or feature directions a token should attend to or be attended by, upstream of any content movement. Value gradients reflect how the model adjusts \emph{what content gets aggregated} once routing is fixed: a large gradient here means the read-out per attended position is changing, not the selection pattern itself. Attention-output gradients capture something distinct from both: $W_O$ is the matrix that projects the concatenated multi-head outputs back into the residual stream, making it the unique write gate through which all attended content must pass before influencing downstream computation. Its gradient therefore summarizes the attention layer's net contribution to the residual stream after routing and
content selection have both been applied. MLP-output gradients play the analogous role for the feed-forward pathway.

We use $W_O$ of the final attention block as our primary gradient probe for two reasons. First, it is the natural complement to the activation covariance spectrum: activation spectra answer \emph{what has been learned} (the state of the residual stream), while
$W_O$ gradients answer \emph{how the attention layer is currently updating that stream},
giving the dual-view its interpretive coherence. Second, $W_O$ is the most stable probe
across the intervention chain. Unlike $W_V$, whose shape and semantics shift when value
pathways are restructured (ValueMix, VTE, BetterWin), and unlike $W_Q$/$W_K$, whose
routing role changes with attention geometry (FixedWin, FlexWin, LSWA), $W_O$ retains
the same architectural role, attention writeback to the residual stream, across all
16 variants. This invariance is what allows gradient spectra to serve as a
consistent comparative diagnostic along the full intervention chain. As
Fig.~\ref{fig:appendix_grad_matrix_choice} confirms, the qualitative diagnostic
conclusions are robust to this choice; the different probes differ in concentration
level and tail shape but not in their implied taxonomic groupings.

\section{Muon/Adam Comparison}
\label{sec:appendix_muon_adam}
Most later variants use a mixed optimizer: Adam-style updates for embeddings, the LM head, and scalar/control parameters, and Muon updates for hidden-layer matrix parameters. The batch-size phenomenon studied in Section~\ref{sec:batch}, point 1 is therefore not automatically a property of Muon alone. To isolate this, we train the LSWA variant under both the standard mixed-Muon setup and an Adam matrix-update variant, keeping all other hyperparameters fixed.

\begin{figure}[t]
  \centering
  \includegraphics[width=0.96\textwidth]{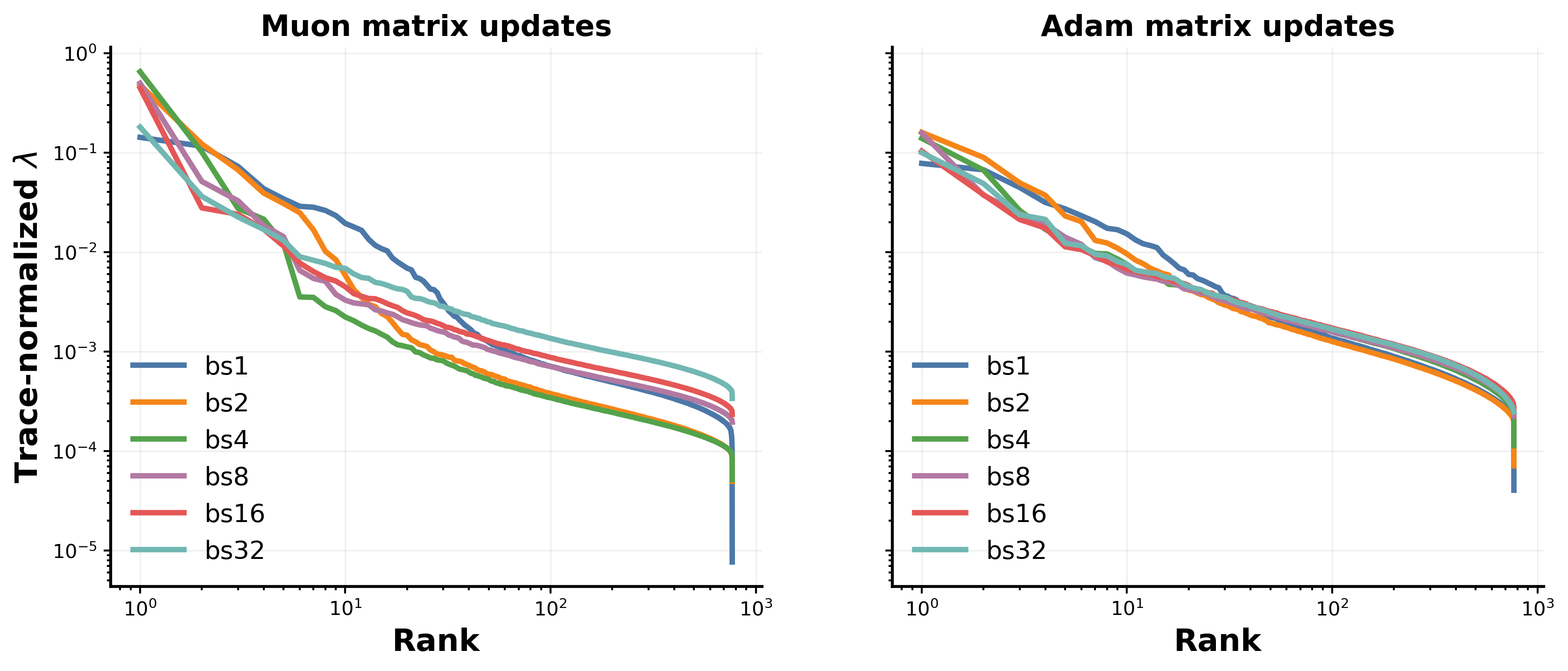}
  \caption{\textbf{Batch-dependent activation spectra appear under both Muon and Adam matrix updates.} Each panel shows final layer-11 activation covariance spectra for LSWA across effective batch tiers. The Adam variant still shows batch-dependent spectral separation, so the hidden-regime effect is not intrinsic to Muon. The separation is stronger and more sharply structured in the Muon runs, consistent with
  Muon changing the geometry of hidden-layer matrix updates.}
  \label{fig:appendix_lswa_adam_muon}
\end{figure}

Fig.~\ref{fig:appendix_lswa_adam_muon} shows that the Adam control preserves the key qualitative result: equal architecture and different batch tiers still produce measurably different activation spectra. Muon amplifies this separation, particularly in the leading modes, but its existence under Adam confirms that effective batch size is a latent determinant of representation geometry, not merely an artifact of the
matrix optimizer choice.

The amplified separation under Muon admits a natural mechanistic explanation from Theorem~\ref{thm:appendix_muon}. Muon's update is the polar factor $Q(G)=UV^\top$ of the gradient matrix, which equalizes every nonzero singular direction to unit magnitude. This scale-balancing is exactly what gives Muon its preconditioning advantage, but it also makes the \emph{direction} of each update sensitive to the singular structure of the empirical gradient $G$ itself. At small effective batch size, $G$ contains substantial mass in noise singular directions whose magnitudes scale as $\mathcal{O}(1/\sqrt{B})$, and orthogonalization treats these directions as equivalent to signal directions rather than down-weighting them by their amplitude. As $B$ grows, the noise singular values shrink relative to signal, and $Q(G)$ converges toward the polar factor of the population gradient. Adam's per-coordinate normalization, by contrast, rescales magnitudes coordinatewise but does not redistribute mass between singular directions at all, so its update direction depends less sharply on $B$. Concurrent work corroborates this picture by showing that Muon admits a substantially larger critical batch size than Adam under matched recipes, so the tiers swept here span more of Muon's rising compute-time curve and produce correspondingly larger trajectory differences.

Two practical implications follow. First, batch-tier selection is more
consequential under Muon than under Adam: an Adam-tuned batch regime does
not transfer cleanly, and the apparent within-architecture variance attributable to batch choice is larger when the matrix optimizer is Muon. Second, the spectral diagnostics introduced in this paper have higher discrimination value under Muon, which is consistent with the sharper batch separation visible throughout the main figures. The amplified batch dependence is not a pathology of Muon but a direct geometric consequence of replacing magnitude-aware updates with direction-equalizing ones.

\section{Predictive Ablations}
\label{sec:appendix_predictive_ablations}
This section tests how robust the early-prediction story is to training progress, random seed, and probe layer. These ablations are intentionally narrower than the main result: they ask whether the diagnostic survives plausible measurement choices rather than claiming that one scalar is universally optimal.

\begin{figure}[t]
  \centering
  \includegraphics[width=\textwidth]{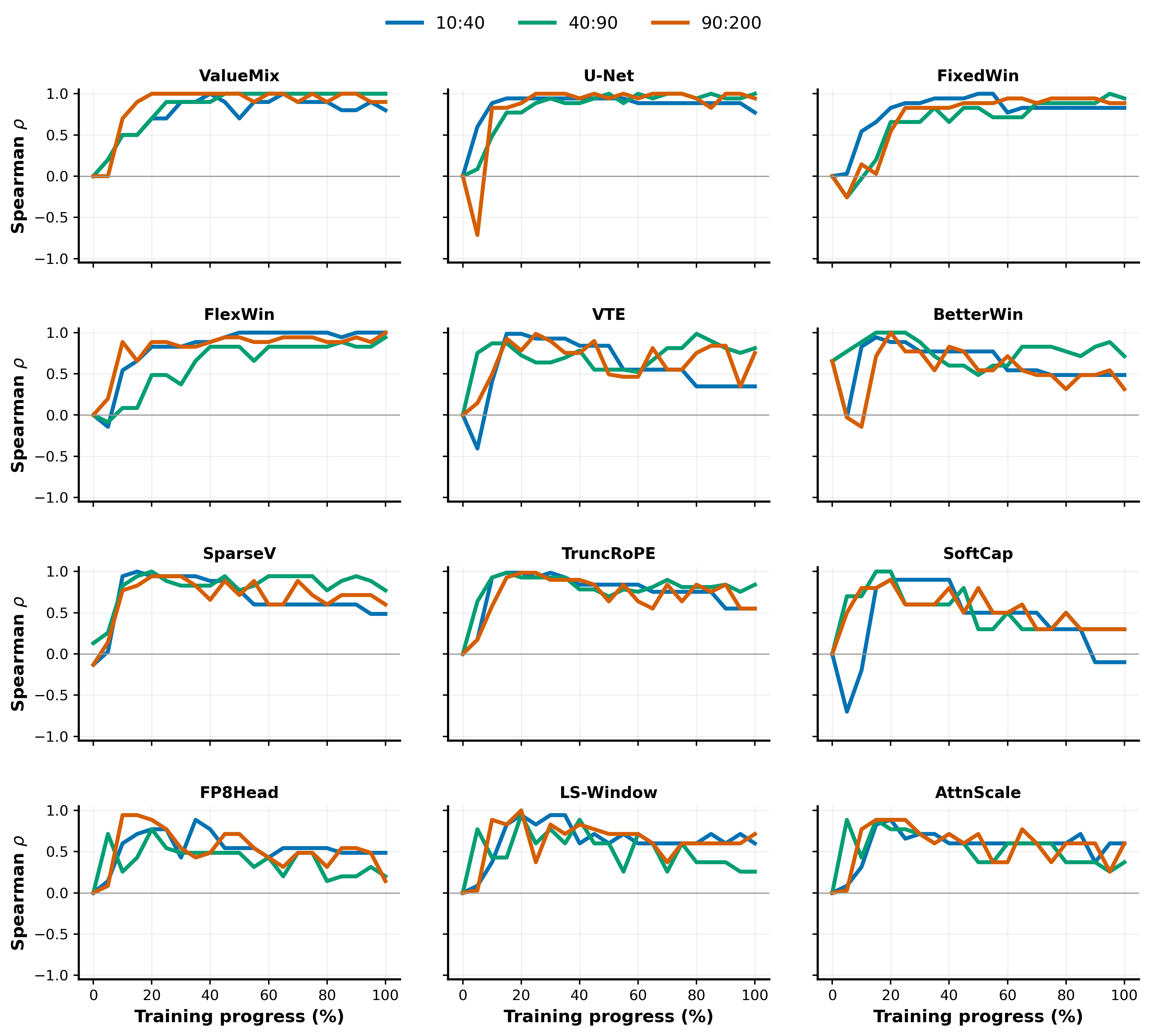}
    \caption{\textbf{Early-prediction strength evolves with training progress for d12 variants.} Each small panel shows one d12 model family. The colored lines report Spearman correlation between final token efficiency and local activation-spectrum exponents fit over rank windows 10--40, 40--90, and 90--200. The deeper-tail window 90--200 typically reaches a high positive correlation by about 20\% of training, often saturating near $\rho_S=1$, while shallower windows are more variant-dependent.}
  \label{fig:appendix_spearman_progress}
\end{figure}

\subsection{Spearman correlation over training percentage}
Fig.~\ref{fig:appendix_spearman_progress} asks when the early-prediction signal becomes visible across the d12 variant chain. Each panel fixes one d12 model family and plots the Spearman correlation between a local activation-spectrum exponent and the final token-efficiency ranking of that family's batch tiers. The three curves correspond to power-law fits over progressively deeper rank windows of the trace-normalized activation
covariance spectrum: ranks $10$--$40$ probe the upper modes just past the leading head, ranks $40$--$90$ probe an intermediate band, and ranks $100$--$200$ correspond to the informative tail window used by our main early-prediction analysis (Eq.~3).

A unifying observation across all panels is that the deepest-window correlation
\emph{peaks near roughly $20\%$ of training progress}, regardless of variant. This holds even for the variants where the signal is otherwise weak or noisy: FP8Head, LSWA, and AttnScale all attain their highest $\rho_S$ near the same early window
before drifting downward over the rest of training. The $20\%$ checkpoint is therefore not a feature of "easy" variants---it is a property of the diagnostic itself, and a practical one: once a variant has produced its peak early-tail signal, additional
training does not improve and may erode the diagnostic. This is consistent with the mechanistic picture in Section~\ref{sec:toy}, where early tail shape reflects which slow modes are still being recruited; once recruitment completes, that informative window closes. Beyond this shared peak time, the panels separate into three regimes that map cleanly onto the architectural groups in Table~2 and the taxonomy of Section~\ref{sec:taxonomy}.

\paragraph{Clean monotone signal in single-mechanism variants.} In ValueMix, U-Net, FixedWin, FlexWin, SparseV, and TruncRoPE, the deepest-window correlation rises smoothly, reaches its maximum around $20\%$ progress, and \emph{holds} near that maximum for the
remainder of training. These variants share a structural property: each introduces a single, uniformly-applied change---one new value-mixing rule, one skip pattern, one window size, one sparsity pattern, or one rotary truncation---so the activation tail
develops along a single timescale and the early shape stably ranks batch tiers throughout training.

\paragraph{Delayed-onset inversion in multi-component or amplitude-dependent variants.}
In VTE, BetterWin, and SoftCap, the deepest-window correlation begins strongly negative and climbs to positive values during training, again peaking near $20\%$ progress. Each of these variants introduces a structural mechanism whose effect on the tail emerges with delay rather than from initialization. VTE injects a dedicated value-token embedding pathway that must itself be learned before its spectral signature stabilizes; BetterWin simultaneously layers split value embeddings, block sliding windows, and separate block masks, requiring the three components to differentiate across blocks;
and SoftCap's bounded $\tanh$ on logits is amplitude-dependent, behaving nearly linearly until logits enter the saturation regime later in training. In all three cases the early tail does not yet reflect the intervention's eventual geometry, producing transient anti-correlation that resolves as the new component specializes. This is consistent with the two-component dynamics analyzed in Theorem~11 (two-layer Fourier
factor model), where mode growth is nonlinear and informative tail structure emerges only after a feature pathway has begun to recruit. Notably, in several of these panels the shallow $10$--$40$ window reaches a high correlation \emph{before} the deeper window does, mirroring the head-then-tail learning order predicted by Theorem~6.

\paragraph{Peak-then-decay in throughput-leaning variants.} FP8Head, LSWA, and AttnScale also peak around $20\%$ progress, but their correlations decay back toward $\rho_S \approx 0.5$ over the rest of training rather than holding. These are precisely
the variants classified as throughput-leaning in Section~\ref{sec:taxonomy}: FP8Head changes the LM-head's
numerical precision rather than its geometry; LSWA assigns different mask types to
different layers, so a single-layer probe sees only one mask regime; and AttnScale
rescales attention scores or selected learning rates while keeping the late-trunk
geometry fixed. None of these interventions reshape the late-trunk feature geometry
that the layer-11 probe measures, so the small early-tail signal that does exist is
gradually overwritten by execution-side noise as training progresses. Under this view
the decay is consistent with, rather than contrary to, our framework: the diagnostic
correctly reports low \emph{stable} discrimination when the underlying token-efficiency
gap is small, while still recovering the canonical $20\%$-peak behavior shared with the
rest of the chain.

Together, the three regimes turn Fig.~\ref{fig:appendix_spearman_progress} from a
robustness check into structural validation of the Section~\ref{sec:taxonomy} taxonomy: every variant
shares the same early-prediction time horizon, but only those whose architectural change
reshapes late-trunk feature geometry retain a stable diagnostic afterward. The
$20\%$-checkpoint rule for measuring $\alpha_{\mathrm{tail}}$ is therefore well-founded
across the entire intervention chain, while persistence of the signal beyond that point
is itself a useful secondary indicator of whether the intervention is learning-side or
execution-side.

\subsection{Random seed experiments}
Fig.~\ref{fig:appendix_seed_robustness_compact} provides a focused robustness check for FlexWin tier-16 configuration, comparing the original seed against
two additional seeds. All three panels support the same conclusion. The validation-loss curves are essentially indistinguishable throughout training, confirming that the three runs reach the same loss trajectory rather than
merely converging at a shared late checkpoint. At the common step-3000 checkpoint, both the activation covariance spectra and the gradient SVD spectra stack so tightly across seeds that the curves are nearly coincident---seed-to-seed variation is negligible relative to the scale of cross-tier separation visible in the main figures.
This is the relevant comparison: the batch-tier differences reported throughout the
paper are not an artifact of a single lucky or unlucky initialization, but reflect a
structural property of the training regime that is stable across independent runs.

\begin{figure}[t]
  \centering
  \includegraphics[width=0.96\textwidth]{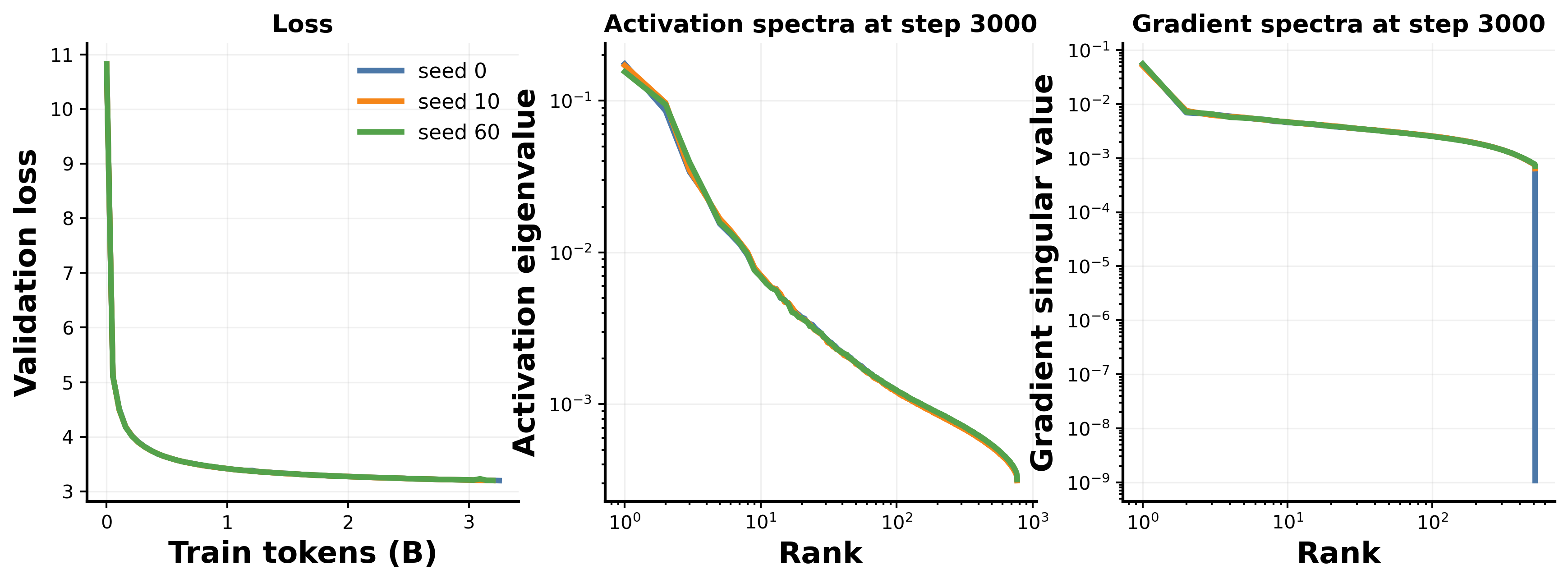}
  \caption{\textbf{FlexWin tier-16 spectra are stable across random seeds.} The left panel shows validation-loss curves; the middle and right panels show activation and gradient spectra at the shared step-3000 checkpoint. The small seed-to-seed variation supports treating the batch-tier separation as larger than ordinary seed noise in this setting.}
  \label{fig:appendix_seed_robustness_compact}
\end{figure}

\subsection{Layerwise ablation}
In this ablation we ask whether the final-layer activation probe used in the main early-prediction analysis is a principled measurement choice, or whether an earlier probe layer would serve equally well. For each d36 run, we select the saved
checkpoint at 0.21B training tokens and refit the tail exponent over ranks $[200,400]$, matching the predictive window used in the main figures. 

Fig.~\ref{fig:appendix_layerwise_prediction} shows the Spearman correlation between this early tail exponent and tokens-to-target as a function of probe layer, for the
three available d36 families. The pattern is consistent: early and middle layers
carry weak or sign-inconsistent signal, while the final stored probe layer (L35) shows
the strongest positive correlation in all three families. This late-layer advantage
is visible directly in the heatmap, correlations tend to become more positive moving
from L00 toward L35, and supports the convention of probing the deepest available
activation layer in the main experiments. We treat this as supporting evidence rather
than a universal layer-selection rule: the evidence is limited to three d36 families
and does not cover every d12 variant or every architectural regime in the chain.

\begin{figure}[t]
  \centering
  \includegraphics[width=\textwidth]{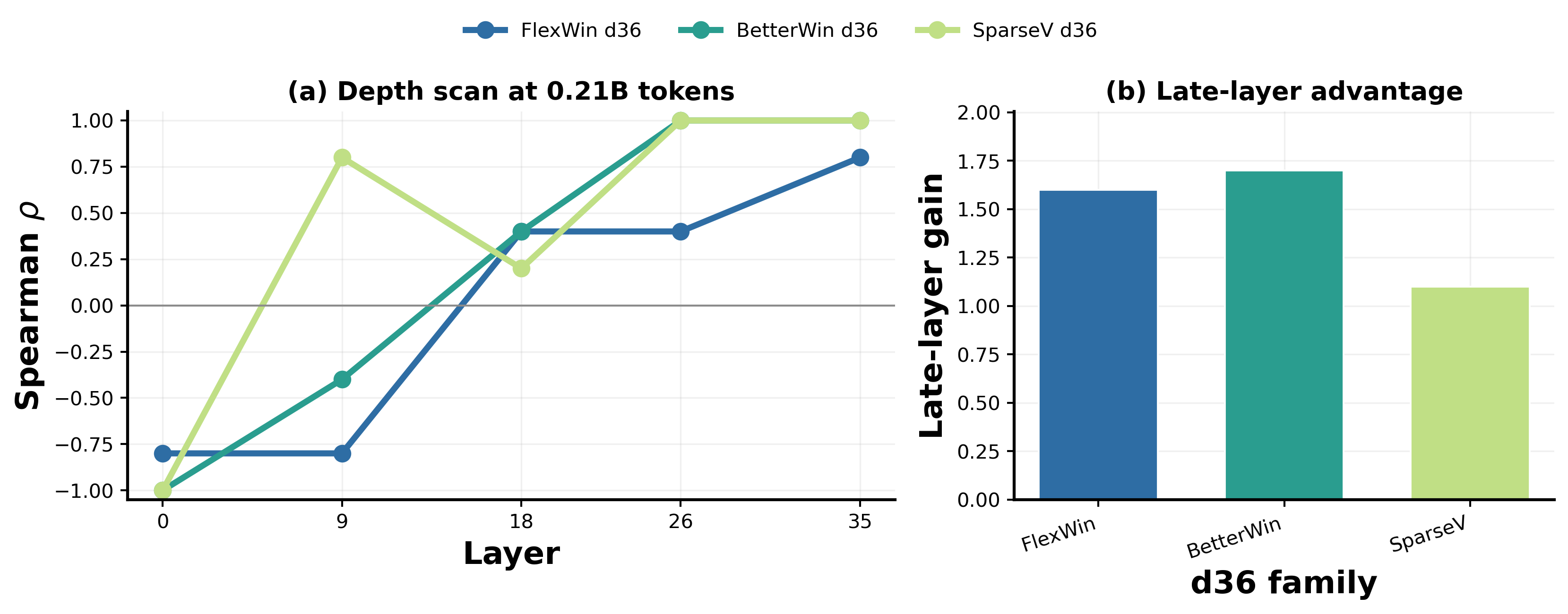}
  \caption{\textbf{Late-layer probes carry the clearest early-prediction signal in the d36 support runs.} (a) Spearman correlation between the early activation tail exponent and tokens-to-target proxy across batch tiers, computed at the saved checkpoint closest to $0.25$B training tokens. Each line is one d36 family and each x-position is a stored probe layer. The tail exponent is fit over ranks 200--400 of the activation covariance spectrum. (b) Difference between the final-layer correlation and the average correlation of the two earliest stored layers, $\rho(L35)-(\rho(L0)+\rho(L9))/2$. Positive values indicate that the late probe gives a stronger monotonic ordering of batch-tier token cost than early-layer probes.}
  \label{fig:appendix_layerwise_prediction}
\end{figure}

\section{Toy Model}
\label{sec:appendix_toy_setup}
The spectral measurements in the main paper are empirical: they track covariance eigenvalues and gradient singular values, but do not by themselves reveal why those quantities predict training efficiency or distinguish architectural interventions. The toy experiments address this gap through two controlled settings in which the same measurements can be traced back to explicit task structure. The analytic Fourier controls show how loss, feature spectra, and gradient spectra can decouple even when
the target basis is fully known. The modular-arithmetic Transformer then provides an empirical bridge: because the task is defined on a finite cyclic group, Fourier modes are explicit task-aligned coordinates, allowing feature learning to be measured directly rather than inferred from black-box covariance statistics.

\subsection{Linearized and diagonal Fourier models}
In the following we show that one-layer gradient-flow model shows that loss alone does not identify the hidden spectrum: different learning rates or noise levels can reach the same objective while retaining different activation covariance spectra, because Fourier modes approach their teacher coefficients at different rates. The two-layer diagonal Fourier model strengthens this by allowing each learned coefficient to factor through two trainable components, making mode growth nonlinear. Early spectral tails then become informative precisely because they reveal which unresolved modes still carry energy even after the scalar objective has moved substantially.

\subsection{Optimization and parameterization interventions}
The modular-arithmetic setting provides interpretable analogues of RoPE, Muon, and untied readouts. RoPE aligns the model with shift-equivariant Fourier coordinates, making the cyclic task structure natural. Muon changes the update geometry for matrix parameters, producing optimization efficiency gains without necessarily inducing an equally large change in feature concentration---a distinction visible in Fig.~\ref{fig:toy_appendix_matched_loss}(b). Untying the readout expands the accessible output subspace and removes constraints that force unrelated Fourier modes to share parameters. Together these distinctions mirror the main paper's taxonomy: some interventions primarily improve optimization efficiency, while others more directly reshape the learned representation.

\subsection{Two-layer Transformer and feature probes}
Our empirical Transformer toy uses a two-layer, four-head modular-arithmetic language model with context length $64$, vocabulary size $1{,}024$, dataset size $32{,}000$, and last-token pooling. We sweep $B \in \{32,64,128,256,512\}$ for the matched-loss spectra and use the cumulative chain Baseline$\rightarrow$RoPE$\rightarrow$Muon$\rightarrow$Untied
for the intervention analyses. Learning rates are selected by the same successive-halving logic used in the language-model experiments, and constant-loss comparisons use a validation-loss target of $2.0$.

Feature-learning probes are computed from saved checkpoints by Fourier-transforming hidden states along task-aligned bands. We report the task-band concentration statistic $H_S$, a peak statistic $H_{\mathrm{peak}}$, a smoother Gini-style whole-profile statistic, and PCA-based localization metrics. The main paper uses $H_S$ because it matches the theoretical quantity in Informal Result 2; the appendix uses $H_{\mathrm{peak}}$ and PC1 peak mass to expose where feature-learning signal enters the intervention chain.

Fig.~\ref{fig:toy_appendix_matched_loss} complements the main-text toy figure with three views. Panel~(a) confirms that even in this controlled modular task, matched validation loss does not imply matched spectra: smaller batches produce visibly steeper tails than larger ones. Panel~(b) decomposes the chain into consecutive $\Delta H_{\mathrm{peak}}$ gains. The Muon$\rightarrow$Untied step (purple) yields a jump roughly five times larger than either preceding transition, which remain small and
tightly clustered. This supports a clean distinction between gains from optimization geometry (RoPE, Muon) and gains from expanding the realizable output class (Untied).

Panel~(c) checks robustness across eight Fourier probe bands of varying width and offset. Baseline is uniformly low; Untied is consistently the strongest stage across every band; RoPE is broadly elevated with mild dips on the narrowest bands; and Muon peaks precisely on those narrow bands while remaining low on wider ones. We interpret
this complementarity as a spectral signature of the optimization-side versus
representation-side distinction: Muon's preconditioning concentrates on a narrow subset
of Fourier directions, while Untied's enlarged readout class produces a broadband
effect consistent with Theorem~15. The toy model thus supports a bounded but coherent
interpretation: the controlled setting links spectral diagnostics to task-aligned
feature recruitment, used qualitatively rather than as a claim about universal
transformer dynamics.

\begin{figure}[t]
  \centering
  \includegraphics[width=\textwidth]{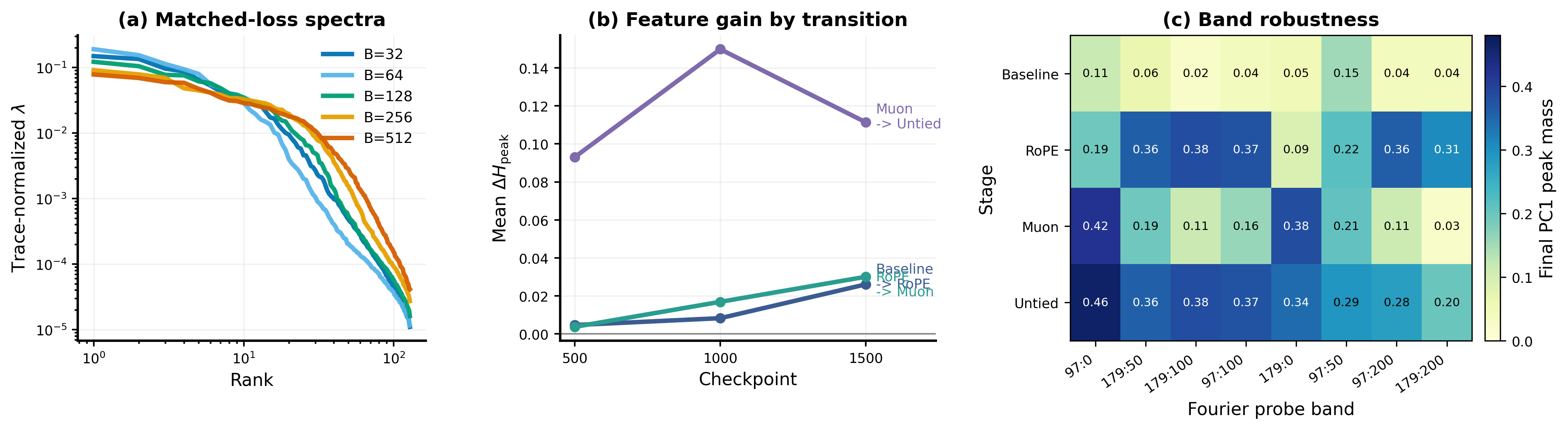}
  \caption{\textbf{The modular-arithmetic toy links matched-loss spectra to
  task-aligned feature learning.} (a) At matched validation loss, the Untied toy runs
  retain batch-dependent activation spectra across $B \in \{32,64,128,256,512\}$,
  paralleling the hidden-regime phenomenon in the language-model experiments.
  Smaller batches produce visibly steeper tails. (b) Consecutive intervention gains
  measured by mean $\Delta H_{\mathrm{peak}}$ at three checkpoints. The
  Muon$\rightarrow$Untied step produces a dramatic jump in feature concentration,
  while Baseline$\rightarrow$RoPE and RoPE$\rightarrow$Muon yield small, comparable
  gains. (c) A band sweep of final PC1 peak mass across eight Fourier probe bands.
  Untied is broadly strong across all bands; RoPE is broadly elevated but dips on the
  narrowest bands; Muon peaks precisely on those narrow bands while remaining low
  elsewhere; Baseline is uniformly weak. The pattern separates band-broad
  (representation-side) from band-narrow (optimization-side) signatures.}
  \label{fig:toy_appendix_matched_loss}
\end{figure}

\subsection{Formal toy-model theory and proofs}
\label{sec:appendix_toy_theory}
The main text states a small number of theorem-level claims in informal form.
This subsection gives full formal statements and complete proofs for the claims we use.
These are statements about the toy model and its measurements, not about full nonlinear transformer training.
We begin by formalizing the task itself: modular arithmetic reduces to tracking a phase on a finite cyclic group, and the Fourier characters on that group are the natural basis.
We then introduce a tractable shift-equivariant kernel-gradient-flow toy model and show that it decouples the learning problem into independent one-dimensional dynamics, one per Fourier mode.
Those dynamics can then be read back into the measurements used in the paper: activation covariance spectra weight learned mode energy, while gradient covariance spectra weight residual mode energy.
With that identification in place, the toy claims used in the main text---that loss-matched runs occupy different spectral states, that early spectra predict later efficiency, that a smooth power-law specialization yields a learned head plus unresolved tail, and that activation and gradient spectra are complementary---follow from the same modewise dynamics.

\subsubsection{One-layer Fourier model}
Fix an integer $c \ge 2$, a step size $d \in \{0,1,\dots,c-1\}$, an offset $o$, and a context length $L \ge 1$.
For each phase $a \in \Z_c$, define the clean single-component sequence
\begin{equation}
    x(a)=\bigl(o+(a+jd)\bmod c\bigr)_{j=0}^{L-1},
\end{equation}
with next-token target
\begin{equation}
    y(a)=o+(a+Ld)\bmod c.
\end{equation}
The essential latent state is therefore the phase $a$ on the finite cyclic group $\Z_c$.
Let
\begin{equation}
    \omega=e^{2\pi i/c},
    \qquad
    \chi_r(a)=\omega^{ra},
    \qquad r=0,1,\dots,c-1,
\end{equation}
be the Fourier characters on $\Z_c$.
They form an orthonormal basis of functions $f:\Z_c\to\C$ under the inner product
\begin{equation}
    \ip{f}{g} = \frac{1}{c}\sum_{a\in\Z_c} f(a)\overline{g(a)}.
\end{equation}

\begin{lemma}[The clean toy target is a cyclic shift]
\label{lem:toy_shift}
Let $S_q$ be the shift operator on functions $f:\Z_c\to\C$ defined by
\begin{equation}
    (S_q f)(a)=f(a+q \bmod c).
\end{equation}
Then the clean modular-arithmetic prediction map is the shift $S_{Ld}$, and each Fourier character is an eigenfunction:
\begin{equation}
    S_{Ld}\chi_r=\omega^{rLd}\chi_r.
\end{equation}
\end{lemma}
Lemma~\ref{lem:toy_shift} characterizes the ideal solution but says nothing about the \emph{path} to it.
To make the training dynamics tractable, we introduce the linearized kernel-gradient-flow model
\begin{equation}
    \partial_t f_t = K(f^\star-f_t),
    \label{eq:appendix_gf}
\end{equation}
where $f^\star$ is the target function and $K$ is a positive semidefinite linear operator on functions on $\Z_c$.
The key modeling assumption is that $K$ is shift-equivariant---it respects the cyclic symmetry of the task:
\begin{equation}
    K(a+s,a'+s)=K(a,a')
    \qquad \text{for all } a,a',s\in\Z_c.
\end{equation}
Equivalently, there exists a kernel $k:\Z_c\to\C$ such that
\begin{equation}
    (Kf)(a)=\sum_{a'\in\Z_c} k(a-a')f(a').
\end{equation}

This symmetry is what makes the analysis simple.
Because $K$ commutes with cyclic shifts, the Fourier characters $\chi_r$ are simultaneously eigenfunctions of $K$ and of the task target, which means the learning problem decomposes into $c$ entirely independent one-dimensional dynamics.
Parts (a)--(c) of the following theorem make this decoupling precise; part (d) then connects the resulting abstract Fourier dynamics to the activation and gradient second moments measured in the toy analysis, after which Proposition~\ref{prop:appendix_centered_cov} converts them to the centered covariances used in the experiments.

\begin{theorem}[Modewise Fourier dynamics for a one-layer linearized model]
\label{thm:appendix_one_layer}
Assume \eqref{eq:appendix_gf} with the shift-equivariant positive semidefinite operator $K$ defined above.
Expand the target and predictor in the Fourier basis:
\begin{equation}
    f^\star=\sum_{r=0}^{c-1}\beta_r\chi_r,
    \qquad
    f_t=\sum_{r=0}^{c-1} a_r(t)\chi_r.
\end{equation}
Then:
\begin{enumerate}[label=(\alph*),leftmargin=1.5em]
    \item each $\chi_r$ is an eigenfunction of $K$ with eigenvalue
    \begin{equation}
        K\chi_r=\kappa_r\chi_r,
        \qquad
        \kappa_r = \sum_{\delta\in\Z_c} k(\delta)\omega^{-r\delta};
    \end{equation}
    \item each Fourier coefficient evolves independently as
    \begin{equation}
        \dot a_r(t) = -\kappa_r(a_r(t)-\beta_r),
    \end{equation}
    hence
    \begin{equation}
        a_r(t)=\beta_r+\bigl(a_r(0)-\beta_r\bigr)e^{-\kappa_r t};
    \end{equation}
    \item under the $L^2(\Z_c)$ loss
    \begin{equation}
        \mathcal L(t)=\frac12\norm{f_t-f^\star}_{L^2(\Z_c)}^2,
    \end{equation}
    we have
    \begin{equation}
        \mathcal L(t)=\frac12\sum_{r=0}^{c-1}|a_r(t)-\beta_r|^2;
    \end{equation}
    \item if $\{u_r\}_{r=0}^{c-1}$ is any orthonormal family of hidden directions in $\C^m$ and
    \begin{equation}
        h_t(a)=\sum_{r=0}^{c-1} a_r(t)u_r\chi_r(a),
    \end{equation}
    while
    \begin{equation}
        g_t(a)=\sum_{r=0}^{c-1}\kappa_r\bigl(\beta_r-a_r(t)\bigr)u_r\chi_r(a),
    \end{equation}
    then the second moments over uniform phase $a\sim \mathrm{Unif}(\Z_c)$ are
    \begin{align}
        M_h(t)
        &= \E\bigl[h_t(a)h_t(a)^\ast\bigr]
        = \sum_{r=0}^{c-1} |a_r(t)|^2 u_r u_r^\ast,
        \\
        M_g(t)
        &= \E\bigl[g_t(a)g_t(a)^\ast\bigr]
        = \sum_{r=0}^{c-1}\kappa_r^2 |\beta_r-a_r(t)|^2 u_r u_r^\ast.
    \end{align}
    In particular, the activation spectrum weights learned energy, while the update-side spectrum weights residual energy.
\end{enumerate}
\end{theorem}

\begin{proof}
For part (a), compute directly:
\begin{align}
    (K\chi_r)(a)
    &= \sum_{a'\in\Z_c} k(a-a')\chi_r(a')
    = \sum_{\delta\in\Z_c} k(\delta)\chi_r(a-\delta)
    \\
    &= \chi_r(a)\sum_{\delta\in\Z_c} k(\delta)\omega^{-r\delta}
    = \kappa_r\chi_r(a).
\end{align}
So each character is an eigenfunction.

For part (b), substitute the Fourier expansions into \eqref{eq:appendix_gf}:
\begin{equation}
    \sum_r \dot a_r(t)\chi_r
    = K\Big(\sum_r \bigl(\beta_r-a_r(t)\bigr)\chi_r\Big)
    = \sum_r \kappa_r\bigl(\beta_r-a_r(t)\bigr)\chi_r.
\end{equation}
since $K$ is a linear operator. Orthogonality of the $\chi_r$ gives
\begin{equation}
    \dot a_r(t) = -\kappa_r(a_r(t)-\beta_r),
\end{equation}
whose solution is the stated exponential formula.

Part (c) is Parseval in the orthonormal Fourier basis:
\begin{equation}
    \norm{f_t-f^\star}_{L^2(\Z_c)}^2
    = \sum_r |a_r(t)-\beta_r|^2.
\end{equation}

For part (d), use orthogonality of the hidden directions and Fourier characters.
For $M_h$,
\begin{align}
    M_h(t)
    &= \E\left[\sum_{r,s} a_r(t)\overline{a_s(t)}\,u_r u_s^\ast\,\chi_r(a)\overline{\chi_s(a)}\right]
    \\
    &= \sum_{r,s} a_r(t)\overline{a_s(t)}\,u_r u_s^\ast\,\E\bigl[\chi_r(a)\overline{\chi_s(a)}\bigr]
    \\
    &= \sum_r |a_r(t)|^2 u_r u_r^\ast.
\end{align}
The formula for $M_g$ is identical after replacing $a_r(t)$ by $\kappa_r(\beta_r-a_r(t))$.
\end{proof}

\begin{proposition}[Second moments and centered covariances]
\label{prop:appendix_centered_cov}
In the setting of Theorem~\ref{thm:appendix_one_layer}(d), define
\begin{equation}
    M_h(t)=\E[h_t(a)h_t(a)^\ast],
    \qquad
    M_g(t)=\E[g_t(a)g_t(a)^\ast].
\end{equation}
Then the means are
\begin{equation}
    \mu_h(t)=\E[h_t(a)] = a_0(t)u_0,
    \qquad
    \mu_g(t)=\E[g_t(a)] = \kappa_0\bigl(\beta_0-a_0(t)\bigr)u_0,
\end{equation}
and the centered covariances are
\begin{align}
    \mathrm{Cov}(h_t(a))
    &= M_h(t)-\mu_h(t)\mu_h(t)^\ast
    = \sum_{r=1}^{c-1} |a_r(t)|^2 u_r u_r^\ast,
    \\
    \mathrm{Cov}(g_t(a))
    &= M_g(t)-\mu_g(t)\mu_g(t)^\ast
    = \sum_{r=1}^{c-1}\kappa_r^2 |\beta_r-a_r(t)|^2 u_r u_r^\ast.
\end{align}
Thus the DC mode contributes to the mean and uncentered second moment, but it drops out of the centered covariance used by the experiments.
\end{proposition}

\begin{proof}
Since $\E_a[\chi_r(a)]=\frac{1}{c}\sum_{a\in\Z_c}\omega^{ra}=\delta_{r,0}$, only the DC mode has nonzero mean:
\begin{align}
    \mu_h(t)
    &=\sum_{r=0}^{c-1} a_r(t)u_r \E[\chi_r(a)]
      = a_0(t)u_0,\\
    \mu_g(t)
    &=\sum_{r=0}^{c-1} \kappa_r(\beta_r-a_r(t))u_r \E[\chi_r(a)]
      = \kappa_0(\beta_0-a_0(t))u_0.
\end{align}
Subtracting $\mu_h(t)\mu_h(t)^\ast$ and $\mu_g(t)\mu_g(t)^\ast$ from the second moments in Theorem~\ref{thm:appendix_one_layer}(d) removes exactly the $r=0$ term, which yields the stated covariance formulas.
\end{proof}

Theorem~\ref{thm:appendix_one_layer} together with Proposition~\ref{prop:appendix_centered_cov} is the analytical foundation for all three empirical claims about the toy model.
Part (d) gives the second-moment dictionary, and Proposition~\ref{prop:appendix_centered_cov} converts it into the centered covariance form used in the experiments: activation spectra weight the energies $|a_r(t)|^2$ of modes that have already been learned, while gradient spectra weight the residual energies $\kappa_r^2|\beta_r - a_r(t)|^2$ of modes still being learned.
Since different training regimes induce different rates $\{\kappa_r\}$, the mode energies diverge across regimes even at equal total loss---and those divergences are directly visible in the observable covariances.
The three results below make this precise.

\begin{proposition}[Matched loss does not identify spectral state]
\label{prop:appendix_matched_loss}
Consider $m\ge 2$ modes with zero initialization and equal target coefficients $\beta_r=1$ for $r=1,\dots,m$.
Suppose regime $\mathrm{A}$ has isotropic rates $\kappa_r^{\mathrm{A}}=\kappa>0$ for all $r$, and regime $\mathrm{B}$ has strictly ordered anisotropic rates
\begin{equation}
    \kappa_1^{\mathrm{B}}>\kappa_2^{\mathrm{B}}>\cdots>\kappa_m^{\mathrm{B}}>0.
\end{equation}
Define the normalized activation mass on mode $1$ by
\begin{equation}
    P_j(t)=\frac{|a_1^{(j)}(t)|^2}{\sum_{r=1}^m |a_r^{(j)}(t)|^2},
    \qquad
    j\in\{\mathrm{A},\mathrm{B}\}.
\end{equation}
Then for every loss level $\ell\in(0,m/2)$ there exist unique times $t_{\mathrm{A}},t_{\mathrm{B}}>0$ such that
$\mathcal{L}_{\mathrm{A}}(t_{\mathrm{A}})=\mathcal{L}_{\mathrm{B}}(t_{\mathrm{B}})=\ell$, but
\begin{equation}
    P_{\mathrm{A}}(t_{\mathrm{A}})=\frac{1}{m},
    \qquad
    P_{\mathrm{B}}(t_{\mathrm{B}})>\frac{1}{m}.
\end{equation}
So equal scalar loss does not determine equal internal spectral state.
\end{proposition}

\begin{proof}
Under zero initialization, Theorem~\ref{thm:appendix_one_layer} gives $a_r^{(j)}(t)=1-e^{-\kappa_r^{(j)}t}$.
In regime $\mathrm{A}$ all modes evolve identically, so $P_{\mathrm{A}}(t)=1/m$ for every $t>0$ and
$\mathcal{L}_{\mathrm{A}}(t)=\frac{m}{2}e^{-2\kappa t}$, which decreases strictly and continuously from $m/2$ to $0$.
In regime $\mathrm{B}$,
\begin{equation}
    \mathcal{L}_{\mathrm{B}}(t)
    = \frac{1}{2}\sum_{r=1}^{m} e^{-2\kappa_r^{(\mathrm{B})}t},
\end{equation}
which also decreases strictly and continuously from $m/2$ to $0$, so for each $\ell\in(0,m/2)$ there are unique matched-loss times.
Since $\kappa_1^{(\mathrm{B})}>\kappa_r^{(\mathrm{B})}$ for all $r>1$, we have
$1-e^{-\kappa_1^{(\mathrm{B})}t}>1-e^{-\kappa_r^{(\mathrm{B})}t}$ for every $t>0$,
hence $P_{\mathrm{B}}(t)>1/m$ for every $t>0$.
Therefore the matched-loss spectral states differ.
\end{proof}

Proposition~\ref{prop:appendix_matched_loss} shows that different regimes leave different spectral fingerprints even at equal loss.
The next result shows that these fingerprints are visible \emph{early}: a run's spectral concentration at any fixed early time $t_0$ already determines its time-to-target, within a family of problems that share the same fast-mode rate but vary in how quickly the slow mode is learned.

\begin{proposition}[Early band concentration predicts time-to-target]
\label{prop:appendix_early_predict}
Consider $m$ modes partitioned into a fast band $F=\{1,\dots,k\}$ with shared rate $\bar\kappa>0$ and a slow band $S=\{k+1,\dots,m\}$ with shared rate $\kappa_s\in(0,\bar\kappa)$, with zero initialization and $\beta_r=1$ for all $r$.
At an early time $t_0>0$, define the band concentration on the fast modes by
\begin{equation}
    C^{(k)}(t_0;\kappa_s)
    = \frac{k\,(1-e^{-\bar\kappa t_0})^2}{k\,(1-e^{-\bar\kappa t_0})^2+(m-k)(1-e^{-\kappa_s t_0})^2}.
\end{equation}
For a target loss $\varepsilon\in(0,m/2)$, let $T_\varepsilon(\kappa_s)$ be the unique time satisfying
\begin{equation}
    \frac{1}{2}\Bigl(k\,e^{-2\bar\kappa T_\varepsilon}+(m-k)e^{-2\kappa_s T_\varepsilon}\Bigr)=\varepsilon.
\end{equation}
Then
\begin{equation}
    \frac{\partial C^{(k)}}{\partial\kappa_s}(t_0;\kappa_s)<0,
    \qquad
    \frac{dT_\varepsilon}{d\kappa_s}(\kappa_s)<0.
\end{equation}
Hence larger early band concentration $C^{(k)}$ implies larger time-to-target $T_\varepsilon$.
\end{proposition}

\begin{proof}
Set $A=k(1-e^{-\bar\kappa t_0})^2$ and $B(\kappa_s)=(m-k)(1-e^{-\kappa_s t_0})^2$.
Since
\begin{equation}
    \frac{dB}{d\kappa_s}=2(m-k)\,t_0\,e^{-\kappa_s t_0}(1-e^{-\kappa_s t_0})>0,
\end{equation}
we get $\partial C^{(k)}/\partial\kappa_s = -AB'(\kappa_s)/(A+B(\kappa_s))^2<0$.

For $T_\varepsilon$, define $F(T,\kappa_s)=\frac{1}{2}(k\,e^{-2\bar\kappa T}+(m-k)e^{-2\kappa_s T})-\varepsilon$.
By construction $F(T_\varepsilon(\kappa_s),\kappa_s)=0$.
Implicit differentiation gives
\begin{equation}
    \frac{dT_\varepsilon}{d\kappa_s}
    = -\frac{\partial_{\kappa_s}F}{\partial_T F}
    = -\frac{-(m-k)\,T_\varepsilon\,e^{-2\kappa_s T_\varepsilon}}{-k\bar\kappa\,e^{-2\bar\kappa T_\varepsilon}-(m-k)\kappa_s\,e^{-2\kappa_s T_\varepsilon}}<0.
\end{equation}
Since both $C^{(k)}$ and $T_\varepsilon$ decrease in $\kappa_s$, higher early band concentration implies longer time-to-target.
\end{proof}

Proposition~\ref{prop:appendix_early_predict} is the minimal two-band version of the batch-selection story.
It is intentionally family-local: runs are compared only after fixing the progress of the fast band and varying the recruitment speed of the slower band.
The next stylized specialization sharpens that same idea into the line-shape language used in the main text: learned head, crossover rank, unresolved tail, exact band-recruitment time, and the outward shift of the useful tail-fit window.

\subsubsection{Smooth-spectrum extension for family-local batch selection}
Assume the teacher energy follows a power law
\begin{equation}
    \beta_r^2 = C r^{-p},
    \qquad
    C>0,
    \quad
    p>0,
    \label{eq:appendix_teacher_power}
\end{equation}
and the batch-dependent learning rates take the form
\begin{equation}
    \kappa_r(B)=\eta_B r^{-q_B},
    \qquad
    \eta_B>0,
    \quad
    q_B\ge 0.
    \label{eq:appendix_rate_power}
\end{equation}
Under zero initialization, Theorem~\ref{thm:appendix_one_layer} gives
\begin{equation}
    a_r(B,t)=\beta_r\bigl(1-e^{-\eta_B t r^{-q_B}}\bigr),
\end{equation}
so the activation-side covariance eigenvalues are
\begin{equation}
    \lambda_r^{\mathrm{act}}(B,t)
    =
    |a_r(B,t)|^2
    =
    C r^{-p}\Bigl(1-e^{-\eta_B t r^{-q_B}}\Bigr)^2.
    \label{eq:appendix_smooth_act}
\end{equation}

\begin{theorem}[Three-zone activation spectrum]
\label{thm:appendix_threezone}
Assume \eqref{eq:appendix_teacher_power}--\eqref{eq:appendix_smooth_act} with $q_B>0$, and define the crossover rank
\begin{equation}
    r_*(B,t):=(\eta_B t)^{1/q_B}.
    \label{eq:appendix_crossrank}
\end{equation}
Then the activation spectrum has three asymptotic zones:
\begin{enumerate}[label=(\roman*),leftmargin=1.5em]
    \item if $r \ll r_*(B,t)$, then
    \begin{equation}
        \lambda_r^{\mathrm{act}}(B,t)=C r^{-p}\bigl(1+o(1)\bigr);
    \end{equation}
    this is the learned head;
    \item if $r \gg r_*(B,t)$, then
    \begin{equation}
        \lambda_r^{\mathrm{act}}(B,t)=C\eta_B^2 t^2 r^{-(p+2q_B)}\bigl(1+o(1)\bigr);
    \end{equation}
    this is the unresolved tail, whose exponent is
    \begin{equation}
        \alpha_{\mathrm{tail}}(B)=p+2q_B;
        \label{eq:appendix_tail_alpha}
    \end{equation}
    \item the bend between the two slopes is centered on $r\asymp r_*(B,t)$.
\end{enumerate}
\end{theorem}

\begin{proof}
Set
\begin{equation}
    x_{r,B}(t)=\eta_B t r^{-q_B}.
\end{equation}
Then \eqref{eq:appendix_smooth_act} becomes
\begin{equation}
    \lambda_r^{\mathrm{act}}(B,t)=C r^{-p}(1-e^{-x_{r,B}(t)})^2.
\end{equation}
If $r \ll r_*(B,t)$, then $x_{r,B}(t)\to\infty$, so $1-e^{-x_{r,B}(t)}\to 1$, which gives the head asymptotic.
If $r \gg r_*(B,t)$, then $x_{r,B}(t)\to 0$, and the Taylor expansion $1-e^{-x}=x+O(x^2)$ yields
\begin{equation}
    (1-e^{-x_{r,B}(t)})^2
    =
    \eta_B^2 t^2 r^{-2q_B}\bigl(1+o(1)\bigr),
\end{equation}
which proves the tail formula and \eqref{eq:appendix_tail_alpha}.
The crossover statement is exactly the regime $x_{r,B}(t)\asymp 1$.
\end{proof}

\begin{theorem}[Band-recruitment law]
\label{thm:appendix_band_recruit}
Assume \eqref{eq:appendix_teacher_power}--\eqref{eq:appendix_smooth_act}, zero initialization, and monotone rates in rank.
Fix a cutoff $R$ and a tolerance $\delta\in(0,1)$, and define
\begin{equation}
    T_{R,\delta}(B)
    :=
    \inf\Bigl\{
        t\ge 0:
        a_r(B,t)\ge (1-\delta)\beta_r
        \text{ for all } 1\le r\le R
    \Bigr\}.
\end{equation}
Then
\begin{equation}
    T_{R,\delta}(B)
    =
    \frac{\log(1/\delta)}{\eta_B}R^{q_B}
    =
    \frac{\log(1/\delta)}{\eta_B}R^{(\alpha_{\mathrm{tail}}(B)-p)/2}.
    \label{eq:appendix_bandlaw}
\end{equation}
Consequently, the exact family-local objective is to maximize the tail-band rate $\eta_B R^{-q_B}$.
\end{theorem}

\begin{proof}
Because the rates are monotone decreasing in $r$, the slowest mode among $\{1,\dots,R\}$ is mode $R$.
Under zero initialization,
\begin{equation}
    a_R(B,t)=\beta_R\bigl(1-e^{-\eta_B t R^{-q_B}}\bigr).
\end{equation}
The condition $a_R(B,t)\ge (1-\delta)\beta_R$ is equivalent to
\begin{equation}
    e^{-\eta_B t R^{-q_B}}\le \delta,
\end{equation}
or
\begin{equation}
    t\ge \frac{\log(1/\delta)}{\eta_B}R^{q_B}.
\end{equation}
This is necessary and sufficient for all $1\le r\le R$ because every earlier mode has rate at least as large as mode $R$.
The second expression in \eqref{eq:appendix_bandlaw} follows from \eqref{eq:appendix_tail_alpha}.
\end{proof}

\begin{corollary}[Head-matched early measurement gives an exact tail criterion]
\label{cor:appendix_headmatched}
Assume the setting of Theorem~\ref{thm:appendix_band_recruit}.
Suppose batches are compared at an early time $t_0$ after matching the progress of a head-anchor mode $r_h$, in the sense that
\begin{equation}
    a_{r_h}(B,t_0)=\xi\beta_{r_h}
    \qquad
    \text{for the same }
    \xi\in(0,1)
    \text{ and every } B.
\end{equation}
Assume also that $R>r_h$, so the target band extends beyond the matched head.
Then
\begin{equation}
    T_{R,\delta}(B)
    =
    \frac{\log(1/\delta)}{-\log(1-\xi)}t_0\Bigl(\frac{R}{r_h}\Bigr)^{q_B}
    =
    C_{\delta,\xi,t_0}\Bigl(\frac{R}{r_h}\Bigr)^{(\alpha_{\mathrm{tail}}(B)-p)/2},
    \label{eq:appendix_headmatched_law}
\end{equation}
for a constant $C_{\delta,\xi,t_0}$ independent of $B$.
Therefore the family-local token-optimal batch is exactly the batch with the smallest informative tail exponent.
\end{corollary}

\begin{proof}
The head-anchor condition gives
\begin{equation}
    1-e^{-\eta_B t_0 r_h^{-q_B}}=\xi,
\end{equation}
so
\begin{equation}
    \eta_B=-t_0^{-1}r_h^{q_B}\log(1-\xi).
\end{equation}
Substituting this identity into \eqref{eq:appendix_bandlaw} yields \eqref{eq:appendix_headmatched_law}.
\end{proof}

\begin{proposition}[Crossover-rank explanation of the tail-window shift]
\label{prop:appendix_window_shift}
In the setting of Theorem~\ref{thm:appendix_threezone}, a fixed fit window $[R_1,R_2]$ recovers the true tail exponent $p+2q_B$ only when it lies in the unresolved-tail regime, that is, when $R_1 \gg r_*(B,t)$.
If scale, probe depth, or early-time dynamics change in such a way that $r_*(B,t)$ increases, then any previously informative fixed window eventually enters the bend region, and the useful tail-fit window must move outward.
\end{proposition}

\begin{proof}
By Theorem~\ref{thm:appendix_threezone}, the unresolved-tail asymptotic with slope $p+2q_B$ is valid only for ranks satisfying $r \gg r_*(B,t)$.
If $R_1/r_*(B,t)$ ceases to be large, then the lower end of the fit window enters the crossover region, where the local slope is shallower than the true tail slope.
So the same fixed window no longer estimates $\alpha_{\mathrm{tail}}(B)$ reliably, and moving the fit window outward restores the unresolved-tail condition.
\end{proof}

Theorem~\ref{thm:appendix_threezone}, Theorem~\ref{thm:appendix_band_recruit}, and Corollary~\ref{cor:appendix_headmatched} are the formal version of the family-local line-shape claim used in the main text.
They show that the useful early spectrum is not summarized by one universal scalar.
Within a matched family, the informative signal is the combination of a sufficiently progressed head, a crossover rank $r_*(B,t)$ pushed outward, and an unresolved tail that remains comparatively flat.

The smooth-spectrum extension explains \emph{when} a run will reach its target and which early tail shape is informative after normalizing for head progress; the final result explains \emph{why} the activation and gradient views of that trajectory are complementary rather than redundant.
Intuitively, the activation spectrum captures what has already been learned (high-$\kappa_r$ modes consolidate first), while the gradient spectrum captures what is still being actively updated (slow modes dominate residual energy at late times).
The following proposition formalizes this crossover.

\begin{proposition}[Activation--gradient complementarity via mode crossovers]
\label{prop:appendix_crossover}
Consider $m$ modes with zero initialization, equal target coefficients $\beta_r=1$, and strictly ordered rates $\kappa_1>\kappa_2>\cdots>\kappa_m>0$.
Define activation energies $A_r(t)=(1-e^{-\kappa_r t})^2$ and update-side energies $G_r(t)=\kappa_r^2 e^{-2\kappa_r t}$.
Then:
\begin{enumerate}[label=(\roman*),leftmargin=1.5em]
    \item $A_i(t)>A_j(t)$ for every pair $i<j$ and every $t>0$;
    \item for each pair $i<j$ there is a unique crossover time
    \begin{equation}
        t_{ij} = \frac{\log(\kappa_i/\kappa_j)}{\kappa_i-\kappa_j}>0
    \end{equation}
    such that $G_i(t)>G_j(t)$ for $0<t<t_{ij}$, $G_i(t)=G_j(t)$ at $t=t_{ij}$, and $G_i(t)<G_j(t)$ for $t>t_{ij}$.
\end{enumerate}
Setting $T^\ast=\max_{i<j}t_{ij}$, for all $t>T^\ast$ the activation ranking satisfies $A_1(t)>\cdots>A_m(t)$ while the update-side ranking is fully reversed: $G_m(t)>\cdots>G_1(t)$.
\end{proposition}

\begin{proof}
Part (i): $\kappa_i>\kappa_j$ implies $1-e^{-\kappa_i t}>1-e^{-\kappa_j t}$ for all $t>0$; squaring gives $A_i(t)>A_j(t)$.

For part (ii), compare the update-side energies of modes $i$ and $j$:
\begin{equation}
    G_i(t)<G_j(t)
    \iff
    \kappa_i^2 e^{-2\kappa_i t}<\kappa_j^2 e^{-2\kappa_j t}
    \iff
    e^{(\kappa_i-\kappa_j)t}>\frac{\kappa_i}{\kappa_j}.
\end{equation}
Since $\kappa_i>\kappa_j$, this holds iff $t>t_{ij}$, and uniqueness follows from strict monotonicity of the exponential.
After $T^\ast$, every pair $(i,j)$ with $i<j$ satisfies $G_i(t)<G_j(t)$, giving the stated full reversal of the gradient ranking.
\end{proof}

In the smooth-spectrum specialization above, this same crossover logic is what separates the learned head from the unresolved tail and makes the tail-window proposition natural.
The point of Proposition~\ref{prop:appendix_crossover} is the complementary one: even when activation and gradient are generated by the same modewise dynamics, they need not rank modes the same way at the same time.

\subsubsection{Two-layer Fourier factor model}
Let $\{\phi_r\}_{r=1}^m$ be any orthonormal real Fourier-derived basis.
Consider the two-layer factor model
\begin{equation}
    f_t(a)=\sum_{r=1}^m u_r(t)v_r(t)\phi_r(a),
\end{equation}
with target
\begin{equation}
    f^\star(a)=\sum_{r=1}^m \beta_r \phi_r(a),
    \qquad
    \beta_r\ge 0,
\end{equation}
trained by gradient flow on the squared loss
\begin{equation}
    \mathcal J(u,v)=\frac12\sum_{r=1}^m (u_r v_r-\beta_r)^2.
\end{equation}

\begin{theorem}[Balanced two-layer Fourier factor dynamics]
\label{thm:appendix_two_layer}
Assume gradient flow on $\mathcal J$ with balanced positive initialization
\begin{equation}
    u_r(0)=v_r(0)>0
    \qquad \text{for all } r.
\end{equation}
Let
\begin{equation}
    m_r(t)=u_r(t)v_r(t).
\end{equation}
Then:
\begin{enumerate}[label=(\alph*),leftmargin=1.5em]
    \item for every $r$ and every $t\ge 0$,
    \begin{equation}
        u_r(t)=v_r(t),
        \qquad
        m_r(t)=u_r(t)^2=v_r(t)^2;
    \end{equation}
    \item each mode follows the autonomous logistic-type ODE
    \begin{equation}
        \dot m_r(t)=2m_r(t)(\beta_r-m_r(t));
    \end{equation}
    \item if $\beta_r>0$, then
    \begin{equation}
        m_r(t)=\frac{\beta_r}{1+\left(\frac{\beta_r}{m_r(0)}-1\right)e^{-2\beta_r t}},
    \end{equation}
    while if $\beta_r=0$, then
    \begin{equation}
        m_r(t)=\frac{m_r(0)}{1+2m_r(0)t}.
    \end{equation}
\end{enumerate}
\end{theorem}

\begin{proof}
Gradient flow on $\mathcal J$ gives, for each $r$,
\begin{equation}
    \dot u_r=-(u_rv_r-\beta_r)v_r,
    \qquad
    \dot v_r=-(u_rv_r-\beta_r)u_r.
\end{equation}
Then
\begin{equation}
    \frac{d}{dt}(u_r^2-v_r^2)=2u_r\dot u_r-2v_r\dot v_r=0.
\end{equation}
Since $u_r(0)=v_r(0)$, we have $u_r(t)^2=v_r(t)^2$ for all $t$.
It is cleaner to track the difference directly:
\begin{equation}
    \frac{d}{dt}(u_r-v_r)=\dot u_r-\dot v_r=-(u_rv_r-\beta_r)(v_r-u_r)=(u_rv_r-\beta_r)(u_r-v_r).
\end{equation}
Because $u_r(0)-v_r(0)=0$, uniqueness of solutions implies $u_r(t)-v_r(t)\equiv 0$, so $u_r(t)=v_r(t)$ for all $t$.
Write this common value as $s_r(t)$.
Then
\begin{equation}
    \dot s_r=s_r(\beta_r-s_r^2).
\end{equation}
Since the vector field is locally Lipschitz and $s_r=0$ is an equilibrium, a solution starting from $s_r(0)>0$ cannot cross zero without violating uniqueness.
Hence $s_r(t)>0$ for all $t$, so the balanced positive branch is preserved.
This proves part (a).

For part (b), differentiate $m_r=u_rv_r$:
\begin{align}
    \dot m_r
    &= \dot u_r v_r + u_r \dot v_r
    \\
    &= -(u_rv_r-\beta_r)(v_r^2+u_r^2)
    \\
    &= -2m_r(m_r-\beta_r)
    = 2m_r(\beta_r-m_r).
\end{align}

For part (c), if $\beta_r>0$, separate variables:
\begin{equation}
    \frac{dm_r}{m_r(\beta_r-m_r)} = 2dt.
\end{equation}
Using the partial-fraction identity
\begin{equation}
    \frac{1}{m_r(\beta_r-m_r)}
    =
    \frac{1}{\beta_r}\left(\frac{1}{m_r}+\frac{1}{\beta_r-m_r}\right),
\end{equation}
we obtain
\begin{equation}
    \frac{1}{\beta_r}\log\!\left|\frac{m_r(t)}{\beta_r-m_r(t)}\right|=2t+C_r.
\end{equation}
Evaluating at $t=0$ gives
\begin{equation}
    C_r=\frac{1}{\beta_r}\log\!\left|\frac{m_r(0)}{\beta_r-m_r(0)}\right|,
\end{equation}
Exponentiating absorbs the sign into the integration constant, so the final closed-form solution is unchanged:
\begin{equation}
    \frac{m_r(t)}{\beta_r-m_r(t)}
    =
    \frac{m_r(0)}{\beta_r-m_r(0)}e^{2\beta_r t}.
\end{equation}
Solving for $m_r(t)$ yields
\begin{equation}
    m_r(t)=\frac{\beta_r}{1+\left(\frac{\beta_r}{m_r(0)}-1\right)e^{-2\beta_r t}}.
\end{equation}
If $\beta_r=0$, the equation becomes $\dot m_r=-2m_r^2$, whose solution is
\begin{equation}
    m_r(t)=\frac{m_r(0)}{1+2m_r(0)t}.
\end{equation}
\end{proof}

\begin{corollary}[Monotone task-band concentration]
\label{cor:appendix_hpeak}
Let $S\subseteq\{1,\dots,m\}$ be a task-relevant Fourier band.
Assume
\begin{equation}
    \beta_r>0 \text{ for } r\in S,
    \qquad
    \beta_r=0 \text{ for } r\notin S,
\end{equation}
and $0<m_r(0)<\beta_r$ for every $r\in S$.
Define the band-concentration statistic
\begin{equation}
    H_S(t)=\frac{\sum_{r\in S} m_r(t)}{\sum_{r=1}^m m_r(t)}.
\end{equation}
Then $H_S(t)$ is strictly increasing for every $t$ as long as both on-band and off-band mass are nonzero.
\end{corollary}

\begin{proof}
Set
\begin{equation}
    A(t)=\sum_{r\in S} m_r(t),
    \qquad
    B(t)=\sum_{r\notin S} m_r(t),
    \qquad
    H_S(t)=\frac{A(t)}{A(t)+B(t)}.
\end{equation}
By Theorem~\ref{thm:appendix_two_layer}, for $r\in S$ we have
\begin{equation}
    \dot m_r = 2m_r(\beta_r-m_r)>0,
\end{equation}
because $0<m_r(t)<\beta_r$ is preserved by the logistic flow.
Thus $A'(t)>0$ whenever $A(t)>0$.

For $r\notin S$, $\beta_r=0$, so Theorem~\ref{thm:appendix_two_layer} gives
\begin{equation}
    \dot m_r = -2m_r^2<0
\end{equation}
whenever $m_r(t)>0$.
Thus $B'(t)<0$ whenever $B(t)>0$.

Differentiate $H_S$:
\begin{equation}
    H_S'(t)
    = \frac{A'(t)(A(t)+B(t)) - A(t)(A'(t)+B'(t))}{(A(t)+B(t))^2}
    = \frac{A'(t)B(t)-A(t)B'(t)}{(A(t)+B(t))^2}.
\end{equation}
If $A(t),B(t)>0$, then $A'(t)>0$ and $-B'(t)>0$, so both terms in the numerator are positive.
Hence $H_S'(t)>0$.
\end{proof}

\subsubsection{Variant-aligned mechanism theorems}
\paragraph{RoPE as symmetry restoration.}
Let $x=(x_t)_{t\in\Z}$ be a bi-infinite sequence with $x_t\in\R^m$, and for a shift $\tau\in\Z$ define
\begin{equation}
    (S_\tau x)_t = x_{t-\tau}.
\end{equation}
Let $W_q,W_k\in\R^{d\times m}$ be fixed matrices and define
\begin{equation}
    q_t(x)=W_qx_t,
    \qquad
    k_t(x)=W_kx_t.
\end{equation}
Assume a family of orthogonal matrices $(R_t)_{t\in\Z}\subset O(d)$ satisfying
\begin{equation}
    R_{s+t}=R_sR_t,
    \qquad
    R_0=I.
\end{equation}
Define the RoPE attention score
\begin{equation}
    a_{ij}(x)=\langle R_i q_i(x), R_j k_j(x)\rangle.
\end{equation}

\begin{theorem}[RoPE score equivariance under sequence shifts]
\label{thm:appendix_rope_equiv}
For every sequence $x$, all indices $i,j\in\Z$, and every shift $\tau\in\Z$,
\begin{equation}
    a_{i+\tau,\,j+\tau}(S_\tau x)=a_{ij}(x).
\end{equation}
Equivalently, the positional part of the score depends on sequence position only through the relative offset $j-i$.
\end{theorem}

\begin{proof}
By definition of the shifted sequence,
\begin{equation}
    q_{i+\tau}(S_\tau x)=W_q(S_\tau x)_{i+\tau}=W_qx_i=q_i(x),
\end{equation}
and similarly $k_{j+\tau}(S_\tau x)=k_j(x)$.
Therefore
\begin{align}
    a_{i+\tau,\,j+\tau}(S_\tau x)
    &= \left\langle R_{i+\tau}q_i(x), R_{j+\tau}k_j(x)\right\rangle
    \\
    &= q_i(x)^\top R_{i+\tau}^\top R_{j+\tau}k_j(x).
\end{align}
Because the $R_t$ are orthogonal and form a representation of $\Z$,
\begin{equation}
    R_{i+\tau}^\top R_{j+\tau}
    = R_{-(i+\tau)}R_{j+\tau}
    = R_{j-i}.
\end{equation}
Hence
\begin{equation}
    a_{i+\tau,\,j+\tau}(S_\tau x)
    = q_i(x)^\top R_{j-i}k_j(x)
    = \langle R_i q_i(x), R_j k_j(x)\rangle
    = a_{ij}(x).
\end{equation}
\end{proof}

\begin{proposition}[Absolute positional tables generically break shift equivariance]
\label{prop:appendix_absolute_breaks}
Let $(p_t)_{t\in\Z}\subset\R^m$ be a positional table and define
\begin{equation}
    \tilde q_t(x)=W_q(x_t+p_t),
    \qquad
    \tilde k_t(x)=W_k(x_t+p_t),
\end{equation}
with score
\begin{equation}
    b_{ij}(x)=\langle \tilde q_i(x), \tilde k_j(x)\rangle.
\end{equation}
Assume $W_q,W_k\in\R^{d\times m}$ have full row rank $d$.
Assume that for every sequence $x$ and every $i,j,\tau\in\Z$,
\begin{equation}
    b_{i+\tau,\,j+\tau}(S_\tau x)=b_{ij}(x).
\end{equation}
Then
\begin{equation}
    W_q p_t \equiv \text{constant in } t,
    \qquad
    W_k p_t \equiv \text{constant in } t.
\end{equation}
In particular, any nontrivial absolute positional signal that survives the $W_q$ or $W_k$ projection breaks translation equivariance.
\end{proposition}

\begin{proof}
Fix $i,j,\tau$.
Using $(S_\tau x)_{i+\tau}=x_i$ and $(S_\tau x)_{j+\tau}=x_j$, the assumed equivariance becomes
\begin{equation}
    \big\langle W_q(x_i+p_{i+\tau}), W_k(x_j+p_{j+\tau})\big\rangle
    =
    \big\langle W_q(x_i+p_i), W_k(x_j+p_j)\big\rangle
\end{equation}
for all $x_i,x_j\in\R^m$.
Expanding both sides and subtracting gives
\begin{align}
    0
    &= \langle W_qx_i, W_k(p_{j+\tau}-p_j)\rangle
    + \langle W_q(p_{i+\tau}-p_i), W_kx_j\rangle
    \\
    &\qquad
    + \langle W_qp_{i+\tau},W_kp_{j+\tau}\rangle
    - \langle W_qp_i,W_kp_j\rangle
\end{align}
for all $x_i,x_j$.
Now set $x_j=0$ and vary $x_i$ arbitrarily.
The first term must vanish for all $x_i$, so
\begin{equation}
    \langle W_qx_i,\,W_k(p_{j+\tau}-p_j)\rangle = 0
    \qquad
    \text{for all } x_i\in\R^m.
\end{equation}
Because $W_q$ has full row rank, its row space is all of $\R^d$, so the vectors $W_qx_i$ span $\R^d$.
Hence the only vector orthogonal to every $W_qx_i$ is the zero vector, and therefore
\begin{equation}
    W_k(p_{j+\tau}-p_j)=0.
\end{equation}
Because $j,\tau$ were arbitrary, $W_kp_t$ is independent of $t$.
By symmetry, setting $x_i=0$ and varying $x_j$ yields
\begin{equation}
    \langle W_q(p_{i+\tau}-p_i),\,W_kx_j\rangle = 0
    \qquad
    \text{for all } x_j\in\R^m.
\end{equation}
Because $W_k$ has full row rank, the vectors $W_kx_j$ span $\R^d$, so
\begin{equation}
    W_q(p_{i+\tau}-p_i)=0
\end{equation}
for all $i,\tau$, so $W_qp_t$ is also independent of $t$.
\end{proof}

\paragraph{Untied readout as reduced factorization constraint.}
Let $E\in\R^{d\times V}$ be the token embedding matrix with full row rank $d<V$.
Consider the two classes of effective token-to-logit maps
\begin{equation}
    \mathcal F_{\mathrm{tied}}=\{E^\top A E : A\in\R^{d\times d}\},
    \qquad
    \mathcal F_{\mathrm{untied}}=\{W^\top A E : A\in\R^{d\times d},\,W\in\R^{d\times V}\}.
\end{equation}

\begin{theorem}[Strict expressivity inclusion for untied heads]
\label{thm:appendix_untied}
Assume $E\in\R^{d\times V}$ has full row rank $d<V$.
Then:
\begin{enumerate}[label=(\arabic*),leftmargin=1.5em]
    \item
    \begin{equation}
        \mathcal F_{\mathrm{untied}} = \{BE : B\in\R^{V\times d}\};
    \end{equation}
    \item
    \begin{equation}
        \mathcal F_{\mathrm{tied}} \subsetneq \mathcal F_{\mathrm{untied}};
    \end{equation}
    \item every tied map $T\in\mathcal F_{\mathrm{tied}}$ satisfies
    \begin{equation}
        \operatorname{im}(T)\subseteq \operatorname{col}(E^\top),
        \qquad
        \operatorname{ker}(E)\subseteq \operatorname{ker}(T);
    \end{equation}
    \item untied maps preserve the input-side bottleneck $\operatorname{ker}(E)\subseteq \operatorname{ker}(T)$ but can realize arbitrary output subspaces of dimension at most $d$.
\end{enumerate}
\end{theorem}

\begin{proof}
For (1), if $T=W^\top A E\in\mathcal F_{\mathrm{untied}}$, set $B=W^\top A\in\R^{V\times d}$, so $T=BE$.
Conversely, for any $B\in\R^{V\times d}$, choosing $A=I_d$ and $W=B^\top$ gives $T=W^\top A E=BE$.
Hence
\begin{equation}
    \mathcal F_{\mathrm{untied}}=\{BE:B\in\R^{V\times d}\}.
\end{equation}

For (2), every tied map is untied by taking $B=E^\top A$, so
\begin{equation}
    \mathcal F_{\mathrm{tied}}\subseteq\mathcal F_{\mathrm{untied}}.
\end{equation}
To show strictness, choose a nonzero vector $u\in\R^V$ with
\begin{equation}
    u\notin \operatorname{col}(E^\top).
\end{equation}
Such a vector exists because $\dim\operatorname{col}(E^\top)=d<V$.
Choose any nonzero $\alpha\in\R^d$ and set
\begin{equation}
    B=u\alpha^\top.
\end{equation}
Then
\begin{equation}
    T:=BE=u(\alpha^\top E)\in \mathcal F_{\mathrm{untied}}.
\end{equation}
Because $E$ has full row rank and $\alpha\neq 0$, we have $\alpha^\top E\neq 0$, so $T\neq 0$.
Moreover,
\begin{equation}
    \operatorname{im}(T)=\operatorname{span}\{u\}.
\end{equation}
Since $u\notin\operatorname{col}(E^\top)$, we have
\begin{equation}
    \operatorname{im}(T)\not\subseteq \operatorname{col}(E^\top),
\end{equation}
so $T\notin \mathcal F_{\mathrm{tied}}$.

For (3), if $T=E^\top A E$, then for every $x\in\R^V$,
\begin{equation}
    Tx = E^\top(A(Ex)),
\end{equation}
so $Tx\in\operatorname{col}(E^\top)$.
Thus $\operatorname{im}(T)\subseteq \operatorname{col}(E^\top)$.
Also if $x\in\operatorname{ker}(E)$, then $Tx=E^\top A E x=0$, so $\operatorname{ker}(E)\subseteq\operatorname{ker}(T)$.

For (4), every untied map has the form $T=BE$, so again $x\in\operatorname{ker}(E)$ implies $Tx=0$.
On the output side, however, $B$ is arbitrary.
Because $E$ has full row rank, it admits a right inverse $R\in\R^{V\times d}$ with $ER=I_d$.
Hence for every $B$,
\begin{equation}
    B=(BE)R,
\end{equation}
so $\operatorname{col}(B)\subseteq\operatorname{col}(BE)$, while trivially $\operatorname{col}(BE)\subseteq\operatorname{col}(B)$.
Therefore $\operatorname{col}(BE)=\operatorname{col}(B)$.
Since $B$ is arbitrary, any output subspace of dimension at most $d$ can occur.
\end{proof}

\begin{corollary}[Irreducible tied-head error outside the tied output subspace]
\label{cor:appendix_tied_lower_bound}
Let $P_{\mathrm{out}}$ be the orthogonal projector onto $\operatorname{col}(E^\top)$.
Then for every target map $T_\star\in\R^{V\times V}$,
\begin{equation}
    \inf_{T\in\mathcal F_{\mathrm{tied}}}\norm{T-T_\star}_F
    \ge
    \norm{(I-P_{\mathrm{out}})T_\star}_F.
\end{equation}
If in addition $T_\star=BE$ for some $B\in\R^{V\times d}$, then $T_\star\in\mathcal F_{\mathrm{untied}}$.
\end{corollary}

\begin{proof}
For any $T\in\mathcal F_{\mathrm{tied}}$, Theorem~\ref{thm:appendix_untied}(3) gives $T=P_{\mathrm{out}}T$.
Hence
\begin{equation}
    T-T_\star = P_{\mathrm{out}}T - P_{\mathrm{out}}T_\star - (I-P_{\mathrm{out}})T_\star.
\end{equation}
The first term has every column in $\operatorname{col}(E^\top)$, while the second has every column in its orthogonal complement, so these components are orthogonal in Frobenius inner product.
Therefore
\begin{equation}
    \norm{T-T_\star}_F^2
    =
    \norm{P_{\mathrm{out}}(T-T_\star)}_F^2 + \norm{(I-P_{\mathrm{out}})T_\star}_F^2
    \ge
    \norm{(I-P_{\mathrm{out}})T_\star}_F^2.
\end{equation}
Taking the infimum over $T\in\mathcal F_{\mathrm{tied}}$ proves the bound.
If $T_\star=BE$, then Theorem~\ref{thm:appendix_untied}(1) implies $T_\star\in\mathcal F_{\mathrm{untied}}$.
\end{proof}

\paragraph{Idealized Muon as spectral preconditioning.}
For a nonzero matrix $G\in\R^{m\times n}$ with compact SVD
\begin{equation}
    G=U\Sigma V^\top,
    \qquad
    \Sigma=\operatorname{diag}(\sigma_1,\dots,\sigma_r),
    \quad
    \sigma_i>0,
\end{equation}
define its exact polar factor
\begin{equation}
    Q(G)=UV^\top.
\end{equation}

\begin{theorem}[Idealized Muon is steepest descent under an operator-norm trust region]
\label{thm:appendix_muon}
Let $G\in\R^{m\times n}$ be nonzero with rank $r$.
Then:
\begin{enumerate}[label=(\arabic*),leftmargin=1.5em]
    \item for every scalar $c>0$,
    \begin{equation}
        Q(cG)=Q(G);
    \end{equation}
    \item
    \begin{equation}
        \max_{\norm{M}_{\mathrm{op}}\le 1} \langle G,M\rangle = \norm{G}_*,
    \end{equation}
    and the maximum is attained by $M=Q(G)$;
    equivalently,
    \begin{equation}
        -\eta Q(G)
    \end{equation}
    minimizes
    \begin{equation}
        \min_{\norm{\Delta}_{\mathrm{op}}\le \eta}\langle G,\Delta\rangle
    \end{equation}
    for every $\eta>0$;
    \item if $f:\R^{m\times n}\to\R$ is $L$-smooth with respect to the Frobenius norm and $G=\nabla f(W)$, then for
    \begin{equation}
        W^+ = W-\eta Q(G)
    \end{equation}
    one has
    \begin{equation}
        f(W^+)
        \le
        f(W)-\eta\norm{G}_*+\frac{L\eta^2}{2}r.
    \end{equation}
    In particular, every
    \begin{equation}
        0<\eta<\frac{2\norm{G}_*}{Lr}
    \end{equation}
    guarantees strict descent.
\end{enumerate}
\end{theorem}

\begin{proof}
For (1), if $G=U\Sigma V^\top$, then $cG=U(c\Sigma)V^\top$ has the same singular vectors, so
\begin{equation}
    Q(cG)=UV^\top=Q(G).
\end{equation}

For (2), the dual norm of the operator norm is the nuclear norm, so
\begin{equation}
    \max_{\norm{M}_{\mathrm{op}}\le 1}\langle G,M\rangle = \norm{G}_*.
\end{equation}
We verify that $M=Q(G)$ attains the maximum:
\begin{equation}
    \langle G,Q(G)\rangle
    =
    \operatorname{tr}\bigl((U\Sigma V^\top)^\top UV^\top\bigr)
    =
    \operatorname{tr}(V\Sigma V^\top)
    =
    \operatorname{tr}(\Sigma)
    =
    \norm{G}_*.
\end{equation}
Hence $Q(G)$ is a maximizer.
Replacing $M$ by $-\Delta/\eta$ shows that $-\eta Q(G)$ is a minimizer of the operator-norm-constrained linearized objective.

For (3), by $L$-smoothness,
\begin{equation}
    f(W+\Delta)\le f(W)+\langle G,\Delta\rangle + \frac{L}{2}\norm{\Delta}_F^2.
\end{equation}
Substitute $\Delta=-\eta Q(G)$:
\begin{equation}
    f(W^+)\le f(W)-\eta\langle G,Q(G)\rangle + \frac{L\eta^2}{2}\norm{Q(G)}_F^2.
\end{equation}
By part (2),
\begin{equation}
    \langle G,Q(G)\rangle = \norm{G}_*.
\end{equation}
Also $Q(G)=UV^\top$ has $r$ singular values equal to $1$, so
\begin{equation}
    \norm{Q(G)}_F^2 = r.
\end{equation}
Therefore
\begin{equation}
    f(W^+) \le f(W)-\eta\norm{G}_*+\frac{L\eta^2}{2}r.
\end{equation}
The strict-descent condition follows immediately.
\end{proof}

\section{Additional Plots}
\label{sec:appendix_additional_plots}
This section collects supporting plots that audit the claims of the main paper but were too large to include there. RankMe is the entropy effective rank defined in Section~\ref{sec:appendix_gradient_computation}, used here as a compact trajectory summary alongside tail exponents and raw spectra.

\subsection{Loss atlases}
Fig.~\ref{fig:appendix_loss_atlases} shows the loss-curve evidence behind every matched-loss comparison in the main paper. Panel~(a) collects validation-loss trajectories for families with stored validation checkpoints; panel~(b) collects train-loss trajectories for the late-trunk and d36 support runs where validation checkpoints were unavailable. Two features matter. First, every batch tier within every family reaches the target loss, confirming that matched-loss comparisons are not
selecting only the easiest tiers. Second, the spread between tiers in tokens-to-target is visible directly: smaller batches sit on shallower learning curves and require substantially more tokens than the family-local efficient regime, which sits near the modal $B=8$ tier in most variants.

\begin{figure}[t]
  \centering
  \begin{subfigure}[t]{0.98\textwidth}
    \centering
    \includegraphics[width=\linewidth]{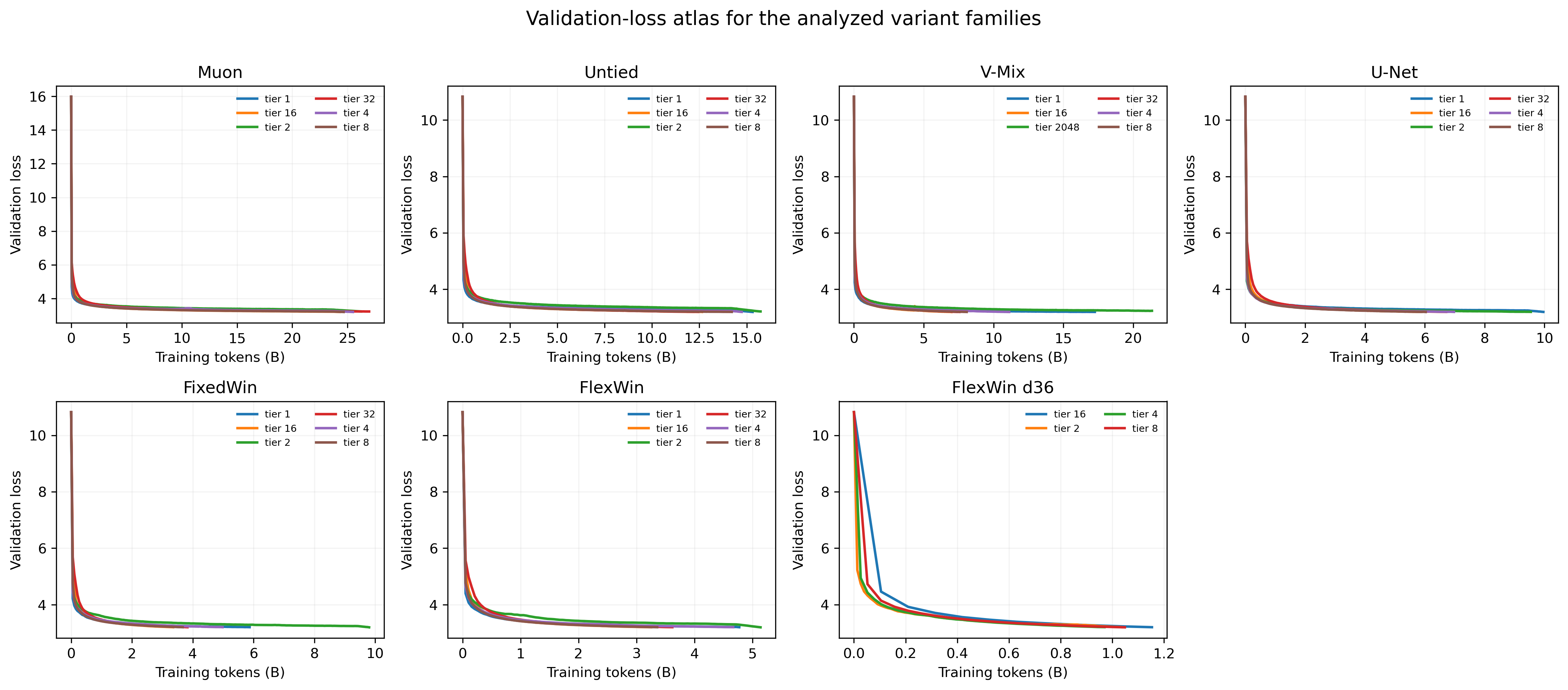}
    \caption{Validation-loss atlas for families with validation checkpoints.}
  \end{subfigure}
  \begin{subfigure}[t]{0.98\textwidth}
    \centering
    \includegraphics[width=\linewidth]{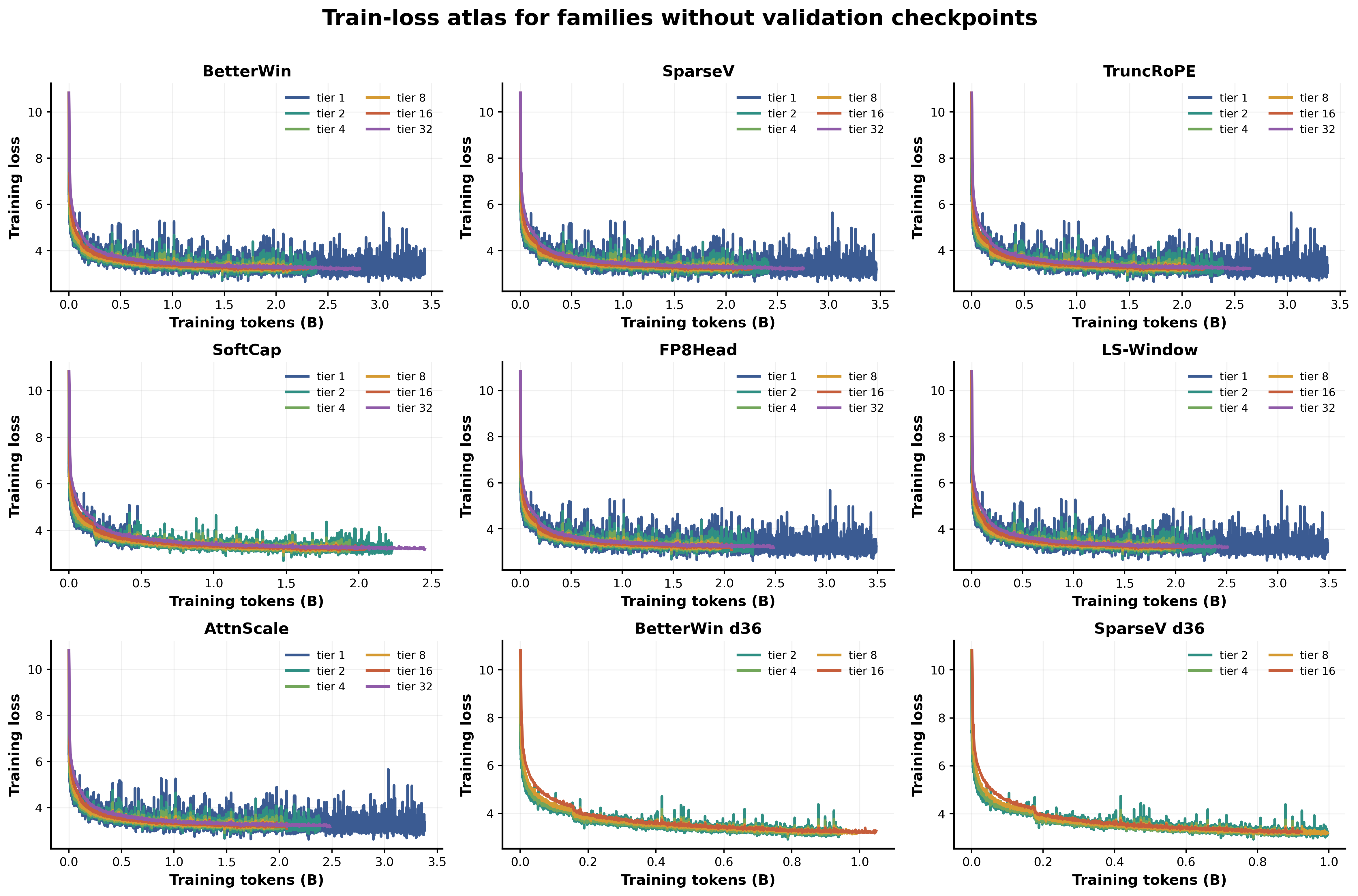}
    \caption{Train-loss atlas for d36 support families without validation checkpoints.}
  \end{subfigure}
  \caption{\textbf{Loss-curve atlases.} Validation-aligned and train-loss-only
  evidence are separated to keep the support runs distinct from the main protocol.
  Within every family, all batch tiers reach the target loss, and the
  tokens-to-target spread across tiers is visible directly in the curves.}
  \label{fig:appendix_loss_atlases}
\end{figure}

\subsection{Spectral atlases for individual models}
Figs.~\ref{fig:appendix_spectral_legacy}--\ref{fig:appendix_spectral_trunk_b} provide the per-variant raw evidence underlying the transition-level taxonomy of Section~\ref{sec:taxonomy}. Each row shows one variant in three views: trace-normalized activation covariance spectra, gradient SVD spectra, and summary trajectories of RankMe and $\alpha_{\mathrm{tail}}$ over training. The legacy prefix atlas (Fig.~\ref{fig:appendix_spectral_legacy}) contains the largest geometric shifts in the chain, spanning both the early gradient-led and activation-led transitions seen in the main taxonomy figure. The matched-trunk atlases (Figs.~\ref{fig:appendix_spectral_trunk_a}--\ref{fig:appendix_spectral_trunk_b}) show the smaller incremental changes from ValueMix through AttnScale, where the transition labels are best inferred from the joint activation--gradient summaries rather than from any single raw spectrum alone. The tier-2 atlas (Fig.~\ref{fig:appendix_d36_d48_tier2}) applies the same format to the larger-scale support runs and shows that the qualitative spectral signatures persist at scale.

\begin{figure}[t]
    \centering
    \includegraphics[width=\textwidth]{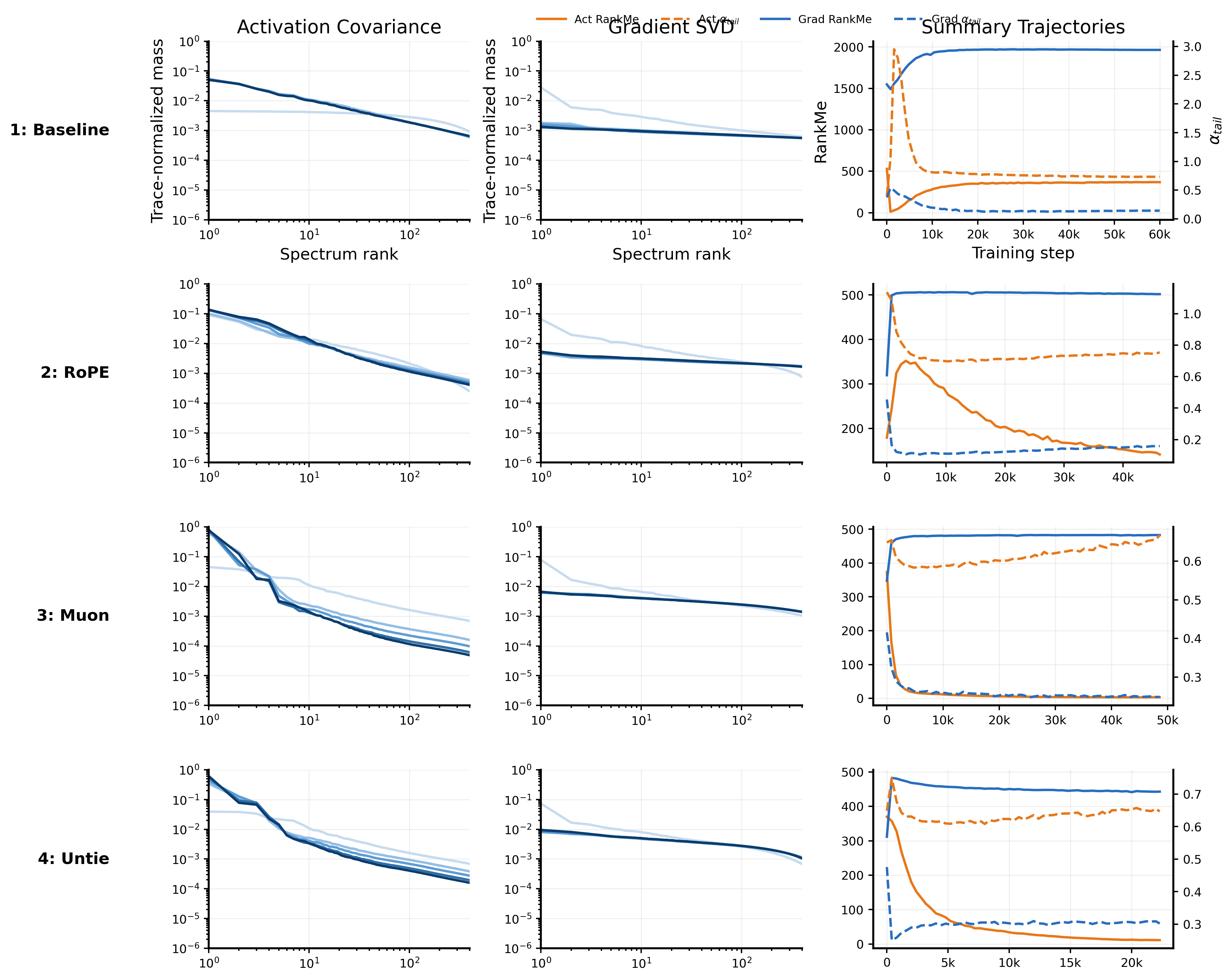}
    \caption{\textbf{Spectral atlas for the legacy prefix variants.} Rows show Baseline, RoPE, Muon, and Untied. Each consecutive variant produces visibly distinct activation and gradient spectra, dominating the activation-led column of the main taxonomy figure.}
    \label{fig:appendix_spectral_legacy}
\end{figure}

\begin{figure}[t]
    \centering
    \includegraphics[width=\textwidth]{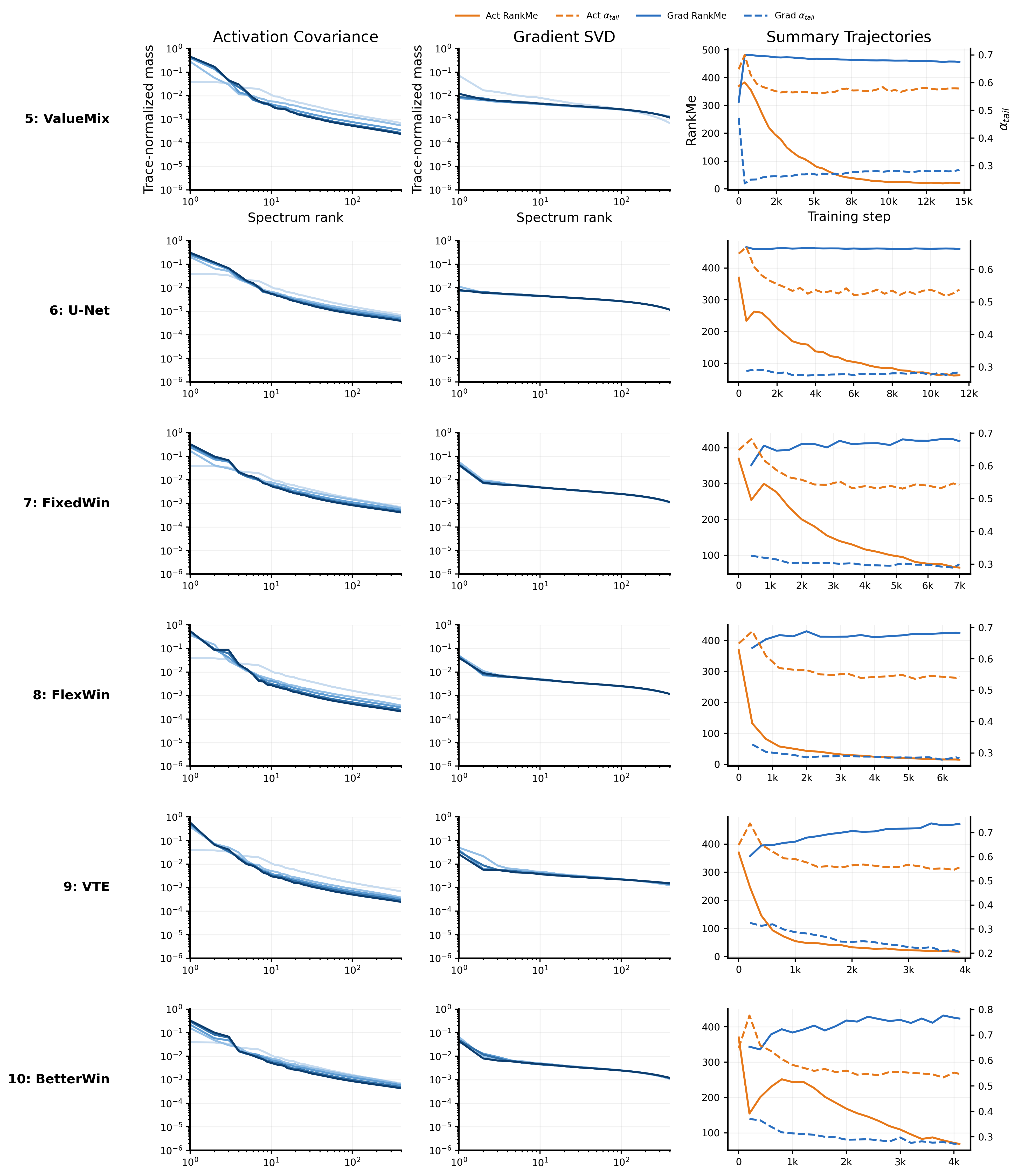}
    \caption{\textbf{Spectral atlas for the first half of the matched trunk.} Rows show ValueMix, U-Net, FixedWin, FlexWin, VTE, and BetterWin. Per-row spectral differences are smaller than across the legacy prefix; the taxonomic split is best read off the joint activation--gradient summary trajectories.}
    \label{fig:appendix_spectral_trunk_a}
\end{figure}

\begin{figure}[t]
    \centering
    \includegraphics[width=\textwidth]{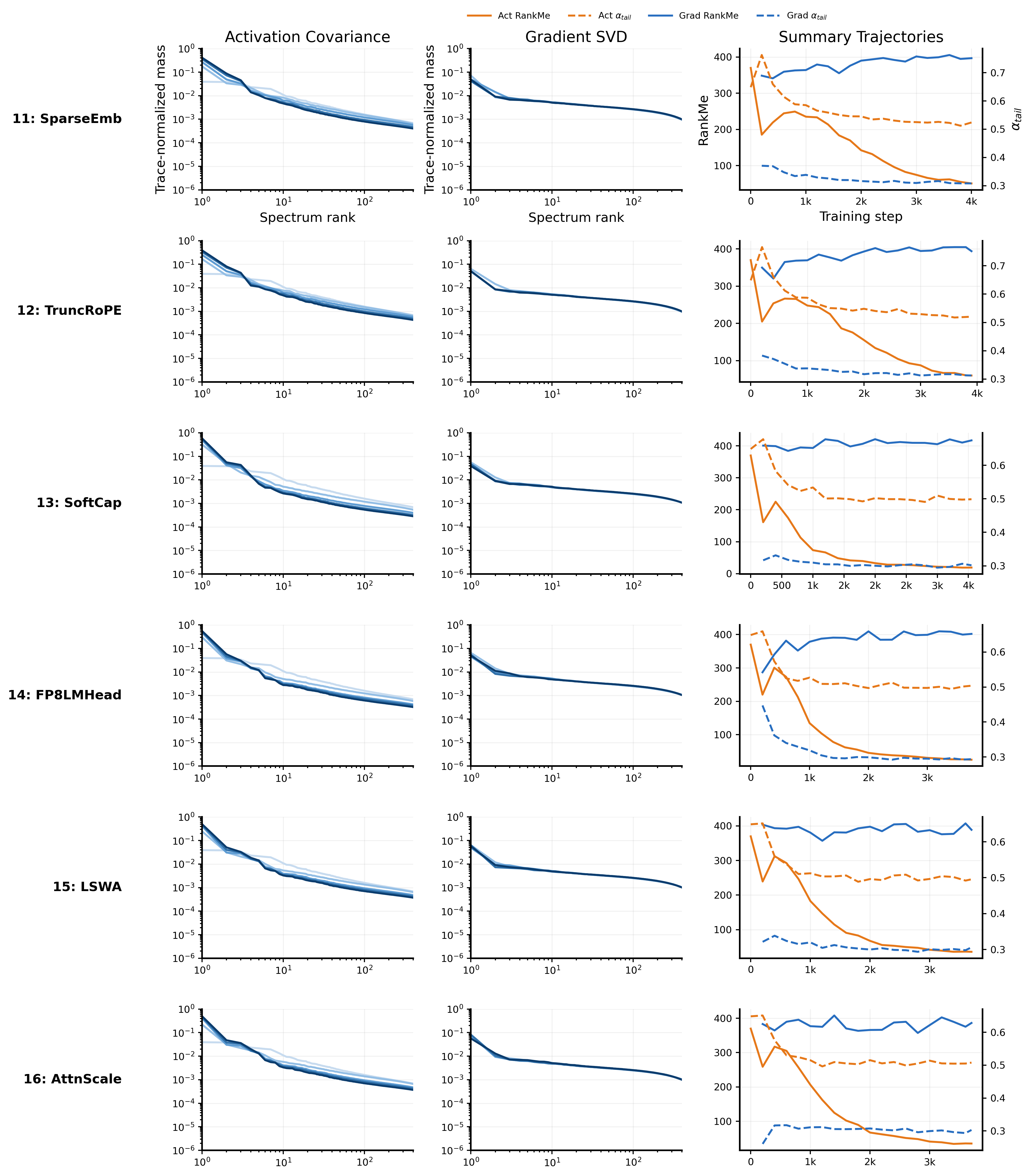}
    \caption{\textbf{Spectral atlas for the second half of the matched trunk.} Rows show SparseV, TruncRoPE, SoftCap, FP8Head, LSWA, and AttnScale. The throughput-leaning variants (FP8Head, LSWA, AttnScale) show smaller activation-side shifts than the earlier trunk variants, consistent with the Section~\ref{sec:taxonomy} taxonomy.}
    \label{fig:appendix_spectral_trunk_b}
\end{figure}

\begin{figure}[t]
  \centering
  \includegraphics[width=\textwidth]{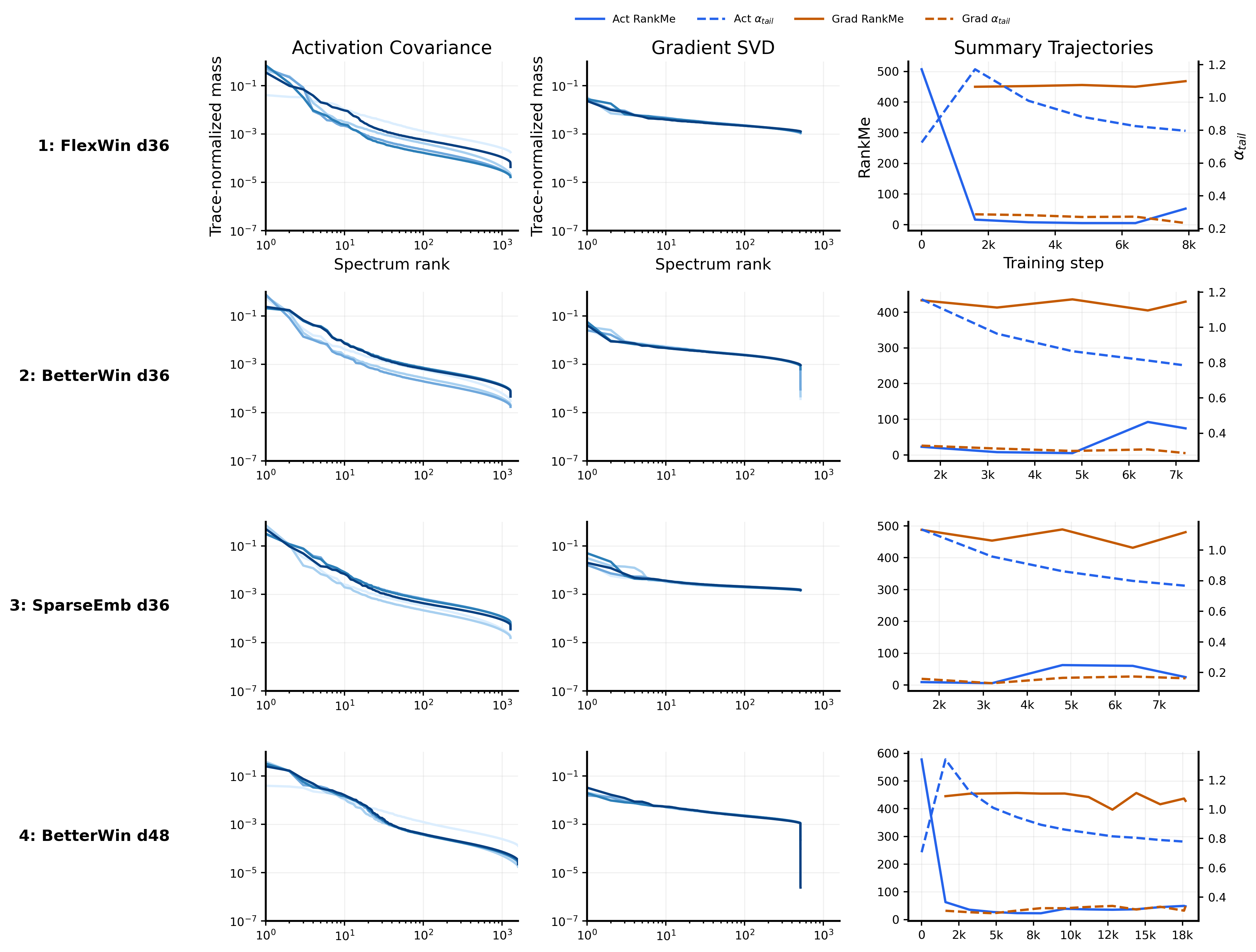}
  \caption{\textbf{Tier-2 spectral atlas for the d36/d48 scale follow-up.} Activation covariance, gradient spectra, RankMe, and tail-exponent trajectories for FlexWin d36, BetterWin d36, SparseV d36, and BetterWin d48. The qualitative spectral signatures match the corresponding d12 variants, supporting the cross-scale claim of Section~\ref{sec:batch}}
  \label{fig:appendix_d36_d48_tier2}
\end{figure}

\subsection{Weight-matrix spectra}
Activation and gradient spectra describe the data-side and update-side of training. Fig.~\ref{fig:weight_spectrum_complements} adds a third view: the spectra of the trained weight matrices themselves, comparing the layer-11 attention-output projection $W_O$ against the layer-11 MLP-output projection at the final checkpoint, with head exponents shown at step~1600 and at the end of training. The MLP-output projection shows clearer cross-variant divergence than $W_O$ in both raw spectrum shape and
head-exponent trajectories. The asymmetry is consistent with the gradient-probe stability argument in Appendix~\ref{sec:appendix_gradient_computation}: $W_O$ retains
the same architectural role across the full chain, while the feed-forward writeback path is touched by several trunk-side interventions (ValueMix, U-Net, BetterWin, SparseV). Either probe is informative, but they answer different questions and should
not be expected to coincide.

\begin{figure}[t]
  \centering
  \begin{subfigure}[t]{0.49\textwidth}
    \centering
    \includegraphics[width=\linewidth]{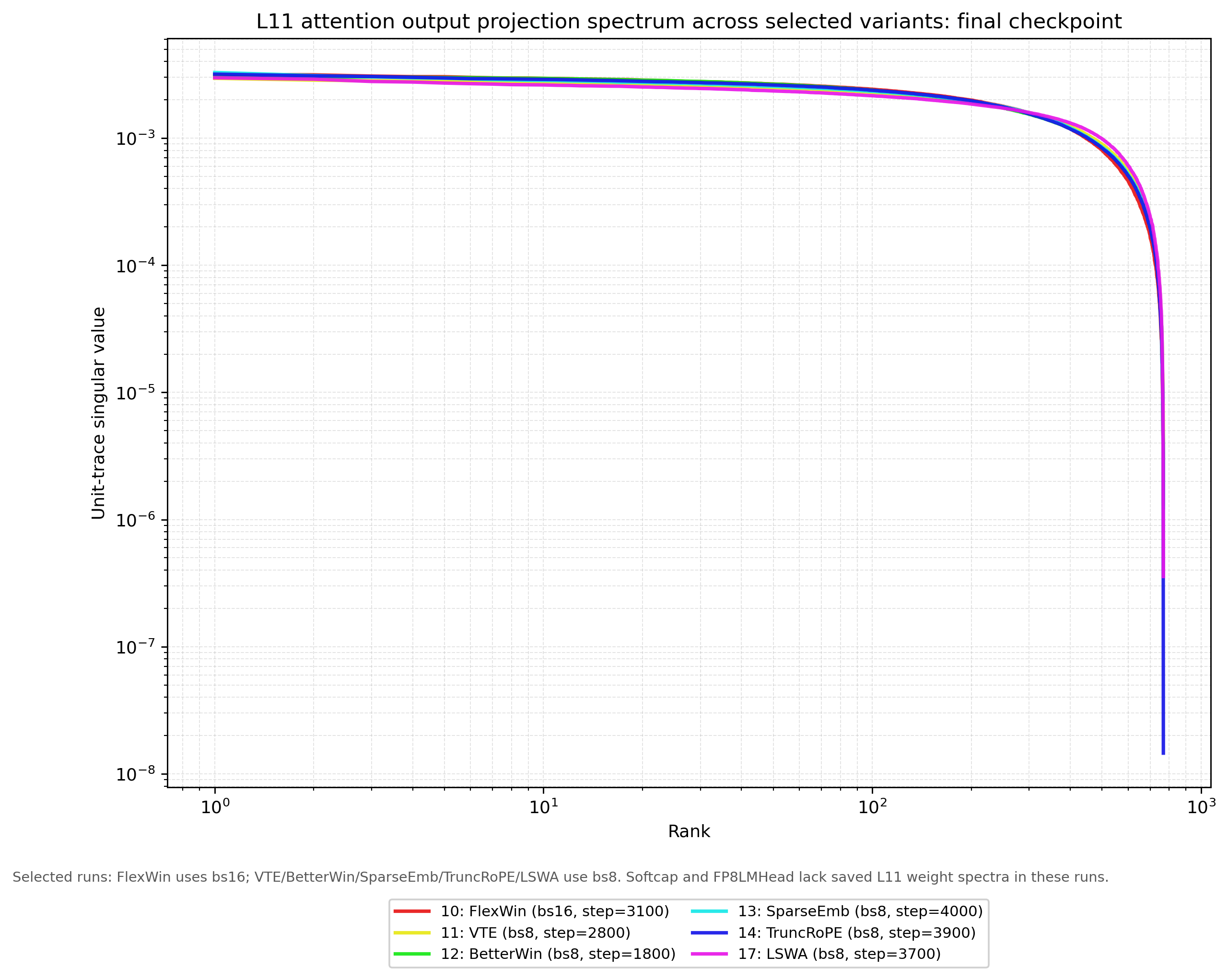}
    \caption{Layer-11 attention-output weight spectrum.}
  \end{subfigure}\hfill
  \begin{subfigure}[t]{0.49\textwidth}
    \centering
    \includegraphics[width=\linewidth]{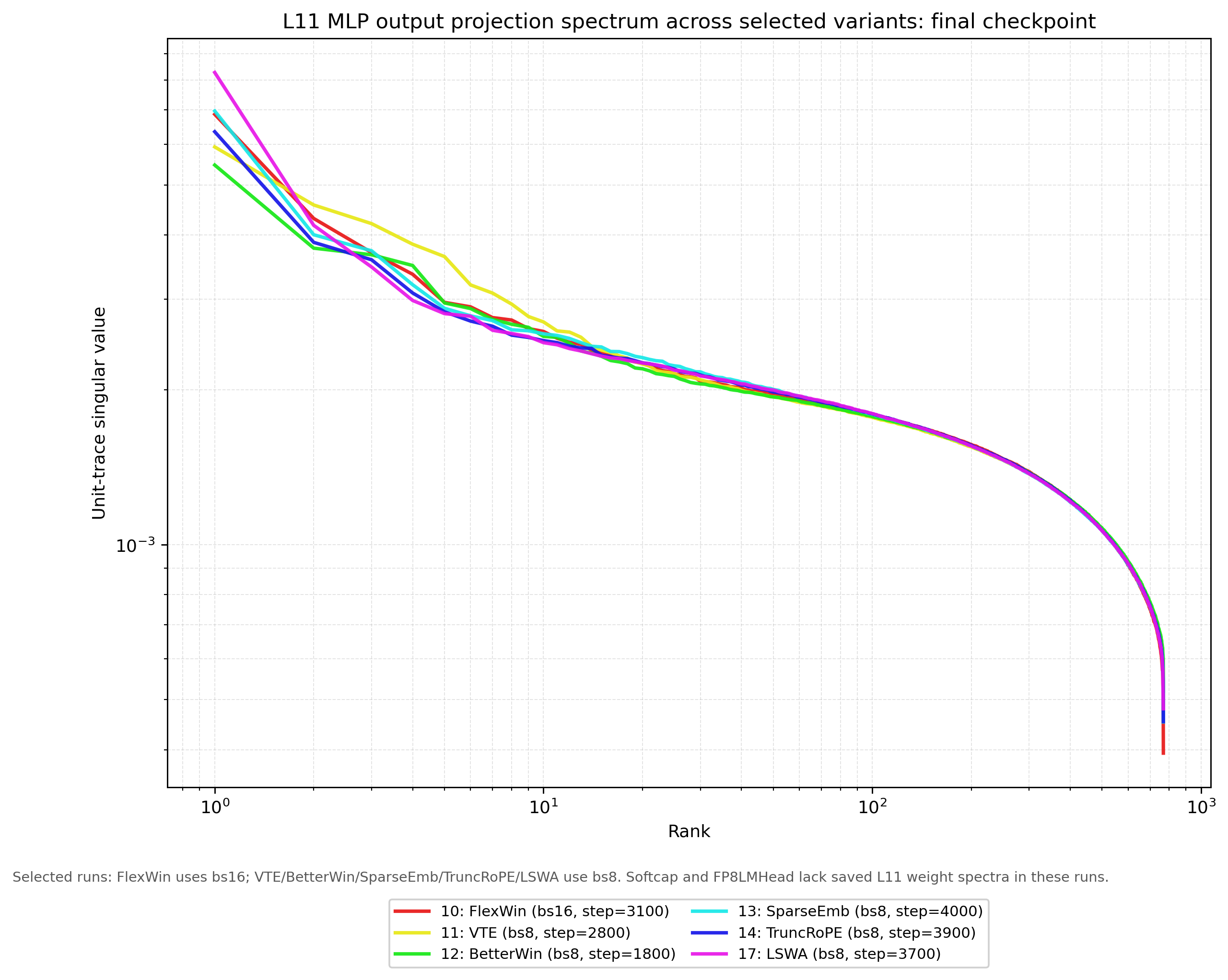}
    \caption{Layer-11 MLP-output weight spectrum.}
  \end{subfigure}\hfill
  \begin{subfigure}[t]{0.49\textwidth}
    \centering
    \includegraphics[width=\linewidth]{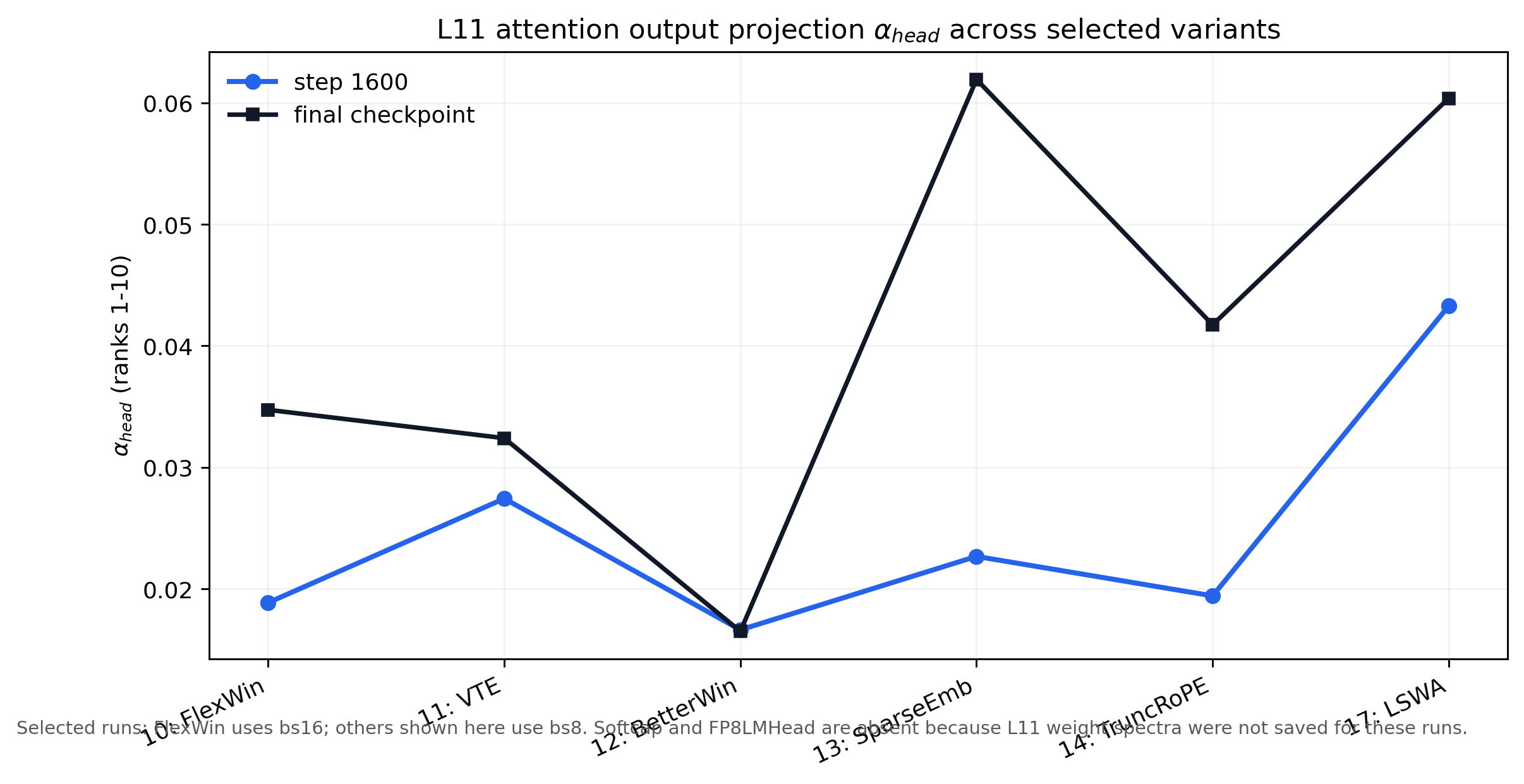}
    \caption{Attention-output head exponent.}
  \end{subfigure}\hfill
  \begin{subfigure}[t]{0.49\textwidth}
    \centering
    \includegraphics[width=\linewidth]{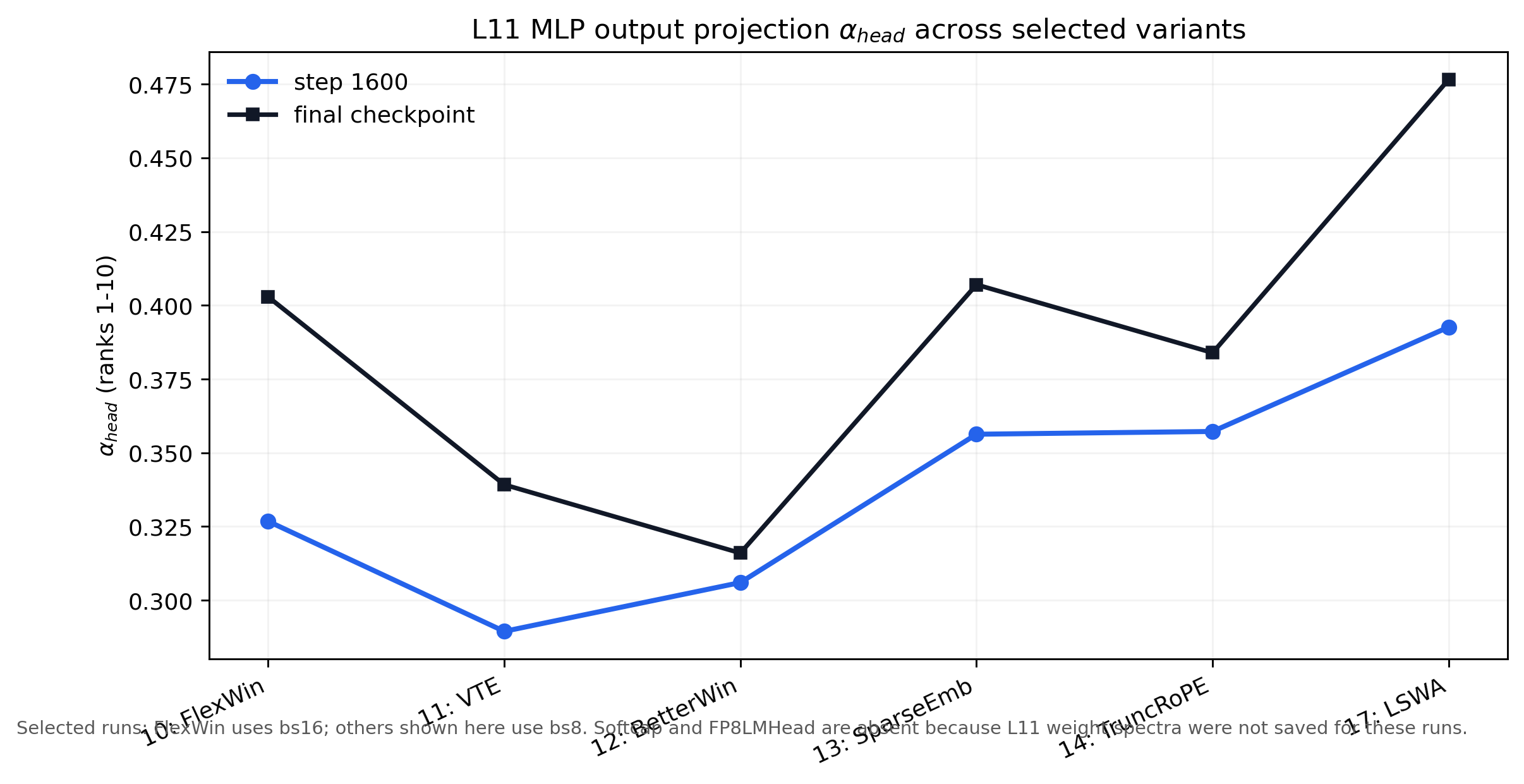}
    \caption{MLP-output head exponent.}
  \end{subfigure}
  \caption{\textbf{Weight spectra are informative but tensor-dependent.} Layer-11 attention-output and MLP-output projections at the final checkpoint, plus their head exponents at step~1600 and at the end of training. The MLP-output projection shows clearer parameter-side divergence across variants, consistent with $W_O$'s stable architectural role and the additional cross-variant variance accumulated by
  the feed-forward writeback path.}
  \label{fig:weight_spectrum_complements}
\end{figure}

\subsection{Training dynamics and phase visibility}
Prior geometry work has reported a collapse--expansion--compression phase sequence in
representation rank during training.
Fig.~\ref{fig:appendix_rankme_progress_all_tiers} plots RankMe trajectories on
normalized training progress for FixedWin, SparseV, and LSWA across all batch tiers.
The classical phase shape is most clearly visible in intermediate tiers; very small
and very large batches often produce monotone or muted trajectories without a
sharply localized expansion peak. We therefore treat phase-like dynamics as secondary
qualitative evidence rather than a universal training signature, since their
visibility depends on both batch tier and variant.

\begin{figure}[t]
  \centering
  \begin{subfigure}[t]{0.32\textwidth}
    \centering
    \includegraphics[width=\linewidth]{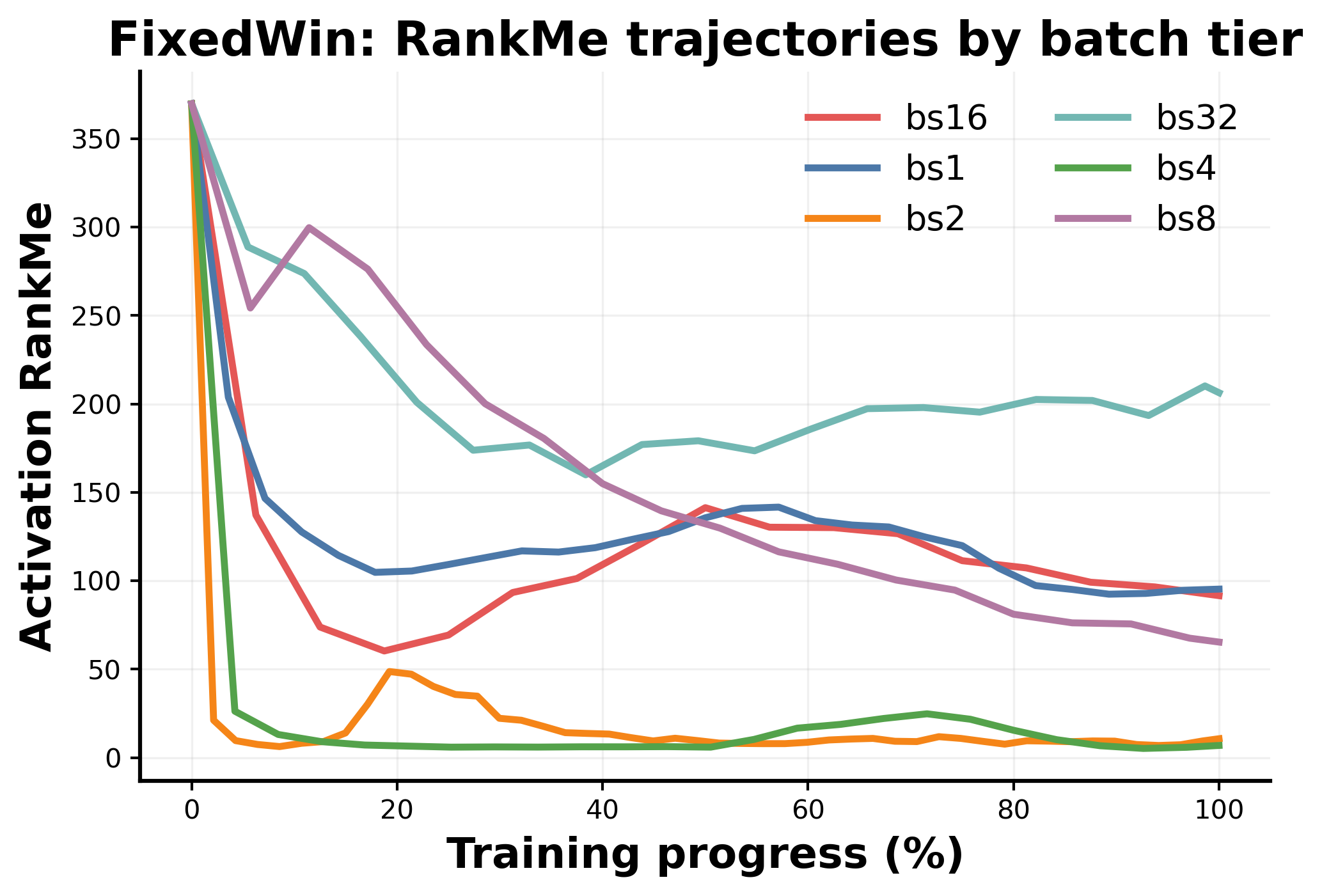}
    \caption{FixedWin.}
  \end{subfigure}\hfill
  \begin{subfigure}[t]{0.32\textwidth}
    \centering
    \includegraphics[width=\linewidth]{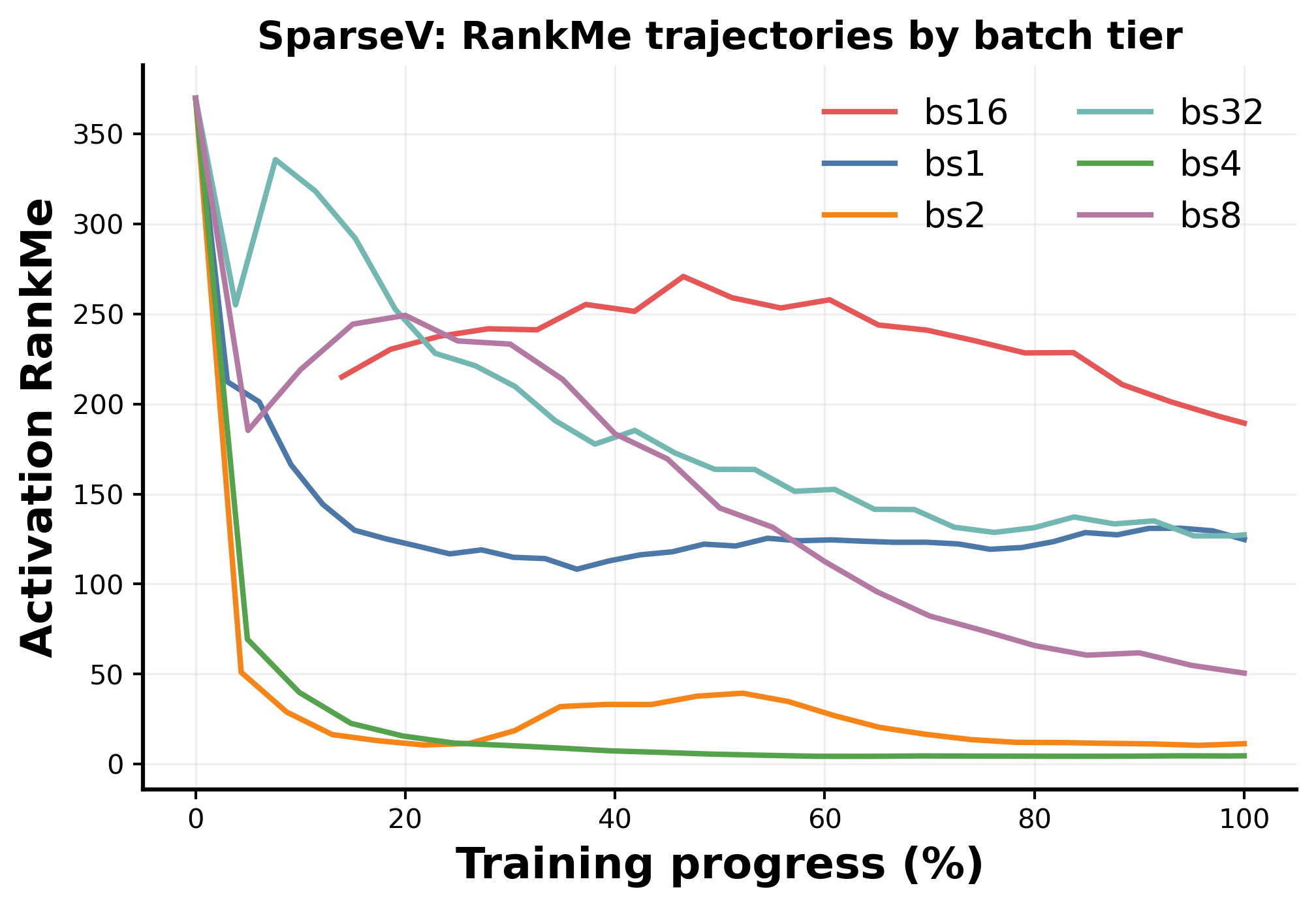}
    \caption{SparseV.}
  \end{subfigure}\hfill
  \begin{subfigure}[t]{0.32\textwidth}
    \centering
    \includegraphics[width=\linewidth]{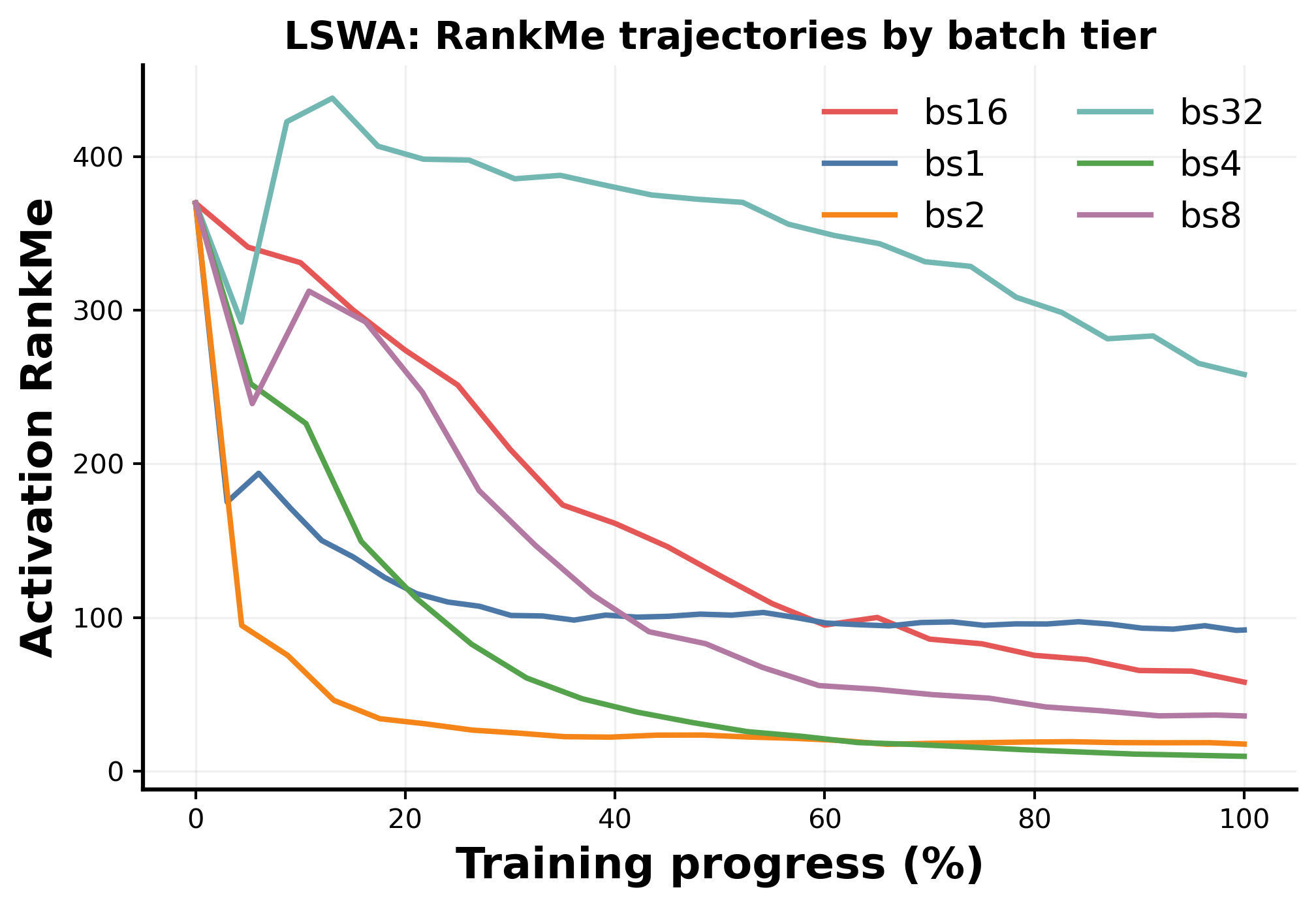}
    \caption{LSWA.}
  \end{subfigure}
  \caption{\textbf{Phase-like RankMe trajectories are batch-regime dependent.}
  Collapse--expansion--compression behavior is not uniform across batch size or
  variant; the phase sequence reported in prior geometry work appears most clearly
  in intermediate tiers, so we treat it as qualitative support rather than a
  universal law.}
  \label{fig:appendix_rankme_progress_all_tiers}
\end{figure}

\end{document}